\documentclass[12pt]{article}

\usepackage{natbib}
\usepackage{times}
\usepackage{hyperref}       
\usepackage{url}            
\usepackage{booktabs}       
\usepackage{amsfonts}       
\usepackage{nicefrac}       
\usepackage{microtype}      
\usepackage{ltablex}
\usepackage{multirow}

\usepackage{amsmath, amssymb, amsthm, graphicx, bm}
\usepackage{dsfont}
\usepackage{tikz}
\usetikzlibrary{automata, positioning, arrows}
\usepackage{mwe}
\usepackage{algorithm, algpseudocode}
\usepackage[section]{placeins}
\usepackage{enumitem}
\newcommand{\Scal}{\mathcal{S}}
\newcommand{\Acal}{\mathcal{A}}

\newcommand{\eps}{\epsilon}

\newcommand{\cA}{\ensuremath{\mathcal{A}}}

\newcommand{\Aset}{\ensuremath{\mathcal{A}}}

\newcommand{\R}{\ensuremath{\mathbb{R}}}

\newcommand{\E}{\ensuremath{\mathbb{E}}}

\DeclareMathOperator{\argmax}{argmax}
\DeclareMathOperator{\argmin}{{argmin}}

\newcommand{\w}{w}

\newcommand{\ind}{\ensuremath{\mathds{1}}}

\newcommand{\Sset}{\ensuremath{\mathcal{S}}}

\theoremstyle{theorem}
\newtheorem{theorem}{Theorem}[section]
\newtheorem{lemma}{Lemma}[section]
\newtheorem{corollary}{Corollary}[section]

\newtheorem{assumption}{Assumption}[section]
\theoremstyle{definition}
\newtheorem{definition}{Definition}[section]

\newtheorem{remark}{Remark}[section]
\newtheorem*{lemma*}{Lemma}
\newtheorem{proposition}{Proposition}[section]

\renewenvironment{proof}{{\bf Proof:}}{\hfill\rule{2mm}{2mm}}

\newcommand{\card}[1]{\lvert#1\rvert}

\newcommand{\Norm}[1]{\left\lVert#1\right\rVert}

\newcommand{\BigNorm}[1]{\Bigl\lVert#1\Bigr\rVert}
\newcommand{\norm}[1]{\lVert#1\rVert}
\newcommand{\abs}[1]{\lvert#1\rvert}
\newcommand{\Abs}[1]{\left\lvert#1\right\rvert}

\newcommand{\bp}{\ensuremath{\theta}}
\newcommand{\cntidx}{\ensuremath{\alpha}}
\newcommand{\bcnt}{\ensuremath{\bm{\cntidx}}}
\newcommand{\idx}{\ensuremath{\beta}}
\newcommand{\bidx}{\ensuremath{\bm{\idx}}}
\newcommand{\hist}{\ensuremath{\text{hist}}}
\newcommand{\mbound}{\ensuremath{\alpha}}
\newcommand{\pmax}{\ensuremath{\bar{p}}}
\newcommand{\pmin}{\ensuremath{\underline{p}}}
\renewcommand{\S}{\ensuremath{\mathbb{S}}}
\newcommand{\kl}{\ensuremath{\mathrm{KL}}}

\newcommand{\epsopt}{\ensuremath{\eps_{\text{opt}}}}

\newcommand{\wh}{\ensuremath{\widehat w}}

\newcommand{\diag}{\textrm{diag}}
\newcommand{\ones}{\mathbf{1}}
\newcommand{\iden}{\textrm{I}}

\newcommand{\Wcal}{\ensuremath{\mathcal{W}}} 

\ifdefined\usebigfont

\usepackage{times}
\usepackage[fontsize=13pt]{scrextend}
\usepackage[left=1.56in,right=1.56in,top=1.71in,bottom=1.77in]{geometry}
\usepackage{enumitem}
\setlist[itemize]{leftmargin=*}
\setlist[enumerate]{leftmargin=*}
\else
\usepackage{fullpage}
\usepackage{times}
\fi
\usepackage{xcolor}

\renewcommand{\aa}[1]{\textcolor{violet}{AA: #1}}
\newcommand{\sk}[1]{\textsf{\color{magenta} SK: #1}}

\newcommand{\epsstat}{\epsilon_{\mathrm{stat}}}
\newcommand{\epsbias}{\epsilon_{\mathrm{bias}}}
\newcommand{\epsapprox}{\epsilon_{\mathrm{approx}}}

\title{On the Theory of Policy Gradient Methods:\\ Optimality,
  Approximation, and Distribution Shift}

\author{%
	Alekh~Agarwal\thanks{Microsoft Research, Redmond, WA 98052. Email: \texttt{alekha@microsoft.com}}\quad
	Sham~M.~Kakade\thanks{University of Washington, Seattle, WA
          98195 \& Microsoft Research. Email: \texttt{sham@cs.washington.edu}}\quad
	Jason~D.~Lee\thanks{Princeton University, Princeton, NJ 08540. Email: \texttt{jasonlee@princeton.edu}}\quad
	Gaurav~Mahajan\thanks{University of California San Diego, La Jolla, CA 92093. Email: \texttt{gmahajan@eng.ucsd.edu}}
}
\date{}
\begin{document}
\maketitle

\begin{abstract}
  Policy gradient methods are among the most effective methods in
challenging reinforcement learning problems with large state and/or
action spaces. However, little is known about even their most basic
theoretical convergence properties, including: if and how fast they
converge to a globally optimal solution or how they cope with
approximation error due to using a restricted class of parametric
policies.  This work provides provable characterizations of the
computational, approximation, and sample size properties of
policy gradient methods in the context of discounted Markov Decision
Processes (MDPs). We focus on both: ``tabular'' policy
parameterizations, where the optimal policy is contained in the class
and where we show global convergence to the optimal policy; and
parametric policy classes (considering both log-linear and
neural policy classes), which may not contain the optimal policy
and where we provide agnostic learning results.  One central
contribution of this work is in providing approximation guarantees that
are average case --- which avoid explicit worst-case dependencies
on the size of state space --- by making a formal connection to
supervised learning under \emph{distribution shift}. This
characterization shows an important interplay between estimation
error, approximation error, and exploration (as characterized through a
precisely defined condition number).

\end{abstract}

\section{Introduction}
Policy gradient methods have a long history in the reinforcement
learning (RL) literature~\citep{williams1992simple, sutton1999policy,
  konda2000actor, Kakade01} and are an attractive class of algorithms
as they are applicable to any differentiable policy parameterization;
admit easy extensions to function approximation; easily incorporate structured state
and action spaces; are easy to implement in a simulation based, model-free manner. Owing to their flexibility and generality,
there has also been a flurry of improvements and refinements to make
these ideas work robustly with deep neural network based approaches (see e.g.~\cite{schulman2015trust,
  schulman2017proximal}).

Despite the large body of empirical work around these
methods, their convergence properties are only established
at a relatively coarse level; in particular, the
folklore guarantee
is that these methods converge to a stationary point of the objective,
assuming adequate smoothness properties hold and assuming either
exact or unbiased estimates of a gradient can be obtained (with
appropriate regularity conditions on the variance).  However, this
local convergence viewpoint does not address some of the most basic
theoretical convergence questions, including: 1) if and how fast they
converge to a globally optimal solution (say with a sufficiently rich
policy class); 2) how they cope with approximation error due to using
a restricted class of parametric policies; or 3) their finite sample
behavior. These questions are the focus of this work.

Overall, the results of this work place policy gradient methods under a solid
theoretical footing, analogous to the global convergence guarantees of
iterative value function based algorithms.

\subsection{Our Contributions}

This work focuses on first-order and quasi second-order policy gradient
methods which directly work in the space of some parameterized policy
class (rather than value-based approaches).  We characterize
the computational, approximation, and sample size properties of
these methods in the context of a discounted Markov Decision Process
(MDP). We focus on: 1) \emph{tabular policy parameterizations}, where there
is one parameter per state-action pair so the policy class is
complete in that it contains the optimal policy, and 2)
\emph{function approximation}, where we have a
restricted class or parametric policies which may not contain the
globally optimal policy.
Note that policy gradient methods
for discrete action MDPs work in the space of stochastic policies,
which permits the policy class to be differentiable. We now discuss our contributions in the both of these contexts.

\renewcommand{\arraystretch}{1.5}
\begin{table}[!t]
\centering
\begin{tabular}{|>{\centering\arraybackslash}m{9.3cm}|>{\centering\arraybackslash}m{6.5cm}|}
\hline\textbf{Algorithm}&\textbf{Iteration complexity}\\ \hline
\rule{0pt}{6ex} Projected Gradient Ascent on Simplex  (Thm~\ref{thm:proj-gd}) & $O\left(\frac{D_\infty^2 |\Scal| |\Acal|}{(1-\gamma)^6\epsilon^2} \right)$ \\[\medskipamount] \hline
\rule{0pt}{4ex} Policy Gradient, softmax parameterization (Thm~\ref{thm:glb-softmax}) & asymptotic \\[\medskipamount] \hline
\rule{0pt}{4ex} Policy
  Gradient $+$ log barrier regularization,\hspace{1cm} softmax parameterization (Cor~\ref{corollary:entropy}) & $O\left(\frac{ D_\infty^2|\Scal|^2|\Acal|^2}{(1-\gamma)^6\,\epsilon^2} \right)$\\[\medskipamount] \hline
\rule{0pt}{4ex} Natural Policy Gradient (NPG),\hspace{5cm} softmax parameterization (Thm~\ref{thm:npg}) & $\frac{2}{(1-\gamma)^2\epsilon}$ \\[\medskipamount]\hline
\end{tabular}
\vspace{0.2cm}
\caption{\textbf{Iteration Complexities with Exact Gradients for the Tabular
    Case: } A summary of the number of iterations required by different algorithms to find a policy $\pi$ such that
  $V^\star(s_0) - V^\pi(s_0) \leq \epsilon$ for some fixed $s_0$,
    assuming access to \emph{exact policy gradients}. The first three algorithms optimize the objective
  $\E_{s\sim \mu}[V^\pi(s)]$, where $\mu$ is the
  starting state distribution for the algorithms. The MDP has $|\Scal|$ states,
  $|\Acal|$ actions, and discount factor $0\leq \gamma < 1$. The quantity $D_\infty := \max_s
  \Big(\frac{d^{\pi^\star}_{s_0}(s)}{\mu(s)}\Big)$ is termed the
  \emph{distribution mismatch coefficient}, where, roughly speaking, $d^{\pi^\star}_{s_0} (s)$ is the
  fraction of time spent in state $s$ when executing an optimal
  policy $\pi^\star$, starting from the state $s_0$ (see ~\eqref{eqn:dpi}). The NPG
  algorithm directly optimizes $V^\pi(s_0)$ for any state $s_0$. In
  contrast to the complexities of the previous three algorithms, NPG
  has no dependence on the coefficient $D_\infty$, nor does it depend
  on the choice of $s_0$.
  Both the MDP Experts Algorithm~\citep{even-dar2009online}
  and MD-MPI algorithm~\citep{geist2019theory} (see
  Corollary 3 of their paper) also yield guarantees for the same update rule as NPG for
  the softmax parameterization, though at a worse rate. See Section~\ref{section:related} for
  further discussion.}
\label{table:tabular}
\end{table}
\renewcommand{\arraystretch}{1}

\paragraph{Tabular case:}
We consider three algorithms: two of which are first order methods,
projected gradient ascent (on the simplex) and gradient ascent (with a
softmax policy parameterization); and the third algorithm, natural
policy gradient ascent, can be viewed as a quasi second-order method
(or preconditioned first-order method). Table~\ref{table:tabular}
summarizes our main results in this case: upper bounds on the number
of iterations taken by these algorithms to find an $\eps$-optimal
policy, when we have access to exact policy gradients.

Arguably, the most natural starting point for an analysis of policy
gradient methods is to consider directly
doing gradient ascent on the policy simplex itself and then to project back onto
the simplex if the constraint is violated after a gradient update; we
refer to this algorithm as projected
gradient ascent on the simplex. Using a notion of gradient domination~\citep{Pol63},
our results provably show that any
first-order stationary point of the value function results in an
approximately optimal policy, under certain regularity
assumptions; this allows for a global convergence analysis by
directly appealing to standard results in the non-convex optimization literature.

A more practical and commonly used parameterization is the softmax
parameterization, where the simplex constraint is explicitly enforced
by the exponential parameterization, thus avoiding projections. This
work provides the first global convergence guarantees using only
first-order gradient information for the widely-used softmax
parameterization.  Our first result for this parameterization
establishes the asymptotic convergence of the policy gradient
algorithm; the analysis challenge here is that the optimal policy
(which is deterministic) is attained by sending the softmax parameters
to infinity.

In order to establish a finite time, convergence rate to optimality
for the softmax parameterization, we then consider a \emph{log barrier}
regularizer
 and provide an iteration complexity bound that is
polynomial in all relevant quantities. The use of our log barrier
regularizer is critical to avoiding the issue of gradients becomingly
vanishingly small at suboptimal near-deterministic policies, an issue
of significant practical relevance. The log barrier regularizer can
also be viewed as using a \emph{relative} entropy regularizer;
here, we note the general approach of
entropy based regularization is common in practice (e.g. see
~\citep{williams1991function,mnih2016asynchronous,
  Peters+MA:2010,abdolmaleki2018maximum, ahmed2019understanding}). One
notable distinction, which we discuss later, is that our analysis is
for the log barrier regularization rather than the entropy regularization.

For these aforementioned algorithms, our convergence rates depend on
the optimization measure having coverage over the state space, as
measured by the \emph{distribution mismatch coefficient} $D_\infty$ (see Table~\ref{table:tabular}
caption). In particular, for the convergence rates shown
in Table~\ref{table:tabular} (for the aforementioned algorithms), we
assume that the optimization objective is the expected (discounted)
cumulative value where the initial state is sampled under some
distribution, and $D_\infty$
is a measure of the coverage of this initial distribution.  Furthermore, we
provide a lower bound that shows such a dependence is unavoidable for
first-order methods, even when exact gradients are available.

We then consider the Natural Policy Gradient (NPG)
algorithm~\citep{Kakade01} (also see
\cite{Bagnell:2003:CPS:1630659.1630805,Peters:2008:NA:1352927.1352986}),
which can be considered a quasi second-order
method due to the use of its particular preconditioner, and
provide an iteration complexity to achieve an $\eps$-optimal policy
that is at most $\frac{2}{(1-\gamma)^2\epsilon}$ iterations,
improving upon the previous related results
of~\citep{even-dar2009online, geist2019theory} (see
Section~\ref{section:related}).  Note the convergence rate has
\emph{no} dependence on the number of states or the number of actions,
nor does it depend on the distribution mismatch coefficient $D_\infty$.
We provide a simple and concise proof for the
convergence rate analysis by extending the approach developed in
~\citep{even-dar2009online}, which uses a mirror descent style of
analysis~\citep{nemirovsky1983problem, Cesa-Bianchi:2006:PLG:1137817}
and also handles the non-concavity of the policy
optimization problem.

This fast and dimension free convergence rate shows how the
variable preconditioner in the natural gradient method improves over
the standard gradient ascent algorithm.  The dimension free aspect of
this convergence rate is worth reflecting on, especially given the
widespread use of the natural policy gradient algorithm along with
variants such as the Trust Region Policy Optimization (TRPO)
algorithm~\citep{schulman2015trust}; our results may help to provide
analysis of a more general family of entropy based algorithms (see for
example~\cite{DBLP:journals/corr/NeuJG17}).

\renewcommand{\arraystretch}{1.5}
\begin{table}[!ht]
\centering
\begin{tabular}{|>{\centering\arraybackslash}m{5.2cm}|>{\centering\arraybackslash}m{5cm}|>{\flushleft\arraybackslash}m{4.6cm}|}
\hline\textbf{Algorithm}&\textbf{Suboptimality \hspace{2cm}\hfill after $T$ Iterations}&\textbf{\hspace{0.5cm}Relevant Quantities}\\ \hline
\rule{0pt}{4ex} Approx. Value/Policy Iteration {\footnotesize \citep{bertsekas1996neuro} }&
$\frac{\eps_\infty}{(1-\gamma)^2} +\frac{\gamma^T}{(1-\gamma)^2}$&
$\eps_\infty$: $\ell_\infty$ error of values
\\[\medskipamount] \hline \rule{0pt}{4ex}
\rule{0pt}{4ex} Approx. Policy Iteration, \hspace{1cm}\hfill
with  concentrability\hspace{3cm}\hfill
{\footnotesize  \citep{munos2005error,antos2008learning} }
&
$\frac{C_{\infty}\eps_1}{(1-\gamma)^2}+\frac{\gamma^T}{(1-\gamma)^2}$
&
$\eps_1$: an $\ell_1$ average error \hspace*{0.3cm}\vspace*{0.1cm}\hfill
$C_\infty$: concentrability\hspace{2cm}\hfill (max density ratio)
\\
[\medskipamount] \hline \rule{0pt}{4ex}
Conservative Policy Iteration\hspace{2cm}\hfill
{\footnotesize \citep{kakade2002approximately}}\hspace{3cm}\hfill
{\footnotesize Related: PSDP~\citep{NIPS2003_2378}, MD-MPI~\cite{geist2019theory}}
&$ \frac{D_\infty\eps_1}{(1-\gamma)^2}+\frac{1}{(1-\gamma)\sqrt{T}}$
&$\eps_1$: an $\ell_1$ average error \hspace*{0.3cm}\vspace*{0.1cm}\hfill
$D_\infty$: max density ratio to opt., $D_\infty \leq C_\infty$
\\[\medskipamount] \hline
\rule{0pt}{4ex} Natural Policy Gradient (Cor.~\ref{cor:q_npg_fa} and Thm.~\ref{thm:npg_fa}) &
$\sqrt{ \frac{\kappa \epsstat+ D_\infty \epsapprox}{(1-\gamma)^3}}
+\frac{1}{(1-\gamma)\sqrt{T}}$&
$\epsstat$: excess risk\vspace*{0.1cm}\hspace{2cm}\hfill
$\epsapprox$: approx. error \hspace{2cm}\hfill
$\kappa$: a condition number \hspace{1cm}\hfill
$D_\infty$: max density ratio to opt., $D_\infty \leq C_\infty$
\\[\medskipamount] \hline
\end{tabular}
\vspace{0.2cm}
\caption{\textbf{Overview of Approximate Methods:}
The suboptimality, $V^\star(s_0) - V^\pi(s_0)$, after $T$ iterations for various
approximate algorithms, which use different notions of approximation
error (sample complexities are not
  directly considered but instead may be thought of as
  part of $\eps_1$ and $\epsstat$. See
Section~\ref{section:related} for further discussion).  Order notation is used to drop
  constants, and we assume $|\Acal|=2$ for ease of exposition.
For approximate dynamic programming methods, the
  relevant error is the worst case, $\ell_\infty$-error in
  approximating a value function, e.g.
  $\eps_\infty=\max_{s,a} | Q^\pi(s,a) -\widehat Q^\pi(s,a)|$, where
  $\widehat Q^\pi$ is what an estimation oracle returns during the
  course of the algorithm.
  The second row (see Lemma 12 in~\cite{antos2008learning}) is a refinement of this
  approach, where $\eps_1$ is an $\ell_1$-average error in fitting the
  value functions under the fitting (state) distribution $\mu$, and, roughly,
  $C_\infty$ is a worst case density ratio between the state
  visitation distribution of any non-stationary
  policy and the fitting distribution $\mu$.
  For Conservative Policy Iteration, $\eps_1$ is a related
  $\ell_1$-average case fitting error with respect to a fitting distribution $\mu$,
   and  $D_\infty $ is as defined as before, in the caption of
   Table~\ref{table:tabular} (see also
   \citep{kakade2002approximately}); here, $D_\infty \leq C_\infty$
   (e.g. see~\cite{Scherrer:API}).
   For NPG, $\epsstat$ and $\epsapprox$ measure the excess
  risk (the regret) and approximation errors in fitting the values. Roughly
  speaking, $\epsstat$ is the excess squared loss relative
  to the best fit (among an appropriately defined parametric class)
  under our fitting distribution (defined with respect to the state distribution $\mu$).
  Here, $\epsapprox$ is the approximation error: the minimal possible
  error (in our parametric class)
  under our fitting
  distribution.
  The condition number $\kappa$ is
  a relative eigenvalue condition between appropriately defined feature covariances
  with respect to the state visitation distribution of an optimal
  policy, $d^{\pi^\star}_{s_0}$, and the state fitting distribution $\mu$.
See
  text for further discussion, and Section~\ref{section:func_approx}
  for precise statements as well as a more general result not explicitly dependent on $D_\infty$.
}
\label{table:approximation}
\end{table}
\renewcommand{\arraystretch}{1}

\vspace{-2mm}
\paragraph{Function Approximation:}
We now summarize our results with regards to policy gradient methods
in the setting where we work with a restricted policy class, which may
not contain the optimal policy. In this sense, these methods can be
viewed as approximate methods. Table~\ref{table:approximation}
provides a summary along with the comparisons to some relevant
approximate dynamic programming methods.

A long line of work in the function approximation setting focuses on mitigating the
worst-case ``$\ell_\infty$'' guarantees that are inherent to
approximate dynamic programming methods~\citep{bertsekas1996neuro}
(see the first row in Table~\ref{table:approximation}).  The reason to focus on average case guarantees is that it
supports the applicability of \emph{supervised machine learning}
methods to solve the underlying approximation problem. This is
because supervised learning methods, like classification and
regression, typically have bounds on the expected error under a
distribution, as opposed to worst-case guarantees over all possible
inputs.

The existing literature largely consists of two lines of provable guarantees that attempt to
mitigate the explicit $\ell_\infty$ error conditions of approximate
dynamic programming: those methods which utilize
a problem dependent parameter (the concentrability
coefficient~\citep{munos2005error}) to provide more refined dynamic
programming guarantees (e.g. see~\cite{munos2005error, szepesvari2005finite,
  antos2008learning, farahmand2010error}) and those which work with a restricted policy
class, making incremental updates, such as Conservative Policy
Iteration (CPI) \citep{kakade2002approximately,scherrer2014local}, Policy Search by
Dynamic Programming (PSDP)~\citep{NIPS2003_2378}, and MD-MPI
~\cite{geist2019theory}. Both styles of approaches give guarantees based
on worst-case density ratios, i.e. they depend on a maximum ratio
between two different densities over the state space.
As discussed in\citep{Scherrer:API}, the assumptions in the latter
class of algorithms are substantially weaker, in that the worst-case
density ratio only depends on the state visitation distribution of an
optimal policy (also see Table~\ref{table:approximation} caption and
Section~\ref{section:related}).

With regards to function approximation, our main contribution is in
providing performance bounds that, in some cases, have milder
dependence on these density ratios.
We precisely quantify an
\emph{approximation/estimation} error decomposition relevant for the
analysis of the natural gradient method; this decomposition is stated
in terms of the \emph{compatible function approximation error} as
introduced in~\cite{sutton1999policy}. More generally, we quantify our
function approximation results in terms of a precisely quantified
transfer error notion, based on approximation error under \emph{distribution shift}. 
Table~\ref{table:approximation} shows a special case of our
convergence rates of NPG, which is
governed by four quantities: $\epsstat$, $\epsapprox$,
$\kappa$, and $D_\infty$.

Let us discuss the important special case of log-linear policies
(i.e. policies that take the softmax of linear functions in a given
feature space) where the relevant quantities are as follows:
$\epsstat$ is a bound on the excess risk (the estimation error) in
fitting linearly parameterized value functions, which can be driven to
$0$ with more samples (at the usual statistical rate of
$O(1/\sqrt{N})$ where $N$ is the number of samples); $\epsapprox$ is
the usual notion of average squared approximation error where the
target function may not be perfectly representable by a linear
function; $\kappa$ can be upper bounded with an inverse dependence on the
minimal eigenvalue of the feature covariance matrix of the
fitting measure (as such it can be viewed as a dimension dependent
quantity but not necessarily state dependent); and $D_\infty$ is as
before.

For the realizable case, where all policies have values which
are linear in the given features (such as in linear
MDP models of
~\citep{jin2019provably,yang2019sample,jiang2017contextual}), we have
that the approximation error $\epsapprox$ is $0$.
Here, our guarantees
yield a fully polynomial and sample efficient convergence guarantee,
provided the condition number $\kappa$ is bounded. Importantly,
there always exists a good (universal) initial measure that ensures $\kappa$ is
bounded by a quantity that is only polynomial in the dimension of
the features, $d$, as opposed to an explicit dependence on the size of the
(infinite) state space (see Remark~\ref{remark:kappa}). Such a guarantee would not be implied by
algorithms which depend on the coefficients $C_\infty$ or
$D_\infty$.\footnote{Bounding $C_\infty$ would require a restriction on the
  dynamics of the MDP (see \cite{chen2019information} and
  Section~\ref{section:related}). Bounding $D_\infty$ would require an
  initial state distribution that is constructed using knowledge of
  $\pi^\star$, through $d^{\pi^\star}$. In contrast,
  $\kappa$ can be made $O(d)$, with an initial state
  distribution that only depends on the geometry of the features (and
  does not depend on any other properties of the MDP). See Remark~\ref{remark:kappa}.}

Our results are also suggestive that a broader class of incremental
algorithms --- such as CPI~\citep{kakade2002approximately},
PSDP~\citep{NIPS2003_2378}, and MD-MPI~\cite{geist2019theory} which
make small changes to the policy from one iteration to the next ---
may also permit a sharper analysis, where the dependence of worst-case
density ratios can be avoided through an appropriate
approximation/estimation decomposition; this is an interesting direction for
future work (a point which we return to in
Section~\ref{sect:discussion}). One significant advantage of NPG is
that the explicit parametric policy representation in NPG (and other
policy gradient methods) leads to a succinct policy representation in
comparison to CPI, PSDP, or related boosting-style
methods~\citep{scherrer2014local}, where the representation complexity
of the policy of the latter class of methods grows linearly in the
number of iterations (since these methods add one policy to the
ensemble per iteration). This representation complexity is
likely why the latter class of algorithms are less widely used in practice.

\section{Related Work}
\label{section:related}

We now discuss related work, roughly in the order which reflects our presentation of results in the previous section.

For the direct policy parameterization in the tabular case, we make
use of a gradient domination-like property, namely any first-order
stationary point of the policy value is approximately optimal up to a
distribution mismatch coefficient.  A variant of this result also appears in Theorem 2 of~\citet{scherrer2014local}, which itself can be viewed as a generalization of the approach in \cite{kakade2002approximately}.
In contrast to CPI~\citep{kakade2002approximately} and the more
general boosting-based approach in \cite{scherrer2014local}, we phrase
this approach as a Polyak-like gradient domination
property~\citep{Pol63} in order to directly allow for the transfer of
any advances in non-convex optimization to policy optimization in
RL. More broadly, it is worth noting the global convergence of policy
gradients for Linear Quadratic Regulators~\citep{fazel2018global} also
goes through a similar proof approach of gradient domination.

Empirically, the recent work of~\citet{ahmed2019understanding}
studies entropy based regularization and shows
the value of regularization in policy optimization, even with exact gradients. This is related to our use
of the log barrier regularization.

For our convergence results of the natural policy gradient algorithm
in the tabular setting, there are close connections between our
results and the works of~\citet{even-dar2009online,
geist2019theory}. \cite{even-dar2009online} provides provable online
regret guarantees in changing MDPs utilizing experts algorithms (also
see ~\citet{NIPS2010_4048,abbasi2019politex}); as a
special case, their MDP Experts Algorithm is equivalent to the natural
policy gradient algorithm with the softmax policy parameterization.
While the convergence result due to \cite{even-dar2009online} was
not specifically designed for this setting, it is instructive to see what it
implies due to the close connections between optimization and
regret~\citep{Cesa-Bianchi:2006:PLG:1137817, shalev2012online}.  The
Mirror Descent-Modified Policy Iteration (MD-MPI) algorithm
~\citep{geist2019theory} with negative entropy as the Bregman
divergence results is an identical algorithm as NPG for softmax
parameterization in the tabular case; Corollary 3
\citep{geist2019theory} applies to our updates, leading to a bound worse
by a $1/(1-\gamma)$ factor and also has logarithmic dependence on $|\Acal|$. Our proof for this case
is concise and may be of independent interest.  Also
worth noting is the Dynamic Policy Programming
of~\cite{Azar:2012:DPP:2503308.2503344}, which is an actor-critic
algorithm with a softmax parameterization; this algorithm, even though
not identical, comes with similar guarantees in terms of its rate (it
is weaker in terms of an additional $1/(1-\gamma)$ factor) than the NPG
algorithm.

We now turn to function approximation, starting with
a discussion of iterative algorithms which make
incremental updates in which the next policy is effectively constrained to
be close to the previous policy, such as in CPI
and PSDP~\citep{NIPS2003_2378}. Here, the work
in~\cite{scherrer2014local} show how CPI is part of broader family of
boosting-style methods. Also,
with regards to PSDP, the work in~\cite{Scherrer:API} shows how PSDP
actually enjoys an improved iteration complexity over CPI, namely
$O(\log 1/\epsopt)$ vs. $O(1/\epsopt^2)$. It is worthwhile to note
that both NPG and projected gradient ascent are also both incremental algorithms.

We now discuss the approximate
dynamic programming results characterized in terms of the
concentrability coefficient. Broadly we use the term approximate dynamic programming to refer to fitted value iteration, fitted policy iteration and more generally generalized policy iteration schemes such as classification-based policy iteration as well, in addition to the classical approximate value/policy iteration works.
While the approximate
dynamic programming results typically require $\ell_\infty$ bounded
errors, which is quite stringent, the notion of
concentrability (originally due to \citep{munos2003error,munos2005error}) permits sharper bounds in terms
of average case function approximation error, provided that the
concentrability coefficient is bounded (e.g. see
~\cite{munos2005error, szepesvari2005finite, antos2008learning, lazaric2016analysis}).
\citet{chen2019information} provide a more detailed  discussion on
this quantity.  Based on this problem dependent
constant being bounded, \cite{munos2005error, szepesvari2005finite},  \cite{antos2008learning} and~\cite{lazaric2016analysis}
provide meaningful sample size and error bounds for approximate
dynamic programming methods, where there is a data collection policy
(under which value-function fitting occurs) that induces a
concentrability coefficient.
In terms of the concentrability coefficient $C_\infty$ and the
``distribution mismatch coefficient'' $D_\infty$ in
Table~\ref{table:approximation} , we have that
$D_\infty \leq C_\infty$, as discussed in~\citep{Scherrer:API} (also
see the table caption). Also, as discussed in~\citet{chen2019information},
a finite
concentrability coefficient is a restriction on the MDP dynamics
itself, while a bounded $D_\infty$ does not require any restrictions
on the MDP dynamics. The more refined quantities
defined by~\citet{farahmand2010error} (for the approximate policy iteration result) partially alleviate some of
these concerns, but their assumptions still implicitly constrain the MDP dynamics, like the finiteness of the concentrability coefficient.

Assuming bounded concentrability coefficient, there are a notable set of
provable average case guarantees for the MD-MPI
algorithm~\citep{geist2019theory} (see
also~\citep{Azar:2012:DPP:2503308.2503344, scherrer2015approximate}),
which are stated in terms of various norms of function approximation
error. MD-MPI is a class of algorithms for approximate planning
under regularized notions of optimality in MDPs. Specifically,
~\citet{geist2019theory} analyze a family of actor-critic style
algorithms, where there are both approximate value functions updates
and approximate policy updates. As a consequence of utilizing
approximate value function updates for the critic, the guarantees of
\citet{geist2019theory} are stated with dependencies on
concentrability coefficients.

When dealing with function approximation,
computational and statistical complexities are
relevant because they determine the effectiveness of
approximate updates with
finite samples. With regards to sample complexity, the work in
~\cite{szepesvari2005finite, antos2008learning} provide finite
sample rates (as discussed above), further generalized to actor-critic
methods in~\cite{Azar:2012:DPP:2503308.2503344,
scherrer2015approximate}. In our policy
optimization approach, the analysis of both computational
and statistical complexities are straightforward, since we can
leverage known statistical and computational results from the
stochastic approximation literature;  in particular, we use the
stochastic projected gradient ascent to obtain a simple, linear time method for the critic estimation step in the
natural policy gradient algorithm.

In terms of the algorithmic updates for the function approximation
setting, our development of NPG bears similarity to the natural
actor-critic algorithm~\cite{Peters:2008:NA:1352927.1352986}, for
which some asymptotic guarantees under finite concentrability
coefficients are obtained in~\citet{bhatnagar2009natural}. While both
updates seek to minimize the compatible function approximation error,
we perform streaming updates based on stochastic optimization using
Monte Carlo estimates for values. In
contrast~\cite{Peters:2008:NA:1352927.1352986} utilize Least Squares
Temporal Difference methods~\citep{boyan1999least} to minimize the
loss. As a consequence, their updates additionally make linear
approximations to the value functions in order to estimate the
advantages; our approach is flexible in allowing for wide family of
smoothly differentiable policy classes (including neural policies).

Finally, we remark on some concurrent works. The work of
~\cite{russoGlobal} provides gradient domination-like conditions under
which there is (asymptotic) global convergence to the optimal
policy. Their results are applicable to the projected gradient ascent
algorithm; they are not applicable to gradient ascent with the softmax
parameterization (see the discussion in Section~\ref{section:softmax}
herein for the analysis challenges). \cite{russoGlobal} also provide global convergence results beyond
MDPs. Also, \cite{caiTRPO} provide an analysis of the TRPO
algorithm~\citep{schulman2015trust} with neural network
parameterizations, which bears resemblance to our natural policy
gradient analysis. In particular, \cite{caiTRPO} utilize ideas from
both~\cite{even-dar2009online} (with a mirror descent style of
analysis) along with ~\cite{caiTD} (to handle approximation with
neural networks) to provide conditions under which TRPO returns a near
optimal policy. \cite{caiTRPO} do not explicitly consider the case
where the policy class is not complete (i.e when there is
approximation). Another related work of~\cite{shani2019adaptive} considers the TRPO
algorithm and provides theoretical guarantees in the tabular case;
their convergence rates with exact updates are $O(1/\sqrt{T})$ for the
(unregularized) objective function of interest; they also provide faster rates on
a modified (regularized) objective function. They do not consider the case of infinite state
spaces and function approximation. The closely related recent
papers~\citep{abbasi2019politex,abbasi2019exploration} also consider
closely related algorithms to the Natural Policy Gradient approach
studied here, in an infinite horizon, average reward
setting. Specifically, the \textsc{EE-Politex} algorithm is closely
related to the Q-NPG algorithm which we study in
Section~\ref{sec:q-npg}, though our approach is in the discounted
setting. We adopt the name Q-NPG to capture its close relationship
with the NPG algorithm, with the main difference being the use of
function approximation for the $Q$-function instead of advantages. We
refer the reader to Section~\ref{sec:q-npg} (and Remark~\ref{remark:politex}) for more discussion of the
technical differences between the two works. 

\section{Setting}
\label{section:setting}

A (finite) Markov Decision Process (MDP) $M = (\Scal, \Acal, P, r, \gamma,\rho)$
is specified by:
a finite state space $\Scal$;  a finite action space $\Acal$; a
transition model $P$ where
$P(s^\prime | s, a)$ is the probability of transitioning into state $s^\prime$ upon taking action $a$
in state $s$; a reward function $r: \Scal\times \Acal \to [0,1]$ where $r(s,a)$ is
the immediate reward associated with taking action $a$ in state $s$;
a discount factor $\gamma \in [0, 1)$; a starting state distribution
$\rho$ over $\Scal$.

A deterministic, stationary policy $\pi: \Scal \to \Acal$
specifies a decision-making strategy in which the agent chooses
actions adaptively based on the current state, i.e., $a_t =
\pi(s_t)$. The agent may also choose actions according
to a stochastic policy $\pi: \Scal \to \Delta(\Acal)$ (where
$\Delta(\Acal)$ is the probability simplex over $\Acal$), and,
overloading notation, we write $a_t \sim \pi(\cdot|s_t)$.

A policy induces a distribution over trajectories $\tau =
(s_t, a_t, r_t)_{t=0}^\infty$, where $s_0$ is drawn from the starting
state distribution $\rho$, and, for all subsequent timesteps $t$, $a_t \sim \pi(\cdot | s_t)$ and
$s_{t+1} \sim P(\cdot | s_t, a_t)$. The value function $V^\pi: \Scal \to \R$ is
defined as the discounted sum of future rewards starting at state $s$
and executing $\pi$, i.e.
\[
V^\pi(s) := \E\left[\sum_{t=0}^\infty \gamma^t  r(s_t, a_t)
  | \pi, s_0 = s\right] \, ,
\]
where the expectation is with respect to the randomness of the
trajectory $\tau$ induced by $\pi$ in $M$. Since we assume that $r(s,a) \in
[0,1]$, we have $0\leq V^\pi(s) \leq \frac{1}{1-\gamma}$.
We overload notation and define $V^\pi(\rho)$ as the expected value
under the initial state distribution $\rho$, i.e.
\begin{align*}
V^\pi(\rho) := \E_{s_0\sim \rho} [ V^\pi(s_0)] .
\end{align*}

%


The action-value (or Q-value) function $Q^\pi: \Scal
\times \Acal \to \R$ and the \emph{advantage} function $A^\pi: \Scal
\times \Acal \to \R$ are defined as:
\[
Q^\pi(s,a) = \E\left[\sum_{t=0}^\infty \gamma^t  r(s_t, a_t) | \pi,
  s_0 = s, a_0 = a \right] , \quad
A^\pi(s,a) := Q^\pi(s,a)-V^\pi(s) \, .
\]

The goal of the agent is to find a policy $\pi$ that maximizes the expected value from the initial state,
i.e. the optimization problem the agent seeks to solve is:
\begin{align} \label{eq:expected_return}
  \max_\pi V^{\pi}(\rho),
\end{align}
where the $\max$ is over all policies.
The famous theorem of \cite{bellman1959functional} shows there
exists a policy $\pi^\star$ which simultaneously maximizes
$V^{\pi}(s_0)$, for all states $s_0\in \Scal$.

\paragraph{Policy Parameterizations.}
This work studies ascent methods
for the optimization problem:
\[
  \max_{\theta\in \Theta} V^{\pi_\theta}(\rho) ,
\]
where $\{\pi_\theta|\theta \in \Theta\}$ is some class
of parametric (stochastic) policies.
We consider a number of different policy classes. The first two are
\emph{complete} in the sense that any stochastic policy can be
represented in the class. The final class may be restrictive. These
classes are as follows:
\begin{itemize}
\item \emph{Direct parameterization:} The policies are parameterized by
\begin{equation}\label{eq:direct}
\pi_\theta (a| s) = \theta_{s,a} ,
\end{equation}
where   $\theta \in \Delta(\cA) ^{|\Sset|}$, i.e. $\theta$ is subject to
$\theta_{s,a} \geq 0$ and $\sum_{a \in \Acal} \theta_{s,a} = 1$ for all $s\in\Scal$ and $a \in \Acal$.
\item \emph{Softmax parameterization:} For unconstrained $\theta \in \R^{|\Scal||\Acal|}$,
  \begin{equation}
    \pi_\theta(a| s) = \frac{\exp(\theta_{s,a})}{\sum_{a'\in\Acal} \exp(\theta_{s,a'})}.
    \label{eq:softmax}
  \end{equation}
The softmax parameterization is also complete.
\item \emph{Restricted parameterizations:} We also study
parametric  classes $\{\pi_\theta|\theta \in \Theta\}$ that may not contain all
  stochastic policies. In particular, we pay close attention to both
  log-linear policy classes and neural policy classes (see Section~\ref{section:func_approx}). Here, the best we may hope for is an agnostic result where
  we do as well as the best policy in this class.
\end{itemize}
While the softmax parameterization is the more natural parametrization among
the two complete policy classes, it is
also informative to consider the direct parameterization.

It is worth explicitly noting that $V^{\pi_\theta} (s)$ is non-concave in $\theta$ for both
the direct and the softmax parameterizations, so the standard tools of
convex optimization are not applicable. For completeness, we formalize
this as follows (with a proof in Appendix~\ref{app:setting}, along
with an example in Figure~\ref{fig:noncon}):

\begin{lemma}
There is an MDP $M$ (described in Figure~\ref{fig:noncon}) such that the optimization problem $
V^{\pi_\theta} (s)$ is not concave for both the direct and softmax
parameterizations.
\label{lemma:softmax-noncon}
\end{lemma}

	\begin{figure}
		\centering
		\begin{minipage}[b]{0.45\textwidth}
			\centering
			\begin{tikzpicture}[->,>=stealth,node distance=2.0cm]
			\node[state] (s1) {$s_1$};
			\node[state, right of=s1] (s2) {$s_2$};
			\node[state, accepting, above of=s2] (s4) {$s_4$};
			\node[state, accepting, right of=s2] (s5) {$s_5$};
			\node[state, accepting, above of=s1] (s3) {$s_3$};
			\draw (s4) edge[loop above] node{$0$} (s4)
			(s5) edge[loop above] node{$0$} (s5)
			(s3) edge[loop above] node{$0$} (s3)
			(s1) edge[above] node{$0$} (s2)
			(s1) edge[left] node{$0$} (s3)
			(s2) edge[right] node{$r>0$} (s4)
			(s2) edge[above] node{$0$} (s5);
			\end{tikzpicture}
			\caption{(Non-concavity example) A deterministic MDP
                          corresponding to Lemma
                          \ref{lemma:softmax-noncon} where  $
                          V^{\pi_\theta} (s)$ is not concave. 
                          Numbers on arrows represent the rewards for
                          each action.} 
			\label{fig:noncon}
		\end{minipage}\hfill
		\begin{minipage}[b]{0.50\textwidth}
			\centering
			\begin{tikzpicture}[->,>=stealth,node distance=2.0cm]
		\node[state] (s0) {$s_0$};
		\node[state, right of=s0] (s1) {$s_1$};
		\node[state, draw=none] (d1) [right=of s1, xshift=-2.0cm] {$\cdots$};
		\node[state, right of=d1] (sh) {$s_{H}$};
		\node[state, accepting, right of=sh] (shp) {$s_{H+1}$};
		\draw (shp) edge[loop above] node{$a_1$} (shp)
		(s0) edge[bend right=30, below] node{$a_1$} (s1)
		(d1) edge[bend right=30, below] node{$a_1$} (sh)
		(sh) edge[bend right=0, below] node{$a_1$} (shp)
		(s1) edge[bend right=0, above] node{$a_2$} (s0)
		(sh) edge[bend right=0, above] node{$a_2$} (d1)
		(s1) edge[bend right=40, above] node{$a_3$} (s0)
		(s1) edge[bend right=90, above] node{$a_4$} (s0)
		(sh) edge[bend right=40, above] node{$a_3$} (d1)
		(sh) edge[bend right=90, above] node{$a_4$} (d1);
			\end{tikzpicture}
			\caption{(Vanishing gradient example) A deterministic, chain MDP of length $H+2$. We
                          consider a policy where $\pi(a | s_i) =
                          \theta_{s_i,a}$ for $i=1,2,\ldots,H$. Rewards
                          are $0$ everywhere other than $r(s_{H+1},
                          a_1) = 1$. See Proposition~\ref{proposition:small_grad}.}
			\label{fig:chain}
		\end{minipage}
	\end{figure}

\paragraph{Policy gradients.}
In order to introduce these
methods, it is useful to define the discounted state
visitation distribution $d_{s_0}^\pi$ of a policy $\pi$ as:
\begin{equation}
d_{s_0}^\pi(s) := (1-\gamma) \sum_{t=0}^\infty \gamma^t {\Pr}^\pi(s_t=s|s_0),
\label{eqn:dpi}
\end{equation}
where $\Pr^\pi(s_t=s|s_0)$ is the state visitation
probability that $s_t=s$, after we execute $\pi$ starting at state
$s_0$. Again, we overload notation and write:
\begin{align*}
d_{\rho}^\pi(s) = \E_{s_0\sim \rho} \left[d_{s_0}^\pi(s)\right] \, ,
\end{align*}
where $d_{\rho}^\pi$ is the discounted state
visitation distribution under initial distribution $\rho$.

The policy gradient functional form (see e.g.~\citet{williams1992simple,
  sutton1999policy}) is then:
\begin{eqnarray}\label{eqn:Qpg}
  \nabla_\theta V^{\pi_\theta}(s_0) =
\frac{1}{1-\gamma} \, \E_{s \sim d_{s_0}^{\pi_\theta} }\E_{a\sim \pi_\theta(\cdot | s) }
\big[\nabla_\theta \log
\pi_{\theta}(a| s) Q^{\pi_\theta}(s,a)\big] .
\end{eqnarray}
Furthermore, if we are working with a differentiable  parameterization of
$\pi_\theta(\cdot|s)$ that explicitly constrains $\pi_\theta(\cdot|s)$ to be
in the simplex, i.e. $\pi_\theta \in \Delta(\cA) ^{|\Sset|}$ for all
$\theta$,  then we also have:
\begin{eqnarray}\label{eqn:Apg}
\nabla_\theta V^{\pi_\theta}(s_0) =  \frac{1}{1-\gamma} \, \E_{s \sim d_{s_0}^{\pi_\theta} }\E_{a\sim \pi_\theta(\cdot | s) }
\big[\nabla_\theta \log \pi_{\theta}(a| s) A^{\pi_\theta}(s,a)\big].
\end{eqnarray}
Note the above gradient expression (Equation~\ref{eqn:Apg}) does not hold for the direct
parameterization, while Equation~\ref{eqn:Qpg} is
valid.~\footnote{This is due to
$\sum_a \nabla_\theta \pi_\theta(a|s) = 0$ 
  not explicitly being maintained by the direct parameterization.}

\paragraph{The performance difference lemma.} The following lemma is
helpful throughout:
\begin{lemma}\label{lemma:perf_diff}
(The performance difference lemma~\citep{kakade2002approximately})
For all policies $\pi, \pi^\prime$ and states
$s_0$,
\begin{eqnarray*}
  V^\pi(s_0) - V^{\pi^\prime}(s_0)
&=& \frac{1}{1-\gamma}\E_{s\sim d_{s_0}^\pi }\E_{a\sim \pi(\cdot|s) }
\left[A^{\pi^\prime}(s,a)\right].
\end{eqnarray*}
\end{lemma}

For completeness, we provide a proof in Appendix~\ref{app:setting}.

\paragraph{The distribution mismatch coefficient.} We often
characterize the difficulty of the exploration problem faced by our
policy optimization algorithms when maximizing the objective
$V^\pi(\mu)$ through the following notion of \emph{distribution
  mismatch coefficient}.

\begin{definition}[Distribution mismatch coefficient]
Given a policy $\pi$ and measures $\rho, \mu\in
\Delta(\Scal)$, we refer to
$\BigNorm{\frac{d_{\rho}^{\pi}}{\mu}}_\infty$ as the
\emph{distribution mismatch coefficient} of $\pi$ relative to
$\mu$. Here, $\frac{d_{\rho}^{\pi}}{\mu}$ denotes componentwise division.
\label{defn:dist_mismatch}
\end{definition}

We often instantiate this coefficient
with $\mu$ as the initial state distribution used in a policy
optimization algorithm, $\rho$ as the distribution to measure
the sub-optimality of our policy (this is the start state distribution
of interest), and where $\pi$ above is often chosen to be $\pi^\star \in
\argmax_{\pi \in \Pi} V^\pi(\rho)$, given a policy class $\Pi$.

\paragraph{Notation.}
Following convention, we use $V^\star$ and $Q^\star$ to denote
$V^{\pi^\star}$ and $Q^{\pi^\star}$ respectively.
For iterative algorithms which obtain policy parameters $\theta^{(t)}$ at
iteration $t$, we let $\pi^{(t)}$, $V^{(t)}$ and $A^{(t)}$ denote the
corresponding quantities parameterized by $\theta^{(t)}$,
i.e. $\pi_{\theta^{(t)}}$, $V^{\theta^{(t)}}$ and $A^{\theta^{(t)}}$, respectively. For vectors $u$
and $v$, we use $\tfrac{u}{v}$ to denote the componentwise ratio;
$u\geq v$ denotes a componentwise inequality; we
use the standard convention where $\|v\|_2=\sqrt{\sum_i v_i^2}$, $\|v\|_1=\sum_i |v_i|$,
and $\|v\|_\infty=\max_i |v_i|$. 

\section{Warmup: Constrained Tabular Parameterization}
\label{section:direct}

Our starting point is, arguably, the simplest first-order
method: we directly take gradient ascent updates on the policy simplex
itself and then project back onto the simplex if the constraints are
violated after a gradient update. This algorithm is projected gradient
ascent on the direct policy parametrization of the MDP, where the
parameters are the state-action probabilities, i.e.
$\theta_{s,a} = \pi_\theta(a|s)$ (see~\eqref{eq:direct}). As noted in
Lemma ~\ref{lemma:softmax-noncon}, $V^{\pi_\theta}(s)$ is non-concave
in the parameters $\pi_\theta$.
Here, we first prove that $V^{\pi_\theta}(\mu)$ satisfies a
Polyak-like gradient domination condition~\citep{Pol63}, and this tool
helps in providing convergence rates. The basic approach was also used
in the analysis of CPI~\citep{kakade2002approximately}; related gradient
domination-like lemmas also appeared in \cite{scherrer2014local}.

It is instructive to consider this special case
due to the connections it makes to the non-convex optimization
literature.
We also provide a lower bound
that rules out algorithms whose runtime appeals to
the curvature of saddle points (e.g.~\citep{NesterovP06,GeHJY15,DBLP:conf/icml/Jin0NKJ17}).


For the direct policy parametrization where $\theta_{s,a} =
\pi_\theta(a|s)$, the gradient is:
\begin{equation}\label{eq:grad_direct}
\frac{\partial V^{\pi}(\mu)}{\partial \pi(a|s)}
= \frac{1}{1-\gamma} d^{\pi}_{\mu} (s) Q^{\pi}(s,a),
\end{equation}
using~\eqref{eqn:Qpg}.
In particular, for this
parameterization, we may
write $\nabla_{\pi} V^{\pi}(\mu)$ instead of
$\nabla_{\theta} V^{\pi_\theta}(\mu)$.

\subsection{Gradient Domination}

Informally, we say a function $f(\theta)$ satisfies a gradient
domination property if for all $\theta \in \Theta$,
\[
f(\theta^\star)- f(\theta) = O( G(\theta)),
\]
where $\theta^\star \in  \argmax_{\theta^\prime\in\Theta}f(\theta')$
and where $G({\theta})$ is some suitable scalar notion of first-order
stationarity, which can be considered a measure of
how large the gradient is (see
~\citep{karimi2016linear,bolte2007lojasiewicz,attouch2010proximal}).
Thus if one can find a $\theta$ that is
(approximately) a first-order stationary point, then the parameter
$\theta$ will be near optimal (in terms of function value). Such
conditions are a standard device to establishing global convergence in
non-convex optimization, as they effectively rule out the presence of
bad critical points. In other words, given such a condition, quantifying the
convergence rate for a specific algorithm, like say projected gradient
ascent, will require quantifying the rate of its convergence to a
first-order stationary point, for which one can invoke standard
results from the optimization literature.

The following lemma shows that the direct policy parameterization
satisfies a notion of gradient domination. This is the basic approach
used in the analysis of CPI~\citep{kakade2002approximately}; a variant
of this lemma also appears in
\cite{scherrer2014local}. We give a proof for completeness.

Even though we are interested in the value $V^{\pi}(\rho)$, it is helpful to
consider the gradient with respect to another state distribution
$\mu\in \Delta(\Scal)$.

\begin{lemma}[Gradient domination]
	\label{thm:first}
For the direct policy parameterization (as in~\eqref{eq:direct}), for
all state distributions $\mu,\rho\in \Delta(\Scal)$, we have
\begin{eqnarray*}
	V^\star (\rho) -V^{\pi}(\rho) &\le &
\, \left\|\frac{d_{\rho}^{\pi^\star}}{d_{\mu}^{\pi} }\right\|_\infty
\max_{\bar\pi }~(\bar\pi - \pi)^\top \nabla_\pi V^{\pi}(\mu)\\
&\le &
\frac{1}{1-\gamma}\left\|\frac{d_{\rho}^{\pi^\star}}{\mu}\right\|_\infty
\max_{\bar\pi }~(\bar\pi - \pi)^\top \nabla_\pi V^{\pi}(\mu),
\end{eqnarray*}
where the max is over the set of all policies, i.e. $\bar\pi \in
\Delta(\cA) ^{|\Sset|}$.
\end{lemma}


Before we provide the proof, a few comments are in order with regards
to the performance measure $\rho$ and the optimization measure $\mu$.
Subtly, note that although the gradient is with respect to
$V^{\pi}(\mu)$, the final guarantee applies to \emph{all}
distributions $\rho$. The significance is that even though we may be
interested in our performance under $\rho$, it may be helpful to
optimize under the distribution $\mu$.  To see this, note the lemma
shows that a sufficiently small gradient magnitude in the feasible
directions implies the policy is nearly optimal in terms of its value,
but only if the state distribution of $\pi$, i.e. $d_{\mu}^{\pi} $,
adequately covers the state distribution of some optimal policy
$\pi^\star$.  Here, it is also worth recalling the theorem of
\cite{bellman1959functional} which shows there exists a single policy
$\pi^\star$ that is simultaneously optimal for all starting states
$s_0$. Note that the hardness of the exploration problem is captured
through the distribution mismatch coefficient
(Definition~\ref{defn:dist_mismatch}).

\begin{proof}[\textbf{of Lemma~\ref{thm:first}}]
By the performance difference lemma (Lemma~\ref{lemma:perf_diff}),
	\begin{align}
	V^\star (\rho) - V^{\pi}(\rho) &=\frac{1}{1-\gamma} \sum_{s,a} d_{\rho}^{\pi^\star}(s) \pi^\star (a|s) A^{\pi}(s,a) \nonumber\\
	&\le \frac{1}{1-\gamma} \sum_{s,a} d_{\rho}^{\pi^\star}(s) \max_{\bar a }A^{\pi}(s,\bar a)
	\nonumber \\
	&=  \frac{1}{1-\gamma} \sum_{s} \frac{d_{\rho}^{\pi^\star}(s)}{d^{\pi}_{\mu} (s)} \cdot d^{\pi}_{\mu}(s) \max_{\bar a } A^{\pi} (s,\bar a) \nonumber\\
	&\le \frac{1}{1-\gamma}  \left( \max_{s} \frac{d_{\rho}^{\pi^\star}(s)}{d^{\pi}_{\mu} (s)}\right) \sum_{s} d^{\pi}_{\mu} (s) \max_{\bar a} A^{\pi}(s,\bar a),
   \label{eq:PD-calculation}
\end{align}
where the last inequality follows since $\max_{\bar a} A^\pi(s,\bar{a}) \geq 0$ for all states $s$ and policies $\pi$.
We wish to upper bound \eqref{eq:PD-calculation}. We then have:
\begin{align*}
\sum_s \frac{d^{\pi}_{\mu}(s)}{1-\gamma} \max_{\bar  a}A^{\pi} (s,\bar a)
 &=\max_{\bar \pi \in \Delta(\cA) ^{|\Sset|}}\sum_{s,a}
    \frac{d^{\pi}_{\mu}(s) }{1-\gamma} \bar \pi(a|s)  A^{{\pi}}(s,a)\\
&=\max_{\bar \pi \in \Delta(\cA) ^{|\Sset|}}\sum_{s,a}
    \frac{d^{\pi}_{\mu}(s) }{1-\gamma}(\bar \pi(a|s) -  {\pi}(a|s))A^{{\pi}}(s,a)\\
&=\max_{\bar \pi \in \Delta(\cA) ^{|\Sset|}}\sum_{s,a}
    \frac{d^{\pi}_{\mu}(s) }{1-\gamma}(\bar \pi(a|s) -  {\pi}(a|s))Q^{{\pi}}(s,a)\\
&=\max_{\bar\pi \in \Delta(\cA) ^{|\Sset|}}~(\bar\pi - {\pi})^\top \nabla_\pi V^{{\pi}}(\mu)
\end{align*}
where the first step follows since $\max_{\bar \pi}$ is attained at an
action which maximizes $A^{\pi}(s,\cdot)$ (per state);
the second step follows as $\sum_a {\pi}(a|s) A^{\pi} (s,a) =0$;
the third step uses \mbox{$\sum_a (\bar\pi(a|s) - \pi(a|s))V^\pi(s) = 0$} for
all $s$; and the final step follows from the gradient expression (see
\eqref{eq:grad_direct}).
Using this in
        \eqref{eq:PD-calculation},
	\begin{align*}
	V^\star (\rho) - V^{\pi} (\rho) &\le  \Norm{ \frac{d_{\rho}^{\pi^\star}}{d^{\pi}_{\mu} }}_\infty \max_{\bar\pi \in \Delta(\cA) ^{|\Sset|}}~(\bar\pi - {\pi})^\top \nabla_\pi V^{{\pi}}(\mu)\\
	&\le \frac{1}{1-\gamma}
   \Norm{\frac{d_{\rho}^{\pi^\star}}{\mu}}_\infty \max_{\bar\pi \in \Delta(\cA) ^{|\Sset|}}~(\bar\pi - {\pi})^\top \nabla_\pi V^{{\pi}}(\mu).
	\end{align*}
where the last step follows due to $\max_{\bar\pi \in \Delta(\cA)
  ^{|\Sset|}}~(\bar\pi - {\pi})^\top \nabla_\pi V^{{\pi}}(\mu) \geq 0$ for
any policy $\pi$ and $d^{\pi}_{\mu} (s) \ge (1-\gamma)
\mu(s)$ (see \eqref{eqn:dpi}).
\end{proof}

In a sense, the use of an appropriate $\mu$ circumvents the issues of
strategic exploration.  It is natural to ask whether this additional
term is necessary, a question which we return to. First, we provide a
convergence rate for the projected gradient ascent algorithm.

\subsection{Convergence Rates for Projected Gradient Ascent}
\label{section:conv_pgd}

Using this notion of gradient domination, we now give an iteration
complexity bound for projected gradient ascent over the space of
stochastic policies, i.e. over $\Delta(\cA) ^{|\Sset|}$. The projected gradient ascent algorithm
updates
\begin{equation}
\label{eq:pgd-update}
	\pi^{(t+1)} = P_{\Delta(\cA) ^{|S|}}(\pi^{(t)} + \eta \nabla_\pi V^{(t)}(\mu)),
\end{equation}
where $P_{\Delta(\cA) ^{|\Sset|}}$ is the projection onto $\Delta(\cA)
^{|\Sset|}$ in the Euclidean norm.

\begin{theorem}
	\label{thm:proj-gd}
The projected gradient ascent algorithm~\eqref{eq:pgd-update} on
$V^{\pi}(\mu)$ with stepsize $\eta = \frac{(1-\gamma)^3}{2\gamma
  |\Acal|}$ satisfies for all distributions
$\rho\in \Delta(\Scal)$,
\[
	\min_{t< T} \left\{ V^\star (\rho) -V^{(t)}(\rho)\right\} \le \epsilon  \quad \mbox{whenever} \quad T > \frac{64 \gamma |\Scal||\Acal|}{(1-\gamma)^6\epsilon^2} \left\|\frac{d_{\rho}^{\pi^\star}}{\mu}\right\|^2_\infty.
\]	
\end{theorem}

A proof is provided  in Appendix~\ref{app:pgd}. The proof first invokes a
standard iteration complexity result of projected gradient ascent to
show that the gradient magnitude with respect to all feasible
directions is small. More concretely, we show the policy is
$\epsilon$-stationary\footnote{See Appendix~\ref{app:pgd} for
  discussion on this definition.}, that is, for all ${\pi_\theta}+\delta
\in \Delta(\cA) ^{|\Sset|}$ and $\|\delta\|_2\le 1$, $\delta^\top \nabla_\pi
V^{{\pi_\theta}}(\mu) \leq \epsilon$. We then use
Lemma~\ref{thm:first} to complete the proof.

Note that the guarantee we provide is for the best policy found over the $T$ rounds,
which we obtain from a bound on the average norm of the gradients. This type of a guarantee is standard in the non-convex optimization literature, where an average regret bound cannot be used to extract a single good solution, e.g. by averaging. In the context of policy optimization, this is not a serious limitation as we collect on-policy trajectories for each policy in doing sample-based gradient estimation, and these samples can be also used to estimate the policy's value. Note that the evaluation step is not required for every policy, and can also happen on a schedule, though we still need to evaluate $O(T)$ policies to obtain the convergence rates described here.

\subsection{A Lower Bound: Vanishing Gradients and Saddle Points}
\label{sec:vanishing-gradients}

To understand the necessity of the distribution mismatch coefficient in
Lemma~\ref{thm:first} and Theorem~\ref{thm:proj-gd}, let us first give an informal argument that
some condition on the state distribution of $\pi$, or equivalently
$\mu$, is necessary for stationarity to imply optimality. For example,
in a sparse-reward MDP (where the agent is only rewarded upon visiting
some small set of states), a policy that does not visit \emph{any}
rewarding states will have zero gradient, even though it is
arbitrarily suboptimal in terms of values. Below, we give a more
quantitative version of this intuition, which demonstrates that even
if $\pi$ chooses all actions with reasonable probabilities (and hence
the agent will visit all states if the MDP is connected), then there
is an MDP where a large fraction of the policies $\pi$ have vanishingly small gradients, and yet
these policies are highly suboptimal in terms of their value.

Concretely, consider the chain MDP of length $H+2$ shown in
Figure~\ref{fig:chain}. The starting state of interest is state $s_0$
and the discount factor $\gamma = H/(H+1)$.
Suppose we work with the direct parameterization, where $\pi_\theta(a | s) =
\theta_{s,a}$ for $a = a_1,a_2,a_3$ and $\pi_\theta(a_4 | s) = 1 -
\theta_{s,a_1} - \theta_{s,a_2} - \theta_{s,a_3}$. Note we do not
over-parameterize the policy.
For this MDP and policy structure, if we were to initialize the
probabilities over actions, say deterministically, then there is an
MDP (obtained by permuting the actions) where all the probabilities
for $a_1$ will be less than $1/4$.

The following result not only shows that the gradient is exponentially
small in $H$, it also shows that many higher order
derivatives, up to $O(H/\log H)$, are also exponentially
small in $H$.

\begin{proposition} [Vanishing gradients at suboptimal parameters]
\label{proposition:small_grad}
Consider the chain MDP of Figure~\ref{fig:chain}, with $H+2$ states, $\gamma =
H/(H+1)$, and with the direct
policy parameterization (with $3|\Scal|$ parameters, as described in the
text above).
Suppose $\theta$ is such that $0<\theta<1$ (componentwise) and
$\theta_{s,a_1} < 1/4$ (for all states $s$). For all
$k \leq  \frac{H}{40\log(2 H)} - 1$, we have $\norm{\nabla_\theta^k V^{\pi_\theta}(s_0)} \leq
(1/3)^{H/4}$, where $\nabla_\theta^k V^{\pi_\theta}(s_0)$ is a tensor of the $k_{th}$
order derivatives of $V^{\pi_\theta}(s_0)$ and the norm is the operator norm of the
tensor.\footnote{The operator norm of a $k_{th}$-order tensor $J \in
  \R^{d^{\otimes k}}$ is defined as $\sup_{u_1,\ldots,u_k \in
    \R^d~:~\|u_i\|_2 = 1} \langle J, u_1\otimes\ldots\otimes
  u_d\rangle$.}  Furthermore, \mbox{$V^{\star}(s_0) - V^{\pi_\theta}(s_0) \geq
(H+1)/8-(H+1)^2/3^H$}.
\end{proposition}

This lemma also suggests that results in the
non-convex optimization literature, on escaping from saddle
points, e.g.~\citep{NesterovP06,GeHJY15,DBLP:conf/icml/Jin0NKJ17}, do not
directly imply global convergence due to that the higher order
derivatives are small.

\begin{remark}
(Exact vs. Approximate Gradients) The chain MDP of Figure~\ref{fig:chain},
is a common example where \emph{sample} based estimates of gradients
will be $0$ under random exploration strategies; there is an
exponentially small in $H$  chance of hitting the goal state under a
random exploration strategy. Note
that this lemma is with regards to \emph{exact} gradients. This
suggests that even with exact computations (along with using exact higher
order derivatives) we might expect numerical instabilities.
\end{remark}

\begin{remark} (Comparison with the upper bound) The lower bound does
  not contradict the upper bound of Theorem~\ref{thm:first} (where a
  small gradient is turned into a small policy suboptimality bound),
  as the distribution mismatch coefficient, as defined in
  Definition~\ref{defn:dist_mismatch}, could be infinite in the chain MDP of
  Figure~\ref{fig:chain}, since the start-state distribution is
  concentrated on one state only. More generally, for any policy with
  $\theta_{s,a_1} < 1/4$ in all states $s$,
  $\left\|\frac{d^{\pi^\star}_{\rho}}{d^{\pi_\theta}_\rho}\right\|_\infty
  = \Omega(4^H)$. 
\end{remark}

\begin{remark} (Comparison with information-theoretic lower bounds)
  The lower bound here is \emph{not information theoretic}, in that it
  does not present a hard problem instance for all algorithms. Indeed,
  exploration algorithms for tabular MDPs starting from
  $E^3$~\citep{kearns2002optimal}, RMAX~\citep{brafman2003r} and
  several subsequent works yield polynomial sample complexities for
  the chain MDP. Proposition~\ref{proposition:small_grad} should be
  interpreted as a hardness result for the specific class of policy
  gradient like approaches that search for a policy with a small
  policy gradient, as these methods will find the initial parameters
  to be valid in terms of the size of (several orders of)
  gradients. In particular, it precludes any meaningful claims on
  global optimality, based just on the size of the policy gradients,
  without additional assumptions as discussed in the previous remark. 
\end{remark}

The proof is provided in Appendix~\ref{app:lb}.
The lemma illustrates that lack of good exploration can indeed be
detrimental in policy gradient algorithms, since the gradient can be
small either due to $\pi$ being near-optimal, or, simply because $\pi$ does
not visit advantageous states often enough. In this sense, it also
demonstrates the necessity of the distribution mismatch coefficient in
Lemma~\ref{thm:first}.

\section{The Softmax Tabular Parameterization}
\label{section:softmax}
We now consider the softmax policy parameterization~\eqref{eq:softmax}.
Here, we still have a
non-concave optimization problem in general, as shown in Lemma
\ref{lemma:softmax-noncon}, though we do show that global optimality can
be reached under certain regularity conditions. From a practical perspective, the softmax parameterization
of policies is preferable to the direct
parameterization, since the parameters $\theta$ are unconstrained
and standard unconstrained optimization algorithms can be employed.
However, optimization
over this policy class creates other challenges as we study in this
section, as the optimal policy (which is deterministic) is attained by
sending the parameters to infinity.

We study three algorithms for this
problem. The first performs direct policy gradient ascent on the
objective without modification, while the second adds a log barrier
regularizer to keep the parameters from becoming too large, as a
means to ensure adequate exploration.
Finally, we study the natural policy gradient
algorithm and establish a global optimality result
with no dependence on the distribution mismatch coefficient or
dimension-dependent factors.

For the softmax parameterization, the gradient takes the form:
\begin{equation}\label{eq:softmax_grad}
\frac{\partial V^{\pi_\theta}(\mu)}{\partial \theta_{s,a}} =
\frac{1}{1-\gamma} d^{\pi_\theta}_{\mu}(s)\pi_{\theta}(a|s)A^{\pi_\theta}(s,a)
\end{equation}
(see Lemma \ref{lemma:softmax-grad} for a proof).


\subsection{Asymptotic Convergence, without Regularization}
\label{sec:asymptotic}


Due to the exponential scaling with the
parameters $\theta$ in the softmax parameterization, \emph{any} policy
that is nearly deterministic will have gradients close to $0$.  In
spite of this difficulty, we provide a positive result that gradient ascent
asymptotically converges to the global optimum for the softmax
parameterization.

The update rule for gradient ascent is:
\begin{equation}    \label{eqn:gd-softmax}
    \theta^{(t+1)} = \theta^{(t)} + \eta \nabla_\theta V^{(t)}(\mu).
\end{equation}

\begin{theorem} [Global convergence for softmax parameterization]
	\label{thm:glb-softmax}
Assume we follow the gradient ascent update rule as specified
in Equation~\eqref{eqn:gd-softmax} and that the distribution
$\mu$ is strictly positive i.e. $\mu(s)>0$ for all states $s$.
Suppose $\eta \leq \frac{(1-\gamma)^3}{8}$, then we have that for all
states $s$,
$V^{(t)}(s) \rightarrow V^\star(s)$ as $t \rightarrow \infty$.
\end{theorem}

\begin{remark} \label{remark:strict}
(Strict positivity of $\mu$ and exploration) Theorem~\ref{thm:glb-softmax} assumed
that optimization distribution $\mu$ was \emph{strictly} positive,
i.e. $\mu(s)>0$ for all states $s$. We leave it is an open question of whether or not
gradient ascent will globally converge if this condition is not met. 
The concern is that if this condition is not met, then gradient ascent
may not globally converge due to that $d^{\pi_\theta}_{\mu}(s)$
effectively scales down the learning rate for the parameters
associated with state $s$ (see \eqref{eq:softmax_grad}).
\end{remark}

The complete proof is provided in the Appendix~\ref{app:gd-softmax}.  We
now discuss the subtleties in the proof and show why the softmax
parameterization precludes a direct application of the gradient
domination lemma. In order to utilize the gradient domination
property (in Lemma~\ref{thm:first}), we would desire to show that:
$ \nabla_\pi V^{\pi}(\mu) \rightarrow 0 $. However, using the
functional form of the softmax parameterization (see Lemma
\ref{lemma:softmax-grad}) and \eqref{eq:grad_direct}, we have that:
\begin{equation*}
\frac{\partial V^{\pi_\theta}(\mu)}{\partial \theta_{s,a}} =
\frac{1}{1-\gamma} d^{\pi_\theta}_{\mu}(s)\pi_{\theta}(a|s)A^{\pi_\theta}(s,a)
= \pi_\theta(a|s) \, \frac{\partial V^{\pi_\theta}(\mu)}{\partial \pi_\theta(a|s)} .
\end{equation*}
Hence, we see that
even if $\nabla_\theta V^{\pi_\theta}(\mu) \rightarrow 0$, we are not
guaranteed that $\nabla_\pi V^{\pi_\theta}(\mu) \rightarrow 0$.

We now briefly discuss the main technical challenges in the proof. The
proof first
shows that the sequence
$V^{(t)}(s)$ is monotone increasing pointwise, i.e. for \emph{every} state $s$, $V^{(t+1)}(s) \geq V^{(t)}(s)$
(Lemma \ref{lem:q-improv}). This implies the existence of a limit $V^{(\infty)}(s)$
by the monotone convergence theorem (Lemma \ref{lem:delta}). Based on
the limiting quantities $V^{(\infty)}(s)$
and $Q^{(\infty)}(s,a)$, which we show exist, define the following limiting sets for each state $s$:
\begin{eqnarray*}
		I^s_{0} &:=&\{a | Q^{(\infty)}(s,a) = V^{(\infty)}(s)\}\\
		I^s_{+} &:=&\{a | Q^{(\infty)}(s,a) > V^{(\infty)}(s)\}\\
		I^s_{-} &:=&\{a | Q^{(\infty)}(s,a) < V^{(\infty)}(s)\} \, .
\end{eqnarray*}
The challenge is to then show that, for all states $s$, the
set $I^s_{+} $ is the empty set, which would immediately imply
$V^{(\infty)}(s) = V^\star(s)$.  The proof proceeds by contradiction,
assuming that $I^s_{+} $ is non-empty. Using that $I^s_{+} $ is
non-empty and that the gradient tends to zero in the limit, i.e.
$\nabla_\theta V^{\pi_\theta}(\mu) \rightarrow 0$, we have
that for all $a \in I^s_{+}$, $\pi^{(t)}(a|s) \rightarrow 0$ (see
\eqref{eq:softmax_grad}).  This, along with the functional form of the
softmax parameterization, implies that there must be divergence (in
magnitude) among the set of parameters associated with \emph{some}
action $a$ at state $s$, i.e. that
$\max_{a\in \Acal} |\theta^{(t)}_{s,a}| \to \infty$. The primary
technical challenge in the proof is to then use this divergence, along
with the dynamics of gradient ascent, to show that $I^s_{+} $ is empty
via a contradiction.

We leave it as a question for future work as to characterizing the
convergence rate, which we conjecture is exponentially slow in some of the
relevant quantities, such as in terms of the size of state space. Here, we turn to a regularization based approach
to ensure convergence at a polynomial rate in all relevant quantities.

\subsection{Polynomial Convergence with Log Barrier Regularization}
\label{sec:entropy}
Due to the exponential scaling with the parameters $\theta$, policies
can rapidly become near deterministic, when optimizing under the softmax
parameterization, which can result in slow convergence. Indeed a key
challenge in the asymptotic analysis in the previous section was to
handle the growth of the absolute values of parameters as they tend to infinity. A
common practical remedy for this is to use entropy-based
regularization to keep the probabilities from getting too
small~\citep{williams1991function, mnih2016asynchronous}, and we study
gradient ascent on a similarly regularized objective in this
section. Recall that the relative-entropy for
distributions $p$ and $q$ is defined as:
$\textrm{KL}(p,q) := \E_{x\sim p}[-\log q(x)/p(x)]$.
Denote the uniform distribution over a set $\mathcal{X}$ by
$\textrm{Unif}_\mathcal{X}$, and define the  following log barrier
regularized objective as:
\begin{eqnarray}
L_\lambda(\theta) &:=& V^{\pi_\theta}(\mu) - \lambda \,
\E_{s\sim \textrm{Unif}_\Scal} \bigg[\textrm{KL} (\textrm{Unif}_\Acal,\pi_\theta(\cdot|s))
\bigg]\nonumber\\
&  = &  V^{\pi_\theta}(\mu) + \frac{\lambda}{|\Scal|\, |\Acal|} \sum_{s,a} \log \pi_\theta(a|s)
+\lambda \log |\Acal| \, , \label{eqn:loss-reg}
\end{eqnarray}
where $\lambda$ is a regularization parameter. The constant (i.e. the
last term) is
not relevant with regards to optimization.
This
regularizer is different from the more commonly utilized entropy
regularizer as in \citet{mnih2016asynchronous}, a point
 which we return to in Remark~\ref{remark:entropy}.

The policy gradient ascent updates for
$L_\lambda(\theta)$ are given by:
\begin{equation}
    \theta^{(t+1)} = \theta^{(t)} + \eta \nabla_\theta L_\lambda(\theta^{(t)}).
    \label{eqn:gd-entropic}
\end{equation}
Our next theorem shows that approximate first-order stationary
points of the entropy-regularized objective are
approximately globally optimal, provided the regularization is
sufficiently small.

\begin{theorem} \label{thm:small-gradient}
(Log barrier regularization)
Suppose $\theta$ is such that:
\[
  \|\nabla_\theta L_\lambda(\theta)\|_2 \leq \epsopt
\]
and $\epsopt \leq \lambda/(2|\Scal|\, |\Acal|)$. Then we have
that for all starting state distributions $\rho$:
\begin{eqnarray*}
  V^{\pi_\theta}(\rho)
  &\geq& V^\star (\rho) -
         \frac{2 \lambda}{1-\gamma} \Norm{\frac{d^{\pi^\star}_{\rho}}{\mu } }_\infty         .
\end{eqnarray*}
\end{theorem}


\begin{proof}
  The proof consists of showing that
  $\max_{a}A^{\pi_\theta}(s,a) \le 2\lambda/(\mu(s)|\Scal|)$ for all
  states. To see that this is sufficient, observe that by the
  performance difference lemma (Lemma~\ref{lemma:perf_diff}),
\begin{align*}
V^\star(\rho) -V^{\pi_\theta}(\rho)
&= \frac{1}{1-\gamma} \sum_{s,a} d^{\pi^\star}_{\rho}(s) \pi^\star (a|s) A^{\pi_\theta}(s,a) \\
&\le \frac{1}{1-\gamma} \sum_{s} d^{\pi^\star}_{\rho}(s) \max_{a\in\Acal}A^{\pi_\theta}(s,a)\\
&\leq
\frac{1}{1-\gamma} \sum_{s} 2 d^{\pi^\star}_{\rho} (s) \lambda/(\mu(s) |\Scal|)\\&
\leq
\frac{2 \lambda}{1-\gamma} \max_s\left( \frac{d^{\pi^\star}_{\rho}(s)}{\mu (s)} \right).
\end{align*}
which would then complete the proof.

We now proceed to show that $\max_{a}A^{\pi_\theta}(s,a) \le 2\lambda/(\mu(s)|\Scal|)$. For this, it
suffices to bound $A^{\pi_\theta}(s,a)$ for any state-action pair
$s,a$ where $A^{\pi_\theta}(s,a)\geq 0$ else the claim is trivially
true.  Consider an $(s,a)$ pair such that $A^{\pi_\theta}(s,a)>0 $.
Using the policy gradient expression for the softmax parameterization
(see Lemma~\ref{lemma:softmax-grad}),
\begin{equation}\label{eqn:grad-KLobjective}
\frac{\partial L_\lambda(\theta)}{\partial \theta_{s,a}}
= \frac{1}{1-\gamma}
d^{\pi_\theta}_\mu(s) \pi_\theta(a|s)A^{\pi_\theta}(s,a)
+ \frac{\lambda}{|\Scal|} \left(\frac{1}{|\Acal|}-\pi_\theta(a|s)\right)
\, .
\end{equation}
The gradient norm assumption $\|\nabla_\theta L_\lambda(\theta)\|_2 \leq \epsopt$ implies that:
\begin{align*}
\epsopt \geq
\frac{\partial L_\lambda(\theta)}{\partial \theta_{s,a}}
&= \frac{1}{1-\gamma}
d^{\pi_\theta}_\mu(s)\pi_\theta(a|s)A^{\pi_\theta}(s,a)
+ \frac{\lambda}{|\Scal|} \left(\frac{1}{|\Acal|}-\pi_\theta(a|s)\right)\\&
\geq  \frac{\lambda}{|\Scal|} \left(\frac{1}{|\Acal|}-\pi_\theta(a|s)\right),
\end{align*}
where we have used $A^{\pi_\theta}(s,a)\geq 0$. Rearranging and using our assumption $\epsopt \leq \lambda/(2 |\Scal|\,|\Acal|)$,
\[
\pi_\theta(a|s) \geq \frac{1}{|\Acal|} - \frac{\epsopt|\Scal|}{\lambda} \geq \frac{1}{2|\Acal|} \, .
\]
Solving for $A^{\pi_\theta}(s,a)$ in \eqref{eqn:grad-KLobjective}, we have:
\begin{eqnarray*}
A^{\pi_\theta}(s,a)
&=& \frac{1-\gamma}{d^{\pi_\theta}_\mu(s) }
\left(\frac{1}{\pi_\theta(a|s)} \frac{\partial L_\lambda(\theta)}{\partial \theta_{s,a}} +\frac{\lambda}{|\Scal|} \left(1-\frac{1}{\pi_\theta(a|s)|\Acal|}\right) \right)\\
&\leq& \frac{1-\gamma}{d^{\pi_\theta}_\mu(s)}\left( 2|\Acal| \epsopt+\frac{\lambda}{|\Scal|} \right) \\
&\leq& 2 \frac{1-\gamma}{d^{\pi_\theta}_\mu(s)} \frac{\lambda}{|\Scal|} \\
&\leq& 2 \lambda /(\mu(s) |\Scal|) \, ,
\end{eqnarray*}
where the penultimate step uses $\epsopt \leq \lambda/(2 |\Scal|\,|\Acal|)$ and the
final step uses $d^{\pi_\theta}_{\mu}(s)\geq
(1-\gamma)\mu(s)$. This completes the proof.
\end{proof}

By combining the above theorem with standard results on the
convergence of gradient ascent (to first order stationary points), we obtain the following corollary.
\begin{corollary}
(Iteration complexity with log barrier regularization)
Let $\beta_\lambda :=\frac{8\gamma }{(1-\gamma)^3} + \frac{2\lambda
}{|\Scal|}$.
  Starting from any initial
  $\theta^{(0)}$, consider the updates~\eqref{eqn:gd-entropic} with
  \mbox{$\lambda = \frac{\epsilon(1-\gamma)}{2\Norm{\frac{d^{\pi^\star}_{\rho}}{\mu} }_\infty}$} and $\eta = 1/\beta_\lambda$.
Then for all starting state distributions $\rho$, we have
    \[
      \min_{t< T} \left\{ V^\star(\rho) - V^{(t)}(\rho)\right\}
      \leq \epsilon \quad \mbox{whenever}
      \quad T \geq \frac{320 |\Scal|^2|\Acal|^2}{(1-\gamma)^6\,\epsilon^2} \Norm{\frac{d^{\pi^\star}_{\rho}}{\mu } }_\infty^2.
    \]
\label{corollary:entropy}
\end{corollary}

See Appendix~\ref{app:entropy} for the proof. The corollary shows the
importance of  balancing how the regularization parameter $\lambda$
is set relative to the desired
accuracy $\epsilon$, as well as the importance of the initial
distribution $\mu$ to obtain global optimality.

\begin{remark}\label{remark:entropy}
(Entropy vs. log barrier regularization)
The more commonly considered regularizer is the entropy~\citep{mnih2016asynchronous} (also
see~\cite{ahmed2019understanding} for a more detailed empirical
investigation), where the regularizer would be:
\[
\frac{1}{|\Scal|}   \sum_s H(\pi_{\theta}(\cdot|s)) = \frac{1}{|\Scal|}  \sum_{s} \sum_a -\pi_{\theta}(a|s)
  \log \pi_{\theta}(a|s).
  \]
  Note the entropy is far less aggressive in penalizing small
  probabilities, in comparison to the log barrier, which is equivalent
  to the relative entropy. In particular, the entropy
  regularizer is always bounded between $0$ and $\log |\Acal|$, while
  the relative entropy (against the uniform distribution over
  actions), is bounded between $0$ and infinity, where it
  tends to infinity as probabilities tend to $0$. We
  leave it is an open question if a polynomial convergence
  rate~\footnote{Here, ideally we would like to be poly in 
$|\Scal|$, $|\Acal|$, $1/(1-\gamma)$, $1/\epsilon$, and the
distribution mismatch coefficient, which we conjecture may not be possible.} is
  achievable with the more common entropy regularizer; our polynomial
  convergence rate using the $\textrm{KL}$ regularizer crucially
  relies on the aggressive nature in which the relative entropy
  prevents small probabilities (the proof shows that any action, with a
  positive advantage, has a significant probability for any
  near-stationary policy of the regularized objective). 
\end{remark}

\subsection{Dimension-free Convergence of Natural Policy Gradient
  Ascent}
\label{sec:npg}
We now show the Natural Policy Gradient
algorithm, with the softmax
parameterization~\eqref{eq:softmax}, obtains an improved iteration
complexity. The NPG algorithm defines a Fisher information matrix
(induced by $\pi$), and performs gradient updates in the geometry
induced by this matrix as follows:
\begin{align}\label{eqn:npg}
  F_\rho(\theta) &= \E_{s \sim d^{\pi_\theta}_{\rho}}
\E_{a\sim \pi_\theta(\cdot | s) }\Big[ \nabla_\theta \log
\pi_\theta(a| s) \Big(\nabla_\theta \log \pi_\theta(a|
s)\Big)^\top \Big] \nonumber\\
\theta^{(t+1)} &= \theta^{(t)} + \eta F_\rho(\theta^{(t)})^\dagger \nabla_\theta V^{(t)}(\rho),
\end{align}
where $M ^\dagger$ denotes the Moore-Penrose pseudoinverse of the
matrix $M$.  Throughout this section, we restrict to using the initial
state distribution $\rho\in \Delta(\Scal)$ in our update rule in
\eqref{eqn:npg} (so our optimization measure $\mu$ and the performance
measure $\rho$ are identical). Also, we restrict attention to states
$s\in\Sset$ reachable from $\rho$, since, without loss of generality,
we can exclude states that are not reachable under this start state
distribution\footnote{Specifically, we restrict the MDP to the set of
  states
  $\{s\in\Sset~:~\exists \pi ~~\text{such that}~~ d^\pi_\rho(s) >
  0\}$.}.

We leverage a particularly convenient form the update takes for the softmax
parameterization (see ~\citet{Kakade01}). For completeness, we provide a proof in
Appendix~\ref{app:npg}.

\begin{lemma}
(NPG as soft policy iteration) For the softmax parameterization~\eqref{eq:softmax}, the NPG
updates~\eqref{eqn:npg} take the form:
\[
\theta^{(t+1)} = \theta^{(t)} + \frac{\eta}{1-\gamma} A^{(t)} \quad \mbox{and}\quad \pi^{(t+1)}(a| s) = \pi^{(t)}(a| s) \frac{\exp(\eta A^{(t)}(s,a)/(1-\gamma))}{Z_t(s)},
\]
where $Z_t(s) = \sum_{a \in \Acal} \pi^{(t)}(a| s)  \exp(\eta A^{(t)}(s,a)/(1-\gamma))$.
\label{lemma:npg-softmax}
\end{lemma}

The updates take a strikingly simple form in this special case; they
are identical to the classical multiplicative weights
updates~\citep{freund1997decision,Cesa-Bianchi:2006:PLG:1137817}
for  online linear optimization over the probability simplex, where the
linear functions are specified by the advantage function of the
current policy at each iteration.
Notably, there is no dependence on
the state distribution $d^{(t)}_{\rho}$, since the pseudoinverse of
the Fisher information  cancels out the effect of the state
distribution in NPG. We now provide a dimension free convergence rate
of this algorithm.

\begin{theorem}[Global convergence for NPG]
Suppose we run the NPG updates~\eqref{eqn:npg} using $\rho\in
\Delta(\Scal)$ and with
$\theta^{(0)}=0$. Fix $\eta>0$. For all $T>0$, we have:
\[
V^{(T)}(\rho) \geq V^\ast(\rho) - \frac{\log |\Acal|}{\eta T} - \frac{1}{(1-\gamma)^2T}.
\]
\label{thm:npg}
\end{theorem}

In particular, setting $\eta \geq (1-\gamma)^2 \log |\Acal|$, we see that
NPG finds an $\epsilon$-optimal policy in a number of iterations that is at most:
\[T \leq \frac{2}{(1-\gamma)^2\epsilon} ,\]
which has no dependence on the number of states or actions, despite the
non-concavity of the underlying optimization problem.

The proof strategy we take borrows ideas from the online regret
framework in changing MDPs (in~\citep{even-dar2009online}); here, we
provide a faster rate of convergence than the analysis implied by
~\citet{even-dar2009online} or by~\citet{geist2019theory}.
We also note that while this proof is obtained for the NPG updates, it is known in
the literature that in the limit of small stepsizes, NPG and TRPO
updates are closely related (e.g. see
~\cite{schulman2015trust,DBLP:journals/corr/NeuJG17,NIPS2017_7233}).

First, the following improvement lemma is helpful:

\begin{lemma}[Improvement lower bound for NPG]
	\label{lemma:npg-gap}
For the iterates $\pi^{(t)}$ generated by the NPG
updates~\eqref{eqn:npg}, we have for all starting state distributions $\mu$
\[
  V^{(t+1)}(\mu) - V^{(t)}(\mu) \geq \frac{(1-\gamma)}{\eta} \E_{s\sim
    \mu} \log Z_t(s) \geq 0 .
\]
\end{lemma}

\begin{proof}
First, let us show that $\log Z_t(s)\geq 0$. To see this, observe:
\begin{multline*}
    \log Z_t(s) = \log \sum_a \pi^{(t)}(a| s)\exp(\eta A^{(t)}(s,a)/(1-\gamma))\\
    \geq \sum_a \pi^{(t)}(a| s) \log \exp(\eta A^{(t)}(s,a)/(1-\gamma))
    = \frac{\eta}{1-\gamma}\sum_a \pi^{(t)}(a| s) A^{(t)}(s,a) = 0.
\end{multline*}
where the inequality follows by Jensen's inequality
on the concave function $\log x$ and the final equality uses $\sum_a
\pi^{(t)}(a| s)A^{(t)}(s,a) = 0$.
Using $d^{(t+1)}$ as shorthand for $d^{(t+1)}_{\mu}$, the performance
difference lemma implies:
\begin{align*}
    V^{(t+1)}(\mu) - V^{(t)}(\mu) &= \frac{1}{1-\gamma} \E_{s\sim d^{(t+1)}} \sum_a \pi^{(t+1)}(a| s) A^{(t)}(s,a)\\
    &= \frac{1}{\eta} \E_{s\sim d^{(t+1)}} \sum_a \pi^{(t+1)}(a| s) \log \frac{\pi^{(t+1)}(a| s)Z_t(s)}{\pi^{(t)}(a| s)}\\
    &= \frac{1}{\eta} \E_{s\sim d^{(t+1)}}
      \kl(\pi^{(t+1)}_s||\pi^{(t)}_s) + \frac{1}{\eta} \E_{s\sim
      d^{(t+1)}} \log Z_t(s)\\
&\geq \frac{1}{\eta} \E_{s\sim d^{(t+1)}} \log Z_t(s)
\geq \frac{1-\gamma}{\eta} \E_{s\sim \mu} \log Z_t(s),
\end{align*}
where the last step uses that $d^{(t+1)} =d^{(t+1)}_{\mu} \geq (1-\gamma)
\mu$, componentwise (by~\eqref{eqn:dpi}), and that $\log Z_t(s)\geq 0$.
\end{proof}

With this lemma, we now prove Theorem~\ref{thm:npg}.

\begin{proof}[\textbf{of Theorem~\ref{thm:npg}}]
Since $\rho$ is fixed, we use $d^\star$ as shorthand for
$d^{\pi^\star}_{\rho}$; we also use $\pi_s$ as shorthand for the vector of
$\pi(\cdot| s)$.
By the performance difference lemma (Lemma~\ref{lemma:perf_diff}),
\begin{align*}
  V^{\pi^\star}(\rho) &- V^{(t)}(\rho)
= \frac{1}{1-\gamma} \E_{s\sim d^{\star}} \sum_a \pi^\star(a| s) A^{(t)}(s,a)\\
&= \frac{1}{\eta} \E_{s\sim d^{\star}} \sum_a \pi^\star(a| s) \log \frac{\pi^{(t+1)}(a| s) Z_t(s)}{\pi^{(t)}(a| s)}\\
&= \frac{1}{\eta} \E_{s\sim d^{\star}} \left(\kl(\pi^\star_s ||
\pi^{(t)}_s) - \kl(\pi^\star_s || \pi^{(t+1)}_s) + \sum_a \pi^*(a| s) \log Z_t(s)\right) \\
&= \frac{1}{\eta} \E_{s\sim d^{\star}} \left(\kl(\pi^\star_s ||
\pi^{(t)}_s) - \kl(\pi^\star_s || \pi^{(t+1)}_s) + \log Z_t(s)\right),
\end{align*}
where we have used the closed form of our updates from
Lemma~\ref{lemma:npg-softmax} in the second step.

By applying Lemma~\ref{lemma:npg-gap} with $d^{\star}$ as the
starting state distribution, we have:
\[
\frac{1}{\eta} \E_{s\sim d^{\star}} \log Z_t(s) \leq \frac{1}{1-\gamma}
\Big(V^{(t+1)}(d^{\star}) - V^{(t)}(d^{\star})\Big)
\]
which gives us a bound on $\E_{s\sim d^{\star}} \log Z_t(s)$.

Using the above equation and that $V^{(t+1)}(\rho) \geq V^{(t)}(\rho)$ (as $V^{(t+1)}(s) \geq V^{(t)}(s)$ for all states $s$ by
Lemma~\ref{lemma:npg-gap}), we have:
\begin{align*}
V^{\pi^\star}(\rho) &- V^{(T-1)}(\rho)
\leq \frac{1}{T} \sum_{t=0}^{T-1} (V^{\pi^\star}(\rho) - V^{(t)}(\rho)) \\
&\leq \frac{1}{\eta T} \sum_{t=0}^{T-1} \E_{s\sim d^{\star}}
(\kl(\pi^\star_s || \pi^{(t)}_s) - \kl(\pi^\star_s || \pi^{(t+1)}_s))
+ \frac{1}{\eta T} \sum_{t=0}^{T-1} \E_{s\sim d^{\star}} \log Z_t(s) \\
&\leq \frac{\E_{s\sim d^{\star}} \kl(\pi^\star_s||\pi^{(0)})}{\eta T}
+ \frac{1}{(1-\gamma) T} \sum_{t=0}^{T-1} \Big(V^{(t+1)}(d^{\star}) - V^{(t)}(d^{\star})\Big) \\
&= \frac{\E_{s\sim d^{\star}}
\kl(\pi^\star_s||\pi^{(0)})}{\eta T}
+ \frac{V^{(T)}(d^{\star}) - V^{(0)}(d^{\star})}{(1-\gamma)T} \\
&\leq \frac{\log |\Acal|}{\eta T} + \frac{1}{(1-\gamma)^2T}.
\end{align*}
The proof is  completed using that $V^{(T)}(\rho) \geq V^{(T-1)}(\rho)$.
\end{proof}

\section{Function Approximation and Distribution Shift}
\label{section:func_approx}
We now analyze the case of using parametric policy classes:
\[
\Pi=  \{\pi_\theta \mid \theta \in \R^d\},
\]
where $\Pi$ may not contain all stochastic policies (and it may not
even contain an optimal policy).
In contrast with the tabular results in
the previous sections, the policy classes that we are often
interested in are not fully expressive, e.g. $d \ll |\Sset| |\Aset|$\
(indeed $\card{\Sset}$ or $\card{\Aset}$ need not even be finite for the results in this
section); in this sense, we are in the regime of function
approximation.

We focus on obtaining \emph{agnostic} results, where
we seek to do as well as the best policy in this class (or as well as some other
comparator policy).  While we are
interested in a solution to the (unconstrained) policy optimization problem
\[
  \max_{\theta \in \R^d}
  V^{\pi_\theta}(\rho),
\]
(for a given initial distribution $\rho$), we will see that
optimization with respect to a different distribution
will be helpful, just as in the tabular case,

We will consider variants of the NPG update rule \eqref{eqn:npg}:
\begin{equation}
\theta \leftarrow \theta + \eta F_\rho(\theta)^{\dagger} \nabla_\theta V^{\theta}(\rho) \, .
\label{eqn:vanilla_NPG}
\end{equation}
Our analysis will leverage a close connection between the NPG update rule
\eqref{eqn:npg} with the notion of  \emph{compatible function
  approximation}~\citep{sutton1999policy}, as formalized in~\cite{Kakade01}.
Specifically, it can be easily seen that:
\begin{equation}\label{eq:npg_argmin}
F_\rho(\theta)^{\dagger} \nabla_\theta V^{\theta}(\rho) =
\frac{1}{1-\gamma} w^\star,
\end{equation}
where $w^\star $ is a minimizer of the following regression problem:
\[
w^\star \in \argmin_w
\E_{s\sim d^{\pi_\theta}_\rho, a\sim \pi_\theta(\cdot | s)}
\left[(w^\top \nabla_\theta \log \pi_\theta(\cdot|s) - A^{\pi_\theta}(s,a))^2\right]
\, .\]
The above is a straightforward consequence of the first order optimality
conditions (see \eqref{eq:easily}). The above regression
problem can be viewed as ``compatible'' function approximation: we are
approximating $A^{\pi_\theta}(s,a)$ using the $\nabla_\theta \log
\pi_\theta(\cdot|s)$ as features. We also consider a variant of the
above update rule, $Q$-NPG, where instead of using advantages in the
above regression we use the $Q$-values.

This viewpoint provides a methodology for approximate updates, where
we can solve the relevant regression problems with samples.  Our main
results establish the effectiveness of NPG updates where there is
error both due to statistical estimation (where we may not use exact
gradients) and approximation (due to using a parameterized function
class); in particular,
we provide a novel estimation/approximation decomposition
relevant for the NPG algorithm. For these algorithms, we will first consider log
linear policies classes (as a special case) and then move on to more
general policy classes (such as neural policy classes).  Finally, it
is worth remarking that the results herein provide one of the first provable approximation
guarantees where the error conditions required do not have explicit worst
case dependencies over the state space.

\subsection{NPG and $Q$-NPG Examples}\label{section:examples}

In practice, the most common policy classes are of the form:
\begin{equation}\label{eq:policy_class}
\Pi=  \left\{
  \pi_\theta(a| s) = \frac{\exp\big(f_\theta(s,a)\big) }
  {\sum_{a' \in \Aset} \exp\big(f_\theta(s,a')\big)}
\ \bigg\vert \ \theta \in \R^d\right\},
\end{equation}
where $f_\theta$ is a differentiable function. For example, the tabular softmax
policy class is one where $f_\theta(s,a) = \theta_{s,a}$.
Typically, $f_\theta$ is either a linear function or a neural network.
Let us consider the NPG algorithm, and a variant $Q$-NPG, in each of these two cases.

\subsubsection{Log-linear Policy Classes and Soft Policy Iteration}\label{section:log_linear}

For any state-action pair $(s,a)$, suppose we have a feature
mapping $\phi_{s,a} \in \R^d$. Each policy in the log-linear policy class is of
the form:
\[
  \pi_\theta(a| s) = \frac{\exp(\theta \cdot \phi_{s,a})}
  {\sum_{a' \in \Aset} \exp(\theta \cdot \phi_{s,a'})},
\]
with $\theta \in
\R^d$. Here, we can take
$f_\theta(s,a) = \theta \cdot \phi_{s,a}$.

With regards to compatible function approximation for the log-linear
policy class, we have:
\[
\nabla_\theta \log \pi_\theta(a| s) = \overline\phi^{\ \theta}_{s,a},
\textrm{where } \quad
\overline\phi^{\ \theta}_{s,a} = \phi_{s,a} - \E_{a'\sim\pi_\theta(\cdot |  s)}[\phi_{s,a'}],
\]
that is, $\overline\phi^{\ \theta}_{s,a} $ is the centered version of
$\phi_{s,a}$. With some abuse of notation, we accordingly also define $\bar\phi^\pi$ for any policy $\pi$. Here, using \eqref{eq:npg_argmin}, the NPG update rule~\eqref{eqn:vanilla_NPG} is equivalent to:
\[
\textrm{NPG: }  \, \,\,
\theta \leftarrow \theta + \eta w_\star , \qquad
w_\star \in  \argmin_{w}  \E_{s\sim d^{\pi_\theta}_\rho, a\sim \pi_\theta(\cdot | s)}
\Big[ \big(A^{\pi_\theta}(s,a)  - w \cdot \overline\phi^{\ \theta}_{s,a}\big)^2
\Big].
\]
(We have rescaled the learning rate $\eta$ in
comparison to~\eqref{eqn:vanilla_NPG}).
Note that we recompute $w_\star$ for every update of $\theta$.
Here, the compatible function approximation error measures the
expressivity of our parameterization in how well linear functions of
the parameterization can capture the policy's advantage function.

We also consider a variant of the NPG update rule~\eqref{eqn:vanilla_NPG}, termed
\emph{$Q$-NPG}, where:
\[
\textrm{$Q$-NPG: }  \, \,\,
\theta \leftarrow \theta + \eta w_\star , \qquad
w_\star \in  \argmin_{w}   \E_{s\sim d^{\pi_\theta}_\rho, a\sim \pi_\theta(\cdot | s)}
\Big[ \big(Q^{\pi_\theta}(s,a)  - w \cdot \phi_{s,a}\big)^2
\Big].
\]
Note we do not center the features for $Q$-NPG; observe that
$Q^{\pi}(s,a)$ is also not 0 in expectation under $\pi(\cdot | s)$, unlike
the advantage function.
\begin{remark}
(NPG/$Q$-NPG and Soft-Policy Iteration)  We now see how we can view both
NPG and $Q$-NPG as an incremental (soft) version of policy iteration, just as
in Lemma~\ref{lemma:npg-softmax} for the tabular case. Rather than
writing the update rule in terms of the parameter $\theta$, we can
write an equivalent update rule directly in terms of the (log-linear)
policy $\pi$:\\
\[
\textrm{NPG: }  \, \,
\pi(a|s) \leftarrow \pi(a|s) \exp(w_\star \cdot \phi_{s,a})/Z_s, \quad
w_\star \in  \argmin_{w}   \E_{s\sim d^{\pi}_\rho, a\sim \pi(\cdot | s)}
\Big[ \big(A^{\pi}(s,a)  - w \cdot \overline\phi^{\ \pi}_{s,a}\big)^2
\Big],
\]
where $Z_s$ is normalization constant. While the policy update
uses the original features $\phi$ instead of $\overline\phi^{\ \pi}$, whereas the quadratic error minimization is terms of the centered
features $\overline\phi^{\ \pi}$, this distinction is not relevant due to that we may also
instead use
$\overline\phi^{\ \pi}$ (in the policy update) which would result in an equivalent update;
the normalization makes the update invariant to (constant) translations of
the features. Similarly, an equivalent update for $Q$-NPG, where we
update $\pi$ directly rather than $\theta$, is:
\[
\textrm{$Q$-NPG: }  \, \,
\pi(a|s) \leftarrow \pi(a|s) \exp(w_\star \cdot \phi_{s,a})/Z_s, \quad
w_\star \in  \argmin_{w}   \E_{s\sim d^{\pi}_\rho, a\sim \pi(\cdot | s)}
\Big[ \big(Q^{\pi}(s,a)  - w \cdot \phi_{s,a}\big)^2
\Big].
\]
\end{remark}

\begin{remark}
(On the equivalence of NPG and $Q$-NPG)  If it is the case that the
compatible function approximation error is $0$,
then it straightforward to verify that the NPG and $Q$-NPG are
equivalent algorithms, in that their corresponding policy updates will
be equivalent to each other.
\end{remark}

\subsubsection{Neural Policy Classes}
\label{sec:neural}


Now suppose $f_\theta(s,a)$ is a neural network parameterized by
$\theta\in\R^d$, where the policy class $\Pi$ is of form in \eqref{eq:policy_class}.
Observe:
\[
\nabla_\theta \log \pi_\theta(a| s) = g_\theta(s,a),
\textrm{where } \quad
g_\theta(s,a) = \nabla_\theta f_\theta(s,a) - \E_{a'\sim\pi_\theta(\cdot |  s)}[\nabla_\theta f_\theta(s,a')],
\]
and, using \eqref{eq:npg_argmin}, the NPG update rule~\eqref{eqn:vanilla_NPG} is equivalent to:
\[
\textrm{NPG: }  \, \,\, \theta \leftarrow \theta + \eta w_\star , \qquad
w_\star \in \argmin_{w}   \E_{s\sim d^{\pi_\theta}_\rho, a\sim \pi_\theta(\cdot | s)}
\Big[ \big(A^{\pi_\theta}(s,a)  - w \cdot g_\theta(s,a)\big)^2\Big]
\]
(Again, we have rescaled the learning rate $\eta$ in
comparison to~\eqref{eqn:vanilla_NPG}).

The $Q$-NPG variant of this update rule is:
\[
\textrm{$Q$-NPG: }  \, \,\, \theta \leftarrow \theta + \eta w_\star , \qquad
w_\star \in  \argmin_{w}   \E_{s\sim d^{\pi_\theta}_\rho, a\sim \pi_\theta(\cdot | s)}
\Big[ \big(Q^{\pi_\theta}(s,a)  - w \cdot \nabla_\theta f_\theta(s,a)\big)^2\Big].
\]

\subsection{$Q$-NPG: Performance Bounds for Log-Linear Policies}
\label{sec:q-npg}

For a state-action distribution $\upsilon$,  define:
\[
L(w;\theta,\upsilon) :=
\E_{s,a \sim \upsilon}
\bigg[ \big(Q^{\pi_\theta}(s,a)  - w \cdot \phi_{s,a} \big)^2 \bigg].
\]
The iterates of the $Q$-NPG algorithm can be viewed as minimizing
this loss under some (changing) distribution $\upsilon$.

We now specify an approximate version of $Q$-NPG. It is helpful to
consider a slightly more general version of the algorithm in the
previous section, where instead of optimizing under a starting state
distribution $\rho$, we have a different starting \emph{state-action} distribution $\nu$. Analogous to the definition of the state visitation measure, $d_{\mu}^\pi$,  we can define
a visitation measure over states \emph{and} actions induced by following $\pi$ after $s_0,a_0 \sim
\nu$. We overload notation using $d_{\nu}^\pi$ to also refer to the
state-action visitation measure; precisely,
\begin{equation}\label{eqn:c_def}
d_{\nu}^\pi(s,a) := (1-\gamma) \E_{s_0,a_0\sim \nu}
\sum_{t=0}^\infty \gamma^t {\Pr}^\pi(s_t=s,a_t=a|s_0,a_0)
\end{equation}
where $\Pr^\pi(s_t=s,a_t=a|s_0,a_0)$ is the
probability that $s_t=s$ and $a_t=a$, after starting at state
$s_0$, taking action $a_0$, and following $\pi$ thereafter. While we
overload notation for visitation distributions ($d^\pi_\mu(s)$ and $d^\pi_\nu(s,a)$) for notational convenience, note that
the state-action measure $d_{\nu}^\pi$ uses the subscript $\nu$, which
is a state-action measure.

$Q$-NPG will be defined with respect to the \emph{on-policy} state
action measure starting with $s_0,a_0\sim\nu$.  As per our convention, we define
\[ d^{(t)} := d_{\nu}^{\pi^{(t)}}.\]  The approximate
version of this algorithm is:
\begin{eqnarray}\label{eq:q-npg}
\textrm{Approx.  Q-NPG: }  \, \,\,
\theta^{(t+1)} = \theta^{(t)} + \eta  w^{(t)}, \qquad
w^{(t)} \approx \argmin_{\|w\|_2\leq W}  L(w;\theta^{(t)},d^{(t)}),
\end{eqnarray}
where the above update rule also permits us to constrain the norm of the update direction $w^{(t)}$
(alternatively, we could use $\ell_2$ regularization as is also common in practice).
The exact minimizer is denoted as:
\[
w_\star^{(t)} \in \argmin_{\|w\|_2\leq W}  L(w;\theta^{(t)},d^{(t)}).
\]
Note that $w_\star^{(t)}$ depends on the current parameter $\theta^{(t)}$.

Our analysis will take into account both the \emph{excess risk} (often also
referred to as estimation error) and the
\emph{transfer error}. Here, the excess risk will be due to that
$w^{(t)}$ may not be equal $w_\star^{(t)}$, and the approximation
error will be due to that even the best linear fit using
$w_\star^{(t)}$ may not perfectly match the $Q$-values, i.e.
$L(w_\star^{(t)};\theta^{(t)};d^{(t)})$ is unlikely to be $0$ in
practical applications.

We now formalize these concepts in the following assumption:

\begin{assumption}[Estimation/Transfer errors] \label{assum:approx_est} Fix a state distribution $\rho$; a
state-action distribution $\nu$; an arbitrary comparator policy
$\pi^\star$ (not necessarily an optimal policy). With respect to
$\pi^\star$, define the state-action measure $d^\star$ as
\[
d^\star(s,a) = d_{\rho}^{\pi^\star}(s) \circ \textrm{Unif}_\Acal(a)
\]
i.e. $d^\star$ samples states from the comparators state visitation
measure, $d_{\rho}^{\pi^\star}$ and
actions from the uniform distribution.
Let us permit the sequence of iterates $w^{(0)},w^{(1)}, \ldots w^{(T-1)}$
used by the $Q$-NPG algorithm to be random, where the randomness could
be due to sample-based,  estimation error.
Suppose the following holds for all $t < T$:
\begin{enumerate}
\item (\emph{Excess risk}) Assume that the estimation error is bounded
  as follows:
\[
\E \Big[
L(w^{(t)};\theta^{(t)},d^{(t)}) - L(w_\star^{(t)};\theta^{(t)},d^{(t)})
\Big] \leq \epsstat
\]
Note that using a sample based approach we would expect $\epsstat =
O(1/\sqrt{N})$ or better, where $N$ is the number of samples used to
estimate. $w_\star^{(t)}$
We formalize this in Corollary~\ref{cor:q_npg_sample}.
\item \label{assum:trasnfer_part} (\emph{Transfer error}) Suppose that the best predictor $w_\star^{(t)}$
  has an error bounded by $\epsbias$, in expectation, with respect
  to the comparator's measure of $d^*$. Specifically, assume:
\[
\E\Big[L(w_\star^{(t)};\theta^{(t)},d^\star) \Big]\leq \epsbias.
\]
We refer to $\epsbias$ as the \emph{transfer error} (or \emph{transfer bias});
it is the error where relevant distribution is shifted to $d^\star$. For the softmax policy parameterization for tabular MDPs,
$\epsbias=0$ (see remark~\ref{remark:linear_MDPs} for another
example).
\end{enumerate}
In both conditions, the expectations are with respect to the randomness
in the sequence of iterates $w^{(0)},w^{(1)}, \ldots w^{(T-1)}$,
e.g. the approximate algorithm may be sample based.
\end{assumption}

Shortly, we discuss how the transfer
error relates to the more standard
approximation-estimation decomposition. Importantly, with the transfer
error, it is always defined with respect to a  single, fixed
measure, $d^\star$.

\begin{assumption}[Relative condition number]\label{assum:conditioning} Consider the same $\rho$, $\nu$, and
$\pi^\star $ as in Assumption~\ref{assum:approx_est}. With respect to any state-action
distribution $\upsilon$,  define:
\[
\Sigma_\upsilon = \E_{s,a \sim \upsilon}\left[ \phi_{s,a}\phi_{s,a}^\top\right],
\]
and define
\[
\sup_{w \in \R^d} \ \frac{w^\top \Sigma_{d^\star} w}
{w^\top \Sigma_\nu w}
=\kappa.
\]
Assume that $\kappa$ is finite.
\end{assumption}

Remark~\ref{remark:kappa} discusses why it is reasonable to
expect that $\kappa$ is not a quantity related to the size of the
state space.\footnote{Technically, we only need the relative condition number $\sup_{w \in \R^d} \ \frac{w^\top \Sigma_{d^\star} w}
{w^\top \Sigma_{\pi^{(t)}} w}$ to be bounded for all $t$. We state
this as a sufficient condition based on the initial distribution $\nu$
due to: this is more interpretable, and, as per
Remark~\ref{remark:kappa}, this quantity can be bounded in  a manner that
is independent of the sequence of iterates produced by the algorithm.}


Our main theorem below shows how the
approximation error, the excess risk, and the conditioning, determine the final
performance. Note that both the transfer error $\epsbias$ and  $\kappa$ are defined
with respect to the comparator policy $\pi^\star$.

\begin{theorem}\label{thm:q_npg_fa}
(Agnostic learning with $Q$-NPG)
Fix a state distribution $\rho$; a state-action distribution $\nu$; an arbitrary
comparator policy $\pi^\star$ (not necessarily an optimal policy). Suppose
Assumption~\ref{assum:conditioning} holds and $\|\phi_{s,a}\|_2\leq
B$ for all $s,a$.
Suppose the $Q$-NPG update rule (in
\eqref{eq:q-npg}) starts with $\theta^{(0)}=0$, $\eta
=\sqrt{2\log |\Acal| /(B^2 W^2T)}$, and the (random) sequence of iterates satisfies
Assumption~\ref{assum:approx_est}.
We have that
\begin{align*}
\E\left[\min_{t< T} \left\{V^{\pi^\star}(\rho) - V^{(t)}(\rho) \right\}\right]
\leq
\frac{BW}{1-\gamma}\sqrt{\frac{2 \log |\Acal|}{T}}
+\sqrt{ \frac{4|\Acal| \kappa \epsstat}{(1-\gamma)^3}}
+ \frac{\sqrt{4|\Acal| \epsbias}}{1-\gamma}\, .
\end{align*}
\end{theorem}

The proof is provided in Section~\ref{section:analysis_fa}.

Note when $\epsbias = 0$, our
convergence rate is $O(\sqrt{1/T})$ plus a term that depends on the
excess risk; hence, provided we obtain enough samples, then
$\epsstat$ will also tend to $0$, and we will be
competitive with the comparison policy
$\pi^\star$.
When $\epsbias = 0$ and $\epsstat = 0$, as in the tabular
setting with exact gradients, the additional two terms become $0$, consistent
 with Theorem~\ref{thm:npg} except that the convergence rate is
$O(\sqrt{1/T})$ rather than the faster rate of $O(1/T)$.
Obtaining a faster rate in the function approximation regime appears
to require stronger conditions on how the
approximation errors are controlled at each iteration.

The usual
approximation-estimation error decomposition
is that we can write our error as:
\begin{eqnarray*}
L(w^{(t)};\theta^{(t)},d^{(t)})
&=&
\underbrace{
L(w^{(t)};\theta^{(t)},d^{(t)}) - L(w_\star^{(t)};\theta^{(t)},d^{(t)})
}_{\textrm{Excess risk}}
+\underbrace{
    L(w_\star^{(t)};\theta^{(t)},d^{(t)})}_{\textrm{Approximation error}}
\end{eqnarray*}
As we obtain more samples, we can drive the excess risk (the estimation error) to $0$
(see Corollary~\ref{cor:q_npg_sample}). The  approximation error above
is due to modeling error.
 Importantly, for our
$Q$-NPG performance bound, it is not this standard approximation error notion which is
relevant, but it is this error
under a different measure $d^\star$,
i.e. $L(w_\star^{(t)};\theta^{(t)},d^\star)$.
One appealing aspect
about the transfer error is that this error is with respect to a fixed
measure, namely $d^\star$. Furthermore, in practice,
modern machine learning methods often performs favorably with regards
to transfer learning, substantially better than worst case theory
might suggest.

The following corollary provides a performance bound in
terms of the usual notion of approximation error, at the cost of also
depending on the
worst case distribution mismatch ratio. The corollary disentangles the
estimation error from the approximation error.

\begin{corollary}\label{cor:q_npg_fa}
(Estimation error/Approximation error bound for $Q$-NPG)
Consider the same setting as in Theorem~\ref{thm:q_npg_fa}.
Rather than assuming the transfer error is bounded
(part~\ref{assum:trasnfer_part} in Assumption~\ref{assum:approx_est}),
suppose that, for all $t\leq T$,
\[
\E\Big[L(w_\star^{(t)};\theta^{(t)},d^{(t)})\Big]\leq \epsapprox.
\]
We have that
\begin{align*}
\E\left[\min_{t< T} \left\{V^{\pi^\star}(\rho) - V^{(t)}(\rho) \right\}\right]
\leq
\frac{BW}{1-\gamma}\sqrt{\frac{2 \log |\Acal|}{T}}
+\sqrt{ \kappa \cdot \frac{ 4|\Acal| \epsstat}{(1-\gamma)^3}}
+\sqrt{ \BigNorm{\frac{d^\star}{\nu}}_\infty \cdot \frac{4|\Acal| \epsapprox}{(1-\gamma)^3}}
\, .
\end{align*}
\end{corollary}

\begin{proof}
We have
the following crude upper bound on the transfer error:
\[
L(w_\star^{(t)};\theta^{(t)},d^\star)
\leq \BigNorm{\frac{d^\star}{d^{(t)}}}_\infty
L(w_\star^{(t)};\theta^{(t)},d^{(t)})
\leq \frac{1}{1-\gamma}\BigNorm{\frac{d^\star}{\nu}}_\infty
L(w_\star^{(t)};\theta^{(t)},d^{(t)}),
\]
where the last step uses the defintion of $d^{(t)}$ (see
~\eqref{eqn:c_def}). This implies $\epsbias \leq
\frac{1}{1-\gamma}\BigNorm{\frac{d^\star}{\nu}}_\infty
\epsapprox$, and the corollary follows.
\end{proof}

The above also shows
the striking difference between the effects of estimation error and
approximation error. The proof shows how the
transfer error notion is weaker than previous conditions based on distribution mistmatch coefficients or
concentrability coefficients. Also, as discussed in \cite{Scherrer:API}, the (distribution mismatch) coefficient
$\BigNorm{\frac{d^\star}{\nu}}_\infty$ is already weaker than
the more standard concentrability coefficients.

A few additional remarks are  now in order. We now make a few observations with regards to $\kappa$.

\begin{remark}\label{remark:kappa}
(Dimension dependence in $\kappa$ and the importance of $\nu$) It is
reasonable to think about $\kappa$ as being dimension dependent (or
worse), but it is not necessarily related to the size of the state space.  For
example, if $\|\phi_{s,a}\|_2\leq
B$, then
$
\kappa \leq \frac{B^2}{\sigma_{\min}(\E_{s,a \sim \nu }[
\phi_{s,a}\phi_{s,a}^\top])}
$
though this bound may be pessimistic.  Here, we also see the
importance of choice of $\nu$ in having a small (relative) condition number; in
particular, this is the motivation for considering the generalization
which allows for a starting state-action distribution $\nu$ vs. just a
starting state distribution $\mu$ (as we did in the tabular case). Roughly speaking, we desire a $\nu$ which
provides good coverage over the features. As the following lemma shows,
there always exists a universal distribution $\nu$, which can be
constructed only with knowledge of the
feature set (without knowledge of
$d^\star$), such that $\kappa\leq d$.
\end{remark}

\begin{lemma}
($\kappa\leq d$ is always possible)
Let $\Phi=\{\phi(s,a)|(s,a)\in \Scal\times\Acal\}\subset \R^d$ and suppose $\Phi$
is a compact set.
There always exists a state-action distribution $\nu$,
which is supported on at most $d^2$ state-action pairs and
which can be constructed only with knowledge of
$\Phi$  (without knowledge of the MDP or $d^\star$), such that:
\[
\kappa\leq d.
\]
\end{lemma}
\begin{proof}
The distribution can be found through
constructing the minimal volume ellipsoid containing $\Phi$, i.e. the
Lo\"wner-John ellipsoid~\citep{John2014ExtremumPW}. In particular,
this $\nu$ is supported on the contact points between this ellipsoid
and $\Phi$; the lemma immediately follows
from properties of this ellipsoid (e.g. see
\cite{ball1997elementary,john_bandit}).
\end{proof}

\vspace*{0.2cm}
It is also worth considering a more general example (beyond tabular
MDPs) in which $\epsbias=0$ for the log-linear policy class.

\begin{remark} \label{remark:linear_MDPs}
($\epsbias=0$ for ``linear'' MDPs) In the recent linear MDP model of
~\cite{jin2019provably,yang2019sample,jiang2017contextual}, where the
transition dynamics are low rank, we have that
$\epsbias=0$ provided we use the features of the linear MDP. Our guarantees also permit model misspecification of
linear MDPs, with non worst-case approximation error where $\epsbias\neq 0$.
\end{remark}

\begin{remark}\label{remark:politex} (Comparison with \textsc{Politex}
  and \textsc{EE-Politex}) Compared with
  \textsc{Politex}~\citep{abbasi2019politex},
  Assumption~\ref{assum:conditioning} is substantially milder, in that
  it just assumes a good relative condition number for one policy
  rather than all possible policies (which cannot hold in general even
  for tabular MDPs). Changing this assumption to an analog of
  Assumption~\ref{assum:conditioning} is the main improvement in the
  analysis of the \textsc{EE-Politex}~\citep{abbasi2019exploration}
  algorithm. They provide a regret bound for the average reward
  setting, which is qualitatively different from the suboptimality
  bound in the discounted setting that we study. They provide a
  specialized result for linear function approximation, similar to
  Theorem~\ref{thm:q_npg_fa}.
\end{remark}

\subsubsection{$Q$-NPG Sample Complexity}

\begin{assumption}[Episodic Sampling Oracle] \label{assum:sampling} For a fixed state-action distribution $\nu$,
we assume the ability to: start at $s_0,a_0\sim\nu$; continue
to act thereafter in the MDP according to any policy $\pi$; and terminate
this ``rollout'' when desired. With this oracle, it is straightforward
to obtain unbiased samples of $Q^{\pi}(s,a)$ (or $A^{\pi}(s,a)$) under
$s,a\sim d_{\nu}^{\pi}$ for any $\pi$; see
Algorithms~\ref{alg:q_sampler} and ~\ref{alg:a_sampler}.
\end{assumption}

\begin{algorithm}[!t]
	\begin{algorithmic}[1]	
		\Require Starting state-action distribution $\nu$.
		\State Sample $s_0,a_0\sim\nu$.
		\State Sample $s,a\sim d_{\nu}^\pi$ as follows:
                at every timestep $h$,  with probability $\gamma$, act
                according to $\pi$; else, accept  $(s_h,a_h)$ as the sample and
                proceed to Step~\ref{state:next_step}. See
                \eqref{eqn:c_def}.
		\State\label{state:next_step} From $s_h,a_h$,
                continue to execute $\pi$, and use a termination probability of
                $1-\gamma$. Upon termination, set
                $\widehat{Q^\pi}(s_h,a_h)$ as the \emph{undiscounted} sum
                of rewards from time $h$ onwards.
                \State \Return $(s_h, a_h)$ and $\widehat{Q^\pi}(s_h,a_h)$.
	\end{algorithmic}
	\caption{Sampler for: $s,a\sim d_{\nu}^\pi$ and
          unbiased estimate of $Q^\pi(s,a)$}
\label{alg:q_sampler}
\end{algorithm}

Algorithm~\ref{alg:q_npg_sample} provides a sample based version of the
$Q$-NPG algorithm; it simply uses stochastic
projected gradient ascent within each iteration. The following corollary shows
this algorithm suffices to obtain an accurate sample based version of $Q$-NPG.

\begin{corollary}\label{cor:q_npg_sample}
(Sample complexity of $Q$-NPG) Assume we are in the setting of
Theorem~\ref{thm:q_npg_fa} and that we have access to
an episodic sampling oracle (i.e. Assumption~\ref{assum:sampling}).
Suppose that the Sample Based $Q$-NPG Algorithm (Algorithm~\ref{alg:q_npg_sample}) is run for $T$
iterations, with $N$ gradient steps per iteration, with an appropriate
setting of the learning rates $\eta$ and  $\alpha$. We have that:
\begin{align*}
&\E\left[\min_{t< T} \left\{V^{\pi^\star}(\rho) - V^{(t)}(\rho) \right\}\right]\\
&\leq
\frac{BW}{1-\gamma}\sqrt{\frac{2 \log |\Acal|}{T}}
+\sqrt{ \frac{8\kappa |\Acal|BW (BW+1)}{(1-\gamma)^4}}  \frac{1}{N^{1/4}}
+ \frac{\sqrt{4|\Acal| \epsbias}}{1-\gamma}\, .
\end{align*}
Furthermore, since each episode has expected length $2/(1-\gamma)$, the
expected number of total samples used by $Q$-NPG is $2NT/(1-\gamma)$.
\end{corollary}

\begin{proof}
  Note that our sampled gradients are bounded by
  $G:=2B(BW+\frac{1}{1-\gamma})$.  Using
  $\alpha = \frac{W}{G\sqrt{N}}$, a standard analysis for stochastic
  projected gradient ascent (Theorem \ref{thm:shalev}) shows that:
\[
\epsstat\leq \frac{2BW(BW+\frac{1}{1-\gamma})}{\sqrt{N}}.
\]
The proof is completed via substitution.
\end{proof}

\begin{algorithm}[!t]
	\begin{algorithmic}[1]	
		\Require Learning rate $\eta$; SGD learning rate
                $\alpha$; number of SGD iterations $N$
		\State Initialize $\theta^{(0)} = 0$.
		\For{$t=0,1,\ldots,T-1$}
		\State Initialize $w_0 = 0$
		\For{$n=0,1,\ldots,N-1$}
		\State Call Algorithm~\ref{alg:q_sampler}  to obtain $s,a\sim d^{(t)}$ and an unbiased
                estimate $\widehat{Q}(s,a)$.
                \State Update:
                \[
                w_{n+1}= \textrm{Proj}_{\mathcal{W}} \Big(w_n - 2 \alpha\left(w_n\cdot \phi_{s,a} - \widehat Q(s,a)\right) \phi_{s,a}\Big)
                \]
               \Statex\quad\qquad where $\mathcal{W} = \{w: \|w\|_2\leq W\}$.
		\EndFor
		\State Set $\wh^{(t)} = \frac{1}{N} \sum_{n=1}^{N} w_n$.
		\State Update $\theta^{(t+1)} = \theta^{(t)} + \eta \wh^{(t)}$.
		\EndFor
	\end{algorithmic}

	\caption{Sample-based $Q$-NPG for Log-linear Policies}
\label{alg:q_npg_sample}
\end{algorithm}

\begin{remark}(Improving the scaling with $N$)
Our current rate of convergence is $1/N^{1/4}$ due to our use of
stochastic projected gradient ascent.
Instead, for the least squares estimator, $\epsstat$ would be
$O(d/N)$ provided certain further regularity assumptions hold (a bound
on the minimal eigenvalue of $\Sigma_{\nu}$ would be
sufficient but not necessary. See~\cite{hsu2014random} for such
conditions). With such further assumptions, our rate of convergence
would be $O(1/\sqrt{N})$.
\end{remark}

\subsection{NPG: Performance Bounds for Smooth
  Policy Classes}

We now return to the analyzing the standard NPG update rule, which
uses advantages rather than $Q$-values (see
Section~\ref{section:examples}). It is helpful to define
\[
L_A(w;\theta,\upsilon) :=
\E_{s,a \sim \upsilon}
\bigg[ \big(A^{\pi_\theta}(s,a)  - w \cdot \nabla_\theta \log \pi_\theta(a| s) \big)^2 \bigg].
\]
where $\upsilon$ is state-action distribution, and the subscript of
$A$ denotes the loss function uses advantages (rather than $Q$-values).
The iterates of the NPG algorithm can be viewed as minimizing
this loss under some appropriately chosen measure.

We now consider an approximate version of the NPG update rule:
\begin{eqnarray}\label{eq:npg}
\textrm{Approx.  NPG: }  \, \,\,
\theta^{(t+1)} = \theta^{(t)} + \eta  w^{(t)}, \qquad
w^{(t)} \approx \argmin_{\|w\|_2\leq W}  L_A(w;\theta^{(t)},d^{(t)}),
\end{eqnarray}
where again we use the on-policy, fitting distribution $d^{(t)}$. As
with $Q$-NPG, we also permit the use of a starting state-action
distribution $\nu$ as opposed to just a starting state distribution (see
Remark~\ref{remark:kappa}).  Again, we let $w_\star^{(t)}$ denote the
minimizer, i.e. $w_\star^{(t)} \in \argmin_{\|w\|_2\leq W}  L_A(w;\theta^{(t)},d^{(t)})$.

For this section, our analysis will focus on more general policy
classes, beyond log-linear policy classes. In particular, we
make the following smoothness assumption on the policy class:

\begin{assumption}\label{assumption:smoothness}
 (Policy Smoothness)
Assume for all $s \in
\Scal$ and $a\in\Acal$ that $\log \pi_{\theta}(a| s)$ is a $\beta$-smooth
function of $\theta$ (to recall the definition of smoothness, see \eqref{eq:def_smoothness}).
\end{assumption}

It is not to difficult to verify that the tabular softmax policy
parameterization is a $1$-smooth policy class in the above sense. The more general class of log-linear policies is also smooth as we remark below.
\begin{remark}(Smoothness of the log-linear policy class)\label{ex:smoothness_linear}
For the log-linear policy class (see Section~\ref{section:log_linear}),
smoothness is implied if the features $\phi$ have bounded Euclidean
norm. Precisely,
if the feature mapping $\phi$ satisfies $\|\phi_{s,a}\|_2 \leq
B$, then it is not difficult to verify that $\log \pi_\theta(a|s)$ is a
$B^2$-smooth function.
\end{remark}

For any state-action distribution $\upsilon$,  define:
\[
\Sigma^\theta_\upsilon = \E_{s,a\sim \upsilon} \left[
\nabla_\theta \log \pi_\theta(a| s) \left(\nabla_\theta \log \pi_\theta(a| s)\right)^\top
\right]
\]
and, again, we use $\Sigma^{(t)}_\upsilon$ as shorthand for $\Sigma^{\theta^{(t)}}_\upsilon$.

\begin{assumption}\label{assum:approx_est_npg}
(Estimation/Transfer/Conditioning) Fix a state distribution $\rho$; a state-action
distribution $\nu$; an arbitrary
comparator policy $\pi^\star$ (not necessarily an optimal policy). With respect to
$\pi^\star$, define the state-action measure $d^\star$ as
\[
d^\star(s,a) = d_{\rho}^{\pi^\star}(s) \pi^\star(a|s).
\]
Note that, in comparison to Assumption~\ref{assum:approx_est},
$d^\star$ is the state-action visitation  measure of the
comparator policy.
Let us permit the sequence of iterates $w^{(0)},w^{(1)}, \ldots w^{(T-1)}$
used by the NPG algorithm to be random, where the randomness could
be due to sample-based,  estimation error.
Suppose the following holds for all $t < T$:
\begin{enumerate}
\item (Excess risk) Assume the estimation error is bounded as:
\[
\E \Big[
L_A(w^{(t)};\theta^{(t)},d^{(t)}) - L_A(w_\star^{(t)};\theta^{(t)},d^{(t)})
\mid  \theta^{(t)} \Big] \leq \epsstat
\]
i.e. the above conditional expectation is bounded (with probability
one).\footnote{The use of a conditional expectation here
  (vs. the unconditional one in Assumption~\ref{assum:approx_est})
  permits the assumption to hold even in settings where we may reuse data in the sample-based approximation of
  $L_A$. Also, the expectation over the iterates allows a more natural
  assumption
  on the relative condition number, relevant for the more general case of smooth
  policies.} As we see in Corollary~\ref{cor:q_npg_sample}, we can
guarantee $\epsstat$ to drop as $\sqrt{1/N}$.
\item (Transfer error) Suppose that:
\[
\E\Big[L_A(w_\star^{(t)};\theta^{(t)},d^\star) \Big]\leq \epsbias.
\]
\item (Relative condition number) For all iterations $t$, assume the average relative condition number is
  bounded as follows:
\begin{equation}\label{eq:conditioning}
\E\left[\sup_{w \in \R^d} \ \frac{w^\top \Sigma^{(t)}_{d^\star} w}
{w^\top\Sigma^{(t)}_{\nu} w}\right]
\leq \kappa.
\end{equation}
Note that term inside the expectation is a random quantity as
$\theta^{(t)}$ is random.
\end{enumerate}
In the above conditions, the expectation is with respect to the randomness
in the sequence of iterates $w^{(0)},w^{(1)}, \ldots w^{(T-1)}$.
\end{assumption}

Analogous to our $Q$-NPG theorem, our main theorem for NPG shows how
the transfer error is relevant in addition the statistical error
$\epsstat$.

\begin{theorem}\label{thm:npg_fa}
(Agnostic learning with NPG)
Fix a state distribution $\rho$; a state-action distribution $\nu$; an arbitrary
comparator policy $\pi^\star$ (not necessarily an optimal policy). Suppose
Assumption~\ref{assumption:smoothness} holds.
Suppose the NPG update rule (in
\eqref{eq:npg}) starts with $\pi^{(0)}$ being the uniform distribution
(at each state), $\eta
=\sqrt{2\log |\Acal| /(\beta W^2T)}$, and the (random) sequence of iterates satisfies
Assumption~\ref{assum:approx_est_npg}.
We have that
\begin{align*}
\E\left[\min_{t< T} \left\{V^{\pi^\star}(\rho) - V^{(t)}(\rho) \right\}\right]
\leq
\frac{W}{1-\gamma}\sqrt{\frac{2\beta \log |\Acal|}{T}}
+\sqrt{ \frac{ \kappa \epsstat}{(1-\gamma)^3}}
+ \frac{\sqrt{ \epsbias}}{1-\gamma}\, .
\end{align*}
\end{theorem}

The proof is provided in Section~\ref{section:analysis_fa}.

\begin{remark} (The $|\Acal|$ dependence: NPG vs. $Q$-NPG)
Observe
there is no polynomial dependence on $|\Acal|$ in the rate for NPG (in
constrast to Theorem~\ref{thm:q_npg_fa}); also observe that here we
define $d^\star$ as the state-action distribution of $\pi^\star$ in
Assumption~\ref{assum:approx_est_npg}, as opposed to a uniform
distribution over the actions, as in
Assumption~\ref{assum:approx_est}. The main difference arises in the
analysis in that, even for $Q$-NPG, we need to bound the error in
fitting the advantage estimates; this leads to the dependence on
$|\Acal|$ (which can be removed with a path dependent bound, i.e. a
bound which depends on the sequence of iterates produced by the algorithm)\footnote{
  For $Q$-NPG, we have to bound two distribution shift terms to both
  $\pi^\star$ and $\pi^{(t)}$ at step $t$ of the algorithm.}.  For
NPG, the direct fitting of the advantage function sidesteps this
conversion step. Note that the relative condition number
assumption in $Q$-NPG (Assumption~\ref{assum:conditioning}) is a
weaker assumption, due to that it can be bounded independently of the path of the
algorithm (see Remark~\ref{assum:conditioning}), while NPG's centering of the features makes the assumption
on the relative condition number depend on the path of the
algorithm.
\end{remark}

\begin{remark} (Generalizing $Q$-NPG for smooth policies)
A similar reasoning as the analysis here can be also used to establish
a convergence result for the $Q$-NPG algorithm in this more general
setting of smooth policy classes. Concretely, we can analyze the
$Q$-NPG update described for neural policy classes in
Section~\ref{sec:neural}, assuming that the function $f_\theta$ is
Lipschitz-continuous in $\theta$. Like for Theorem~\ref{thm:npg_fa},
the main modification is that Assumption~\ref{assum:conditioning} on
relative condition numbers is now defined using the covariance matrix
for the features $f_\theta(s,a)$, which depend on $\theta$, as opposed
to some a feature map $\phi(s,a)$ in the log-linear case. The rest
of the analysis follows with an appropriate adaptation of the results
above.
\end{remark}

\subsubsection{NPG Sample Complexity}

Algorithm~\ref{alg:npg_sample} provides a sample based version of the
NPG algorithm, again using stochastic projected gradient ascent; it
uses a slight modification of the $Q$-NPG algorithm to obtain
unbiased gradient estimates. The following corollary shows that this
algorithm provides an accurate sample based version of NPG.

\begin{algorithm}[!t]
	\begin{algorithmic}[1]	
		\Require Starting state-action distribution $\nu$.
                \State Set $\widehat{Q^\pi}=0$ and $\widehat{V^\pi}=0$.
		\State Start at state $s_0\sim\nu$. Sample
                $a_0\sim\nu(\cdot|s_0)$ (though do not necessarily
                execute $a_0$).
                \State ($d_{\nu}^\pi$ sampling) At every timestep $h\geq 0$,
                \begin{itemize}
                \item With probability $\gamma$, execute $a_h$,
                  transition to $s_{h+1}$, and sample $a_{h+1}\sim\pi(\cdot|s_{h+1})$.
                \item Else accept $(s_h,a_h)$ as the
                sample and proceed to Step~\ref{state:next_step}.
                \end{itemize}
		\State\label{state:next_step} ($A^\pi(s,a)$ sampling) Set $\textrm{SampleQ}=\textrm{True}$
                with probability $1/2$.
\begin{itemize}
\item If $\textrm{SampleQ}=\textrm{True}$, execute $a_h$ at state $s_h$ and then
continue executing $\pi$ with a termination probability of
$1-\gamma$. Upon termination, set $\widehat{Q^\pi}$ as the \emph{undiscounted} sum
of rewards from time $h$ onwards.
\item Else sample $a'_h\sim
  \pi(\cdot|s_h)$.
 Then execute $a'_h$ at state $s_h$ and then
continue executing $\pi$ with a termination probability of
$1-\gamma$. Upon termination, set $\widehat{V^\pi}$ as the \emph{undiscounted} sum
of rewards from time $h$ onwards.
\end{itemize}
\State \Return $(s_h, a_h)$ and $\widehat{A^\pi}(s_h,a_h)=2(\widehat{Q^\pi}-\widehat{V^\pi})$.
	\end{algorithmic}
	\caption{Sampler for: $s,a\sim d_{\nu}^\pi$ and
          unbiased estimate of $A^\pi(s,a)$}
\label{alg:a_sampler}
\end{algorithm}

\begin{corollary}\label{cor:npg_sample}
(Sample complexity of NPG) Assume we are in the setting of
Theorem~\ref{thm:npg_fa} and that we have access to
an episodic sampling oracle (i.e. Assumption~\ref{assum:sampling}).
Suppose that the Sample Based NPG Algorithm (Algorithm~\ref{alg:npg_sample})
is run for $T$ iterations, with $N$ gradient steps per iteration.
Also, suppose that $\|\nabla_\theta \log \pi^{(t)}(a| s) \|_2\leq B$
holds with probability one.
There exists a setting of  $\eta$ and $\alpha$ such that:
\begin{align*}
&\E\left[\min_{t< T} \left\{V^{\pi^\star}(\rho) - V^{(t)}(\rho) \right\}\right]\\
&\leq
\frac{W}{1-\gamma}\sqrt{\frac{2 \beta\log |\Acal|}{T}}
+\sqrt{ \frac{8\kappa BW (BW+1)}{(1-\gamma)^4}}  \frac{1}{N^{1/4}}
+ \frac{\sqrt{ \epsbias}}{1-\gamma}\, .
\end{align*}
Furthermore, since each episode has expected length $2/(1-\gamma)$, the
expected number of total samples used by NPG is $2NT/(1-\gamma)$.
\end{corollary}

\begin{proof}
Let us see that the update direction in
Step~\ref{algline:NPG-fa-update} of Algorithm~\ref{alg:npg_sample}
uses an unbiased estimate of the true gradient of the loss function $L_A$:
\begin{eqnarray*}
&&2 \E_{s,a\sim d^{(t)}}\left[ \bigg(w_n \cdot \nabla_\theta \log
  \pi^{(t)}(a | s) -\widehat{A}(s,a)  \bigg) \nabla_\theta \log
  \pi^{(t)}(a | s)\right]\\
&=&
2 \E_{s,a\sim d^{(t)}}\left[ \bigg(w_n \cdot \nabla_\theta \log
  \pi^{(t)}(a | s) -\E[\widehat{A}(s,a)|s,a] \bigg) \nabla_\theta \log
  \pi^{(t)}(a | s)\right]\\
&=&
2 \E_{s,a\sim d^{(t)}}\left[ \bigg(w_n \cdot \nabla_\theta \log
  \pi^{(t)}(a | s) -A^{(t)}(s,a)  \bigg) \nabla_\theta \log
  \pi^{(t)}(a | s)\right]\\
&=&
\nabla_w L_A(w_n; \theta^{(t)},d^{(t)})
\end{eqnarray*}
where the last step follows due to that sampling procedure in
Algorithm~\ref{alg:a_sampler} produces a conditionally unbiased estimate.

Since $\|\nabla_\theta \log \pi^{(t)}(a| s) \|_2\leq B$ and since $\widehat{A}(s,a)\leq2/(1-\gamma)$,
our sampled gradients are bounded by $G:=8B(BW+\frac{1}{1-\gamma})$.
The remainder of the proof follows that of
Corollary~\ref{cor:q_npg_sample}
\end{proof}

\begin{algorithm}[!t]
	\begin{algorithmic}[1]	
		\Require Learning rate $\eta$; SGD learning rate
                $\alpha$; number of SGD iterations $N$
		\State Initialize $\theta^{(0)} = 0$.
		\For{$t=0,1,\ldots,T-1$}
		\State Initialize $w_0 = 0$
		\For{$n=0,1,\ldots,N-1$}
		\State Call Algorithm~\ref{alg:a_sampler} to obtain
                $s,a\sim d^{(t)}$, and an unbiased
                estimate $\widehat{A}(s,a)$ of $A^{(t)}(s,a)$.
        \State Update: \label{algline:NPG-fa-update}
                \[
                w_{n+1}= \textrm{Proj}_{\mathcal{W}} \bigg(w_n - 2\alpha
                \Big(w_n \cdot \nabla_\theta\log \pi^{(t)}(a| s)
                -\widehat{A}(s,a)\Big)\nabla_\theta \log \pi^{(t)}(a| s) \bigg),
                \]
                \Statex\quad\qquad where $\mathcal{W} = \{w: \|w\|_2\leq W\}$
		\EndFor
		\State Set $\wh^{(t)} = \frac{1}{N} \sum_{n=1}^{N} w_n$.
		\State Update $\theta^{(t+1)} = \theta^{(t)} + \eta \wh^{(t)}$.
		\EndFor
	\end{algorithmic}

	\caption{Sample-based NPG}
\label{alg:npg_sample}
\end{algorithm}

\subsection{Analysis}\label{section:analysis_fa}

We first proceed by providing a general analysis of NPG, for arbitrary
sequences. We then specialize it to complete the proof of our two
main theorems in this section.

\subsubsection{The NPG ``Regret Lemma'' }

It is helpful for us to consider NPG more abstractly, as an update rule of the form
\begin{eqnarray}\label{eqn:general_update}
\theta^{(t+1)} = \theta^{(t)} + \eta \w^{(t)}  .
\end{eqnarray}
We will now provide a lemma where $\w^{(t)}$ is an \emph{arbitrary}
(bounded) sequence, which will be helpful when specialized.

Recall a function $f:\R^d \rightarrow \R$ is said to be $\beta$-smooth if for all
$x, x^\prime \in \R^d$:
\[
\|\nabla f(x) - \nabla f(x^\prime)\|_2
\leq \beta \| x-x^\prime\|_2 \, ,
\]
and, due to Taylor's theorem, recall that this implies:
\begin{equation}\label{eq:def_smoothness}
\bigg| f(x^\prime) -f(x)
- \nabla f(x) \cdot (x^\prime-x)
\bigg|
\leq\frac{\beta}{2} \|x^\prime-x\|_2^2 \, .
\end{equation}

The following analysis of NPG is based on the mirror-descent approach
developed in~\citep{even-dar2009online}, which motivates us to
refer to it as a ``regret lemma''.

\begin{lemma}\label{lemma:npg_regret}
(NPG Regret Lemma) Fix a comparison policy $\widetilde\pi$ and a
  state distribution $\rho$. Assume for all $s \in
\Scal$ and $a\in\Acal$ that $\log \pi_{\theta}(a| s)$ is a $\beta$-smooth
function of $\theta$. Consider the update rule
  \eqref{eqn:general_update}, where $\pi^{(0)}$ is the uniform
  distribution (for all states) and where the sequence of weights
  $w^{(0)},\ldots , w^{(T)}$, satisfies $\|w^{(t)}\|_2 \leq W $ (but is
  otherwise arbitrary).  Define:
\[
\mathrm{err}_t = \E_{s\sim \widetilde d}\, \E_{a\sim \widetilde\pi(\cdot| s)}\Big[ A^{(t)}(s,a) - w^{(t)} \cdot \nabla_\theta \log \pi^{(t)}(a| s) \Big].
\]
We have that:
\begin{align*}
  \min_{t< T} \left\{V^{\widetilde\pi}(\rho) - V^{(t)}(\rho) \right\}  \leq\frac{1}{1-\gamma} \left(
\frac{\log |\Acal|}{\eta T}  +\frac{\eta \beta W^2}{2}
+\frac{1}{T}\sum_{t=0}^{T-1}\mathrm{err}_t\right).
  \end{align*}
\end{lemma}

\begin{proof}
By smoothness (see \eqref{eq:def_smoothness}),
  \begin{eqnarray*}
\log \frac{\pi^{(t+1)}(a| s) }{\pi^{(t)}(a| s)} &\geq&
    \nabla_\theta \log \pi^{(t)}(a| s) \cdot \big(\theta^{(t+1)}-\theta^{(t)}\big)
    -\frac{\beta}{2} \|\theta^{(t+1)}-\theta^{(t)}\|_2^2\\
 &=&
    \eta \nabla_\theta \log \pi^{(t)}(a| s) \cdot w^{(t)}
    -\eta^2\frac{\beta}{2} \|w^{(t)}\|_2^2 .
\end{eqnarray*}

We use $\widetilde d$ as shorthand for
$d^{\widetilde\pi}_{\rho}$ (note $\rho$ and $\widetilde\pi$ are
fixed); for any policy $\pi$, we also use $\pi_s$ as shorthand for the vector
$\pi(\cdot| s)$.
Using the performance difference lemma (Lemma~\ref{lemma:perf_diff}),
\begin{align*}
&\E_{s\sim \widetilde d} \left(\kl( \widetilde\pi_s || \pi^{(t)}_s) -
\kl( \widetilde\pi_s || \pi^{(t+1)}_s) \right)\\
&= \E_{s\sim \widetilde d}\, \E_{a\sim \widetilde\pi(\cdot| s)}
\left[\log \frac{\pi^{(t+1)}(a| s) }{\pi^{(t)}(a| s)}\right]\\
&\geq \eta \E_{s\sim \widetilde d}\, \E_{a\sim \widetilde\pi(\cdot| s)} \left[
\nabla_\theta \log \pi^{(t)}(a| s) \cdot w^{(t)}\right]
-\eta^2\frac{\beta}{2} \| w^{(t)}\|_2^2 \tag{using previous display}
\\
&= \eta \E_{s\sim \widetilde d}\, \E_{a\sim \widetilde\pi(\cdot| s)} \left[
A^{(t)}(s,a)\right]-\eta^2\frac{\beta}{2} \| w^{(t)}\|_2^2
\\&\qquad +
\eta \E_{s\sim \widetilde d}\, \E_{a\sim \widetilde\pi(\cdot| s)} \left[
\nabla_\theta \log \pi^{(t)}(a| s) \cdot w^{(t)}-A^{(t)}(s,a)\right]\\
&= (1-\gamma)\eta \bigg(V^{\widetilde\pi}(\rho) - V^{(t)}(\rho)\bigg)-\eta^2\frac{\beta}{2} \| w^{(t)}\|_2^2
- \eta \ \mathrm{err}_t
\end{align*}
Rearranging, we have:
\begin{align*}
V^{\widetilde\pi}(\rho)  - V^{(t)}(\rho)
\leq \frac{1}{1-\gamma}\left(\frac{1}{\eta } \E_{s\sim \widetilde d}
\left(\kl(\widetilde\pi_s || \pi^{(t)}_s) - \kl(\widetilde\pi_s || \pi^{(t+1)}_s) \right)
+\frac{\eta \beta}{2} W^2 + \mathrm{err}_t\right)
\end{align*}
Proceeding,
\begin{eqnarray*}
 \frac{1}{T} \sum_{t=0}^{T-1} (V^{\widetilde\pi}(\rho) - V^{(t)}(\rho))& \leq& \frac{1}{\eta T (1-\gamma)} \sum_{t=0}^{T-1} \E_{s\sim \widetilde d}\,
(\kl(\widetilde\pi_s || \pi^{(t)}_s) - \kl(\widetilde\pi_s || \pi^{(t+1)}_s))
 \\
&&+ \frac{1}{T(1-\gamma)} \sum_{t=0}^{T-1} \left(
\frac{\eta \beta W^2}{2}
+\mathrm{err}_t\right)\\
&\leq&  \frac{\E_{s\sim \widetilde d}\, \kl(\widetilde\pi_s||\pi^{(0)})}{\eta T(1-\gamma)}
+\frac{\eta \beta W^2}{2(1-\gamma)}
+  \frac{1}{T(1-\gamma)} \sum_{t=0}^{T-1} \mathrm{err}_t\\
&\leq& \frac{\log |\Acal|}{\eta T (1-\gamma)}
+\frac{\eta \beta W^2}{2(1-\gamma)} +  \frac{1}{T(1-\gamma)} \sum_{t=0}^{T-1} \mathrm{err}_t,
\end{eqnarray*}
which completes the proof.
\end{proof}

\subsubsection{Proofs of Theorem~\ref{thm:q_npg_fa} and
  \ref{thm:npg_fa}}

\begin{proof} (of Theorem~\ref{thm:q_npg_fa})
Using the NPG regret lemma (Lemma~\ref{lemma:npg_regret}) and the
smoothness of the log-linear policy class (see
Example~\ref{ex:smoothness_linear}),
\begin{align*}
\E\left[\min_{t< T} \left\{V^{\pi^\star}(\rho) - V^{(t)}(\rho) \right\}\right]
\leq
\frac{BW}{1-\gamma}\sqrt{\frac{2 \log |\Acal|}{T}}
+\E\left[\frac{1}{T}\sum_{t=0}^{T-1}\mathrm{err}_t\right]\, .
\end{align*}
where we have used our setting of $\eta$.

We make the following decomposition of $\mathrm{err}_t$:
\begin{align*}
\mathrm{err}_t
&= \E_{s\sim d^\star_\rho,a\sim \pi^\star(\cdot| s)}\Big[ A^{(t)}(s,a) - w_\star^{(t)} \cdot \nabla_\theta \log \pi^{(t)}(a| s) \Big]\\
&\quad+ \E_{s\sim d^\star_\rho,a\sim \pi^\star(\cdot| s)}\Big[ \big(w_\star^{(t)}- w^{(t)}\big) \cdot \nabla_\theta \log \pi^{(t)}(a| s) \Big].
\end{align*}
For the first term, using that $\nabla_\theta \log \pi_\theta(a| s) =
\phi_{s,a} - \E_{a'\sim\pi_\theta(\cdot |  s)}[\phi_{s,a'}]$ (see
Section~\ref{section:log_linear}), we have:
\begin{eqnarray}
&&\E_{s\sim d^\star_\rho,a\sim \pi^\star(\cdot| s)}\Big[ A^{(t)}(s,a) - w_\star^{(t)} \cdot \nabla_\theta \log \pi^{(t)}(a| s) \Big]\nonumber\\
&=& \E_{s\sim d^\star_\rho,a\sim \pi^\star(\cdot| s)}\Big[ Q^{(t)}(s,a) - w_\star^{(t)} \cdot \phi_{s,a}\Big]
-\E_{s\sim d^\star_\rho,a'\sim \pi^{(t)}(\cdot| s)}\Big[ Q^{(t)}(s,a') - w_\star^{(t)} \cdot \phi_{s,a'}\Big]\nonumber\\
&\leq& \sqrt{\E_{s\sim d^\star_\rho,a\sim \pi^\star(\cdot| s)} \left(Q^{(t)}(s,a) - w_\star^{(t)} \cdot \phi_{s,a}\right)^2}
+\sqrt{\E_{s\sim d^\star_\rho,a'\sim \pi^{(t)}(\cdot| s)}\left(Q^{(t)}(s,a') - w_\star^{(t)} \cdot \phi_{s,a'}\right)^2}\nonumber\\
&\leq& 2\sqrt{|\Acal|\E_{s\sim d^\star_\rho,a\sim \textrm{Unif}_\mathcal{\Acal}}\Big[ \left(Q^{(t)}(s,a) - w_\star^{(t)} \cdot \phi_{s,a}\right)^2\Big]}
=2\sqrt{|\Acal| L(w_\star^{(t)};\theta^{(t)},d^\star)}.\label{eq:term_one}
\end{eqnarray}
where we have used the definition of $d^\star$ and
$L(w_\star^{(t)};\theta^{(t)},d^\star)$ in the last step.

For the second term, let us now show that:
\begin{align}
&\E_{s\sim d^\star_\rho,a\sim \pi^\star(\cdot| s)}\Big[ \big(w_\star^{(t)}- w^{(t)}\big) \cdot \nabla_\theta \log \pi^{(t)}(a| s) \Big]\nonumber\\
&\qquad \leq\ 2\sqrt{\frac{|\Acal|\kappa}{1-\gamma} \left( L(w^{(t)};\theta^{(t)},d^{(t)}) - L(w_\star^{(t)};\theta^{(t)},d^{(t)})\right)}
\label{eq:term_two}
\end{align}
To see this, first observe that a similar argument to the above leads to:
\begin{align*}
&\E_{s\sim d^\star_\rho,a\sim \pi^\star(\cdot| s)}\Big[ \big(w_\star^{(t)}- w^{(t)}\big) \cdot \nabla_\theta \log \pi^{(t)}(a| s) \Big]\\
&=\E_{s\sim d^\star_\rho,a\sim \pi^\star(\cdot| s)}\Big[ \big(w_\star^{(t)}- w^{(t)}\big) \cdot \phi_{s,a} \Big]
-\E_{s\sim d^\star_\rho,a'\sim \pi^{(t)}(\cdot| s)}\Big[ \big(w_\star^{(t)}- w^{(t)}\big) \cdot \phi_{s,a'} \Big]\\
&\leq 2\sqrt{|\Acal|\E_{s,a\sim d^\star}\Big[ \left( \big(w_\star^{(t)}- w^{(t)}\big) \cdot \phi_{s,a}\right)^2\Big]}
= 2\sqrt{|\Acal| \cdot \| w_\star^{(t)}- w^{(t)} \|_{\Sigma_{d^\star}}^2},
\end{align*}
where we use the notation $\|x\|_M^2:=x^\top M x$ for a matrix $M$ and
a vector $x$. From the definition of $\kappa$,
\begin{eqnarray*}
\| w_\star^{(t)}- w^{(t)} \|_{\Sigma_{d^\star}}^2
\leq \kappa \| w_\star^{(t)}- w^{(t)} \|_{\Sigma_{\nu}}^2
\leq \frac{\kappa}{1-\gamma} \| w_\star^{(t)}- w^{(t)} \|_{\Sigma_{d^{(t)}}}^2
\end{eqnarray*}
using that $(1-\gamma)\nu \leq
d_{\nu}^{\pi^{(t)}}$ (see \eqref{eqn:c_def}).
Due to that $w_\star^{(t)}$ minimizes $L(w;\theta^{(t)}, d^{(t)})$
over the set $\Wcal:=\{w:\|w\|_2\leq W\}$, for any $w \in \Wcal$
the first-order optimality conditions for
$w_\star^{(t)}$ imply that:
\[
(w-w_\star^{(t)})\cdot \nabla L(w^{(t)}_\star; \theta^{(t)}, d^{(t)}) \geq 0.
\]
Therefore, for any $w \in \Wcal$,
\begin{align*}
&L(w;\theta^{(t)},d^{(t)}) - L(w_\star^{(t)};\theta^{(t)},d^{(t)})\\
&=\E_{s,a\sim d^{(t)}}\left[\big(w\cdot \phi(s,a) - w_\star\cdot \phi(s,a)+ w_\star\cdot \phi(s,a)-Q^{(t)}(s,a)\big)^2\right]
- L(w_\star^{(t)};\theta^{(t)},d^{(t)})\\
&=\E_{s,a\sim d^{(t)}}\left[\big(w\cdot \phi(s,a) - w_\star\cdot \phi(s,a)\big)^2\right]+
2(w - w_\star )\E_{s,a\sim d^{(t)}}\left[\phi(s,a)\big(w_\star\cdot \phi(s,a)-Q^{(t)}(s,a)\big)\right]\\
&=\|w - w_\star^{(t)}\|_{\Sigma_{d^{(t)}}}^2+
(w-w_\star^{(t)})\cdot \nabla L(w^{(t)}_\star; \theta^{(t)}, d^{(t)})\\
&\geq \|w - w_\star^{(t)}\|_{\Sigma_{d^{(t)}}}^2.
\end{align*}
Noting that $w^{(t)} \in \Wcal$ by construction in
Algorithm~\ref{alg:npg_sample} yields the claimed bound on the second
term in \eqref{eq:term_two}.

Using the bounds on the first and second terms in \eqref{eq:term_one}
and \eqref{eq:term_two}, along with concavity of the square root function, we have that:
\begin{align*}
\E[\mathrm{err}_t ] \leq 2\sqrt{|\Acal| \E\left[L(w_\star^{(t)};\theta^{(t)},d^\star)\right]}
+2\sqrt{\frac{|\Acal|\kappa}{1-\gamma} \E\left[ L(w^{(t)};\theta^{(t)},d^{(t)}) - L(w_\star^{(t)};\theta^{(t)},d^{(t)})\right]}.
\end{align*}
The proof is completed by substitution and using our assumptions on
$\epsstat$ and $\epsbias$.
\end{proof}

The following proof for the NPG algorithm follows along similar lines.

\begin{proof} (of Theorem~\ref{thm:npg_fa})
Using the NPG regret lemma and our setting of $\eta$,
\begin{align*}
\E\left[\min_{t< T} \left\{V^{\pi^\star}(\rho) - V^{(t)}(\rho) \right\}\right]
\leq
\frac{W}{1-\gamma}\sqrt{\frac{2 \beta \log |\Acal|}{T}}
+\E\left[\frac{1}{T}\sum_{t=0}^{T-1}\mathrm{err}_t\right]\, .
\end{align*}
where the expectation is with respect to the sequence of iterates $w^{(0)},w^{(1)}, \ldots w^{(T-1)}$.

Again, we make the following decomposition of $\mathrm{err}_t$:
\begin{align*}
\mathrm{err}_t
&= \E_{s\sim d^\star_\rho,a\sim \pi^\star(\cdot| s)}\Big[ A^{(t)}(s,a) - w_\star^{(t)} \cdot \nabla_\theta \log \pi^{(t)}(a| s) \Big]\\
&\quad+ \E_{s\sim d^\star_\rho,a\sim \pi^\star(\cdot| s)}\Big[ \big(w_\star^{(t)}- w^{(t)}\big) \cdot \nabla_\theta \log \pi^{(t)}(a| s) \Big].
\end{align*}
For the first term,
\begin{eqnarray*}
&&\E_{s\sim d^\star_\rho,a\sim \pi^\star(\cdot| s)}\Big[ A^{(t)}(s,a) - w_\star^{(t)} \cdot \nabla_\theta \log \pi^{(t)}(a| s) \Big]\\
&\leq& \sqrt{\E_{s\sim d^\star_\rho,a\sim \pi^\star(\cdot| s)}\Big[ \left(A^{(t)}(s,a) - w_\star^{(t)} \cdot \phi_{s,a}\right)^2\Big]}
=\sqrt{ L_{A}(w_\star^{(t)};\theta^{(t)},d^\star)}.
\end{eqnarray*}
where we have used the definition of
$L_{A}(w_\star^{(t)};\theta^{(t)},d^\star)$ in the last step.

For the second term, a similar argument leads to:
\begin{align*}
\E_{s\sim d^\star_\rho,a\sim \pi^\star(\cdot| s)}\Big[ \big(w_\star^{(t)}- w^{(t)}\big) \cdot \nabla_\theta \log \pi^{(t)}(a| s) \Big]
= \sqrt{\| w_\star^{(t)}- w^{(t)} \|_{\Sigma_{d^\star}}^2}.
\end{align*}
Define $\kappa^{(t)} := \|(\Sigma^{(t)}_{\nu})^{-1/2}
\Sigma_{d^\star}(\Sigma^{(t)}_{\nu})^{-1/2}\|_2$, which is the relative
condition number at iteration $t$. We have
\begin{eqnarray*}
\| w_\star^{(t)}- w^{(t)} \|_{\Sigma_{d^\star}}^2
&\leq& \|(\Sigma^{(t)}_{\nu})^{-1/2}\Sigma_{d^\star}(\Sigma^{(t)}_{\nu})^{-1/2}\|_2\,
\| w_\star^{(t)}- w^{(t)} \|_{\Sigma_{\nu}}^2\\
&\leq& \frac{\kappa^{(t)}}{1-\gamma}
\| w_\star^{(t)}- w^{(t)} \|_{\Sigma_{d^{(t)}}}^2\\
&\leq& \frac{\kappa^{(t)}}{1-\gamma}
\left(L_{A}(w^{(t)};\theta^{(t)},d^{(t)}) - L_{A}(w_\star^{(t)};\theta^{(t)},d^{(t)})\right)
\end{eqnarray*}
where the last step uses that $w_\star^{(t)}$ is a minimizer of $L_A$
over $\Wcal$ and that $w^{(t)}$ is feasible as before (see the proof
of Theorem~\ref{thm:q_npg_fa}). Now
taking an expectation we have:
\begin{eqnarray*}
\E\left[\| w_\star^{(t)}- w^{(t)} \|_{\Sigma_{d^\star}}^2\right]
&\leq&
\E\left[\frac{\kappa^{(t)}}{1-\gamma}
\left(L_{A}(w^{(t)};\theta^{(t)},d^{(t)}) - L_{A}(w_\star^{(t)};\theta^{(t)},d^{(t)})\right)\right]\\
&=& \E\left[\frac{\kappa^{(t)}}{1-\gamma}
\E\Big[L_{A}(w^{(t)};\theta^{(t)},d^{(t)}) - L_{A}(w_\star^{(t)};\theta^{(t)},d^{(t)})
\mid \theta^{(t)}\Big] \right]\\
&\leq& \E\left[\frac{\kappa^{(t)}}{1-\gamma} \right] \cdot \epsstat
\leq \frac{\kappa \epsstat}{1-\gamma}
\end{eqnarray*}
where we have used our assumption on $\kappa$ and $\epsstat$.

The proof is completed by substitution and using the concavity of the square root function.
\end{proof}

\section{Discussion}\label{sect:discussion}

This work provides a systematic study of the convergence properties of
policy optimization techniques, both in the tabular and the
function approximation settings. At the core, our results imply
that the non-convexity of the policy optimization problem is not the
fundamental challenge for typical variants of the policy gradient
approach. This is evidenced by the global convergence results which we
establish and that demonstrate the relative niceness of the underlying
optimization problem. At the same time, our results highlight that
insufficient exploration can lead to the convergence to sub-optimal
policies, as is also observed in practice; technically, we show how
this is an issue of conditioning. Conversely, we can expect
typical policy gradient algorithms to find the best policy from
amongst those whose state-visitation distribution is adequately
aligned with the policies we discover, provided a distribution-shifted
notion of approximation error is small.

In the tabular case, our results show that the nature and severity of the
exploration/distribution mismatch term differs
in different policy optimization approaches.
For instance, we find that doing policy
gradient in its standard form for both the direct and softmax
parameterizations can be slow to converge, particularly in the face of
distribution mismatch, even when policy gradients are computed
exactly. Natural policy gradient, on the other hand, enjoys a fast
dimension-free convergence when we are in tabular settings with exact
gradients. On the other hand, for the function approximation setting,
or when using finite samples, all algorithms suffer to some degree
from the exploration issue captured through a conditioning effect.

With regards to function approximation, the guarantees herein are the first
provable results that permit average case approximation errors, where the guarantees do not
have explicit worst case dependencies over the state space. These worst case dependencies are avoided
by precisely characterizing an approximation/estimation error
decomposition, where the relevant approximation error is under
distribution shift to an optimal policies measure. Here, we see that successful function
approximation relies on two key aspects: good conditioning (related to
exploration)  and low distribution-shifted, approximation error.
In particular, these results identify the relevant measure of the expressivity of a
policy class, for the natural
policy gradient. 

With regards to sample size issues, we showed that simply using
stochastic (projected) gradient ascent suffices for accurate policy
optimization. However, in terms of improving sample efficiency and polynomial dependencies, there are number of 
important questions for future research, including variance reduction
techniques along with data re-use. 

There are number of compelling directions for further
study. The first is in understanding how to remove the density
ratio guarantees among prior algorithms; 
our results are suggestive that the incremental policy optimization
approaches, including CPI~\citep{kakade2002approximately},
PSDP~\citep{NIPS2003_2378}, and MD-MPI~\cite{geist2019theory}, may
permit such an improved analysis. The question of understanding what
representations are robust to distribution shift is well-motivated by the
nature of our distribution-shifted, approximation error (the transfer error).
Finally, we hope that policy optimization approaches can be combined
with exploration approaches, so that, provably, these approaches can
retain their robustness properties (in terms of their agnostic
learning guarantees) while mitigating the need for a well conditioned
initial starting distribution.

\subsection*{Acknowledgments}
We thank the anonymous reviewers who provided detailed and
constructive feedback that helped us significantly improve the
presentation and exposition.  Sham Kakade and Alekh Agarwal gratefully
acknowledge numerous helpful discussions with Wen Sun with regards to
the $Q$-NPG algorithm and our notion of transfer error. We also acknowledge numerous helpful comments from
Ching-An Cheng and Andrea Zanette on an earlier draft of this work.
We thank Nan Jiang, 
Bruno Scherrer, and Matthieu Geist for their comments with regards to
the relationship between concentrability coefficients, the condition
number, and the transfer error; this discussion
ultimately lead to Corollary~\ref{cor:q_npg_fa}.
Sham Kakade acknowledges funding from the Washington Research
Foundation for Innovation in Data-intensive Discovery, the ONR award
N00014-18-1-2247, and the DARPA award FA8650-18-2-7836. Jason D. Lee
acknowledges support of the ARO under MURI Award W911NF-11-1-0303.
This is part of the collaboration between US DOD, UK MOD and UK
Engineering and Physical Research Council (EPSRC) under the
Multidisciplinary University Research Initiative.

	\bibliographystyle{plainnat}
	\bibliography{references,refs}
	\newpage
	\appendix
	\section{Proofs for Section \ref {section:setting}}
\label{app:setting}

\begin{proof}[\textbf{of Lemma~\ref{lemma:softmax-noncon}}]
Recall the MDP in Figure \ref{fig:noncon}. Note that since actions in terminal states $s_3$, $s_4$ and $s_5$ do not change the expected reward, we only consider actions in states $s_1$ and $s_2$. Let the "up/above" action as $a_1$ and "right" action as $a_2$. Note that
	\[V^\pi(s_1) = \pi(a_2|s_1) \pi(a_1|s_2) \cdot r \]
	
	Consider
	\[\theta^{(1)} = (\log 1, \log 3, \log 3, \log 1), \quad \theta^{(2)} = (-\log 1, -\log 3, -\log 3, -\log 1) \] where $\theta$ is written as a tuple $(\theta_{a_1,s_1}, \theta_{a_2,s_1}, \theta_{a_1,s_2}, \theta_{a_2,s_2})$. Then, for the softmax parameterization, we have
	\[\pi^{(1)}(a_2|s_1) = \frac{3}{4};\quad \pi^{(1)}(a_1|s_2) = \frac{3}{4};\quad V^{(1)}(s_1) = \frac{9}{16} r\]
	and
	\[\pi^{(2)}(a_2|s_1) = \frac{1}{4};\quad \pi^{(2)}(a_1|s_2) = \frac{1}{4};\quad V^{(2)}(s_1) = \frac{1}{16} r\]
	
	Also, for $\theta^{(\text{mid})} = \frac{\theta^{(1)} + \theta^{(2)}}{2}$,
	\[\pi^{(\text{mid})}(a_2|s_1) = \frac{1}{2};\quad \pi^{(\text{mid})}(a_1|s_2) = \frac{1}{2};\quad V^{(\text{mid})}(s_1) = \frac{1}{4} r\]
	
	This gives
	\[V^{(1)}(s_1) + V^{(2)}(s_1) > 2  V^{(\text{mid})}(s_1)\]
	which shows that $V^\pi$ is non-concave.
\end{proof}

\begin{proof}[\textbf{of Lemma~\ref{lemma:perf_diff}}]
Let $\Pr^\pi(\tau | s_0 = s)$ denote the probability of observing a trajectory $\tau$ when starting in state $s$ and following the policy $\pi$. Using a telescoping argument, we have:
\begin{eqnarray*}
V^\pi(s) - V^{\pi'}(s) &=&  \E_{\tau \sim {\Pr}^\pi(\tau|s_0=s) }
\left[\sum_{t=0}^\infty \gamma^t r(s_t,a_t)\right] - V^{\pi'}(s) \\
&=& \E_{\tau \sim {\Pr}^\pi(\tau|s_0=s) }
\left[\sum_{t=0}^\infty \gamma^t \left(r(s_t,a_t)+V^{\pi'}(s_t)-V^{\pi'}(s_t) \right)\right]-V^{\pi'}(s)\\
&\stackrel{(a)}{=}& \E_{\tau \sim {\Pr}^\pi(\tau|s_0=s) }
    \left[\sum_{t=0}^\infty \gamma^t \left(r(s_t,a_t)+\gamma V^{\pi'}(s_{t+1})-V^{\pi'}(s_t)\right)\right]\\
&\stackrel{(b)}{=}&\E_{\tau \sim {\Pr}^\pi(\tau|s_0=s) }
    \left[\sum_{t=0}^\infty \gamma^t \left(r(s_t,a_t)+\gamma \E[V^{\pi'}(s_{t+1})|s_t,a_t]-V^{\pi'}(s_t)\right)\right]\\
&=& \E_{\tau \sim {\Pr}^\pi(\tau|s_0=s) }
    \left[\sum_{t=0}^\infty \gamma^t A^{\pi'}(s_t,a_t)\right] = \frac{1}{1-\gamma}\E_{s'\sim d^\pi_s }\,\E_{a\sim \pi(\cdot | s)}
    \gamma^t A^{\pi'}(s',a),
    \end{eqnarray*}
where $(a)$ rearranges terms in the summation and cancels the $V^{\pi'}(s_0)$ term with the $-V^{\pi'}(s)$ outside the summation, and $(b)$ uses the tower property of conditional expectations and the final equality follows from the definition of $d^\pi_s$.
  \end{proof}

\section{Proofs for Section \ref{section:direct}}
\label{app:direct}

\subsection{Proofs for Section \ref{section:conv_pgd}}
\label{app:pgd}

We first define first-order optimality for constrained optimization.
\begin{definition}[First-order Stationarity]
	\label{def:eps-stationary}
	A policy ${\pi_\theta}\in \Delta(\cA) ^{|S|}$ is $\epsilon$-stationary with respect to the initial state distribution $\mu$ if
	\begin{align*}
	G({{\pi_\theta}}):=\max_{{\pi_\theta}+\delta \in \Delta(\cA) ^{|S|},~\|\delta\|_2\le 1}\delta^\top \nabla_\pi V^{{\pi_\theta}}(\mu) \leq \epsilon.
	\end{align*} where $\Delta(\cA) ^{|S|}$ is the set of all policies.
\end{definition}

Due to that we are working with the direct parameterization (see
~\eqref{eq:direct}), we drop
the $\theta$ subscript.

\begin{remark}
	If $\epsilon =0$, then the definition simplifies to $\delta^\top \nabla_\pi
	V^{{{\pi}} }(\mu)\le
	0$. Geometrically, $\delta$ is a feasible
	direction of movement since the probability simplex $\Delta(\cA) ^{|S|}$ is convex. Thus the gradient
	is negatively correlated with any feasible direction of movement, and
	so ${\pi}$ is first-order stationary.
\end{remark}

\begin{proposition}
	\label{prop-pgd}
	Let $V^{\pi}(\mu)$ be $\beta$-smooth in ${\pi}$. Define the gradient mapping \[
	G^\eta(\pi) =  \frac{1}{\eta} \left(P_{\Delta(\cA) ^{|S|}} ( {\pi} +\eta \nabla_\pi V^{\pi}(\mu)) - {\pi} \right),
	\] and the update rule for the projected gradient is $\pi^+ =
        {\pi} + \eta G^\eta(\pi)$. If  $\|G^\eta(\pi)\|_2 \le \epsilon$, then
	\begin{align*}
	\max_{{\pi} + \delta \in \Delta(A)^{|\Scal|},~ \Norm{\delta}_2\leq 1} \delta^\top \nabla_\pi V^{\pi^+}(\mu) \le  \epsilon(\eta \beta+1).
	\end{align*}
\end{proposition}
\begin{proof}
	By Theorem \ref{thm:ghadimi},
	\begin{align*}
	\nabla_\pi V^{\pi^+}(\mu) \in N_{\Delta(\cA)^{|\Scal|}} ( \pi^+) +\epsilon(\eta \beta+1) B_2,
	\end{align*}
	where $B_2$ is the unit $\ell_2$ ball, and $N_{\Delta(\cA)^{|\Scal|}}$ is the normal cone of the product simplex $\Delta(\cA)^{|\Scal|}$. Since $\nabla_\pi V^{\pi^+}(\mu)$ is $\epsilon (\eta\beta+1)$ distance from the normal cone and $\delta$ is in the tangent cone, then
$	\delta^\top\nabla_\pi V^{\pi^+}(\mu) \le \epsilon(\eta
\beta+1)
$.
\end{proof}

\begin{proof}[\textbf{of Theorem~\ref{thm:proj-gd}}]
	Recall the definition of gradient mapping
	\[G^\eta({\pi}) = \frac{1}{\eta}\left(P_{\Delta(\cA) ^{|S|}}({\pi} + \eta \nabla_\pi V^{(t)}(\mu)) - {\pi}\right)\]
	From Lemma \ref{lemma:smooth-pi}, we have $V^{\pi}(s)$ is $\beta$-smooth for all states $s$ (and also hence $V^{\pi}(\mu)$ is also $\beta$-smooth) with
$\beta = \frac{2\gamma|\Acal|}{(1-\gamma)^3}$.
	Then, from standard result (Theorem \ref{thm:beck}), we have that for $G^\eta ({\pi})$ with step-size $\eta = \frac{1}{\beta}$,
	\[\min_{t=0,1,\ldots, T-1}\|G^{\eta} (\pi^{(t)})\|_2 \leq \frac{\sqrt{2\beta(V^\star(\mu)-V^{(0)}(\mu))}}{\sqrt{T}} \]
	Then, from Proposition \ref{prop-pgd}, we have
	\[\min_{t=0,1,\ldots, T} \max_{\pi^{(t)} + \delta \in \Delta(A)^{|\Scal|},~ \Norm{\delta}_2\leq 1}	\delta^\top\nabla_\pi V^{\pi^{(t+1)} }(\mu) \le (\eta \beta+1) \frac{\sqrt{2\beta(V^\star(\mu)-V^{(0)}(\mu))}}{\sqrt{T}} \]
	Observe that \begin{align*}
		\max_{\bar\pi \in \Delta(A)^{|\Scal|}}(\bar\pi - {\pi})^\top \nabla_\pi V^{{\pi}}(\mu)	&=2\sqrt{|\Scal|} \max_{\bar\pi \in \Delta(A)^{|\Scal|}}~\frac{1}{2\sqrt{|\Scal|}}(\bar\pi - {\pi})^\top \nabla_\pi V^{{\pi}}(\mu)\\
			&\leq  2\sqrt{|\Scal|} \max_{{\pi}+\delta \in \Delta(A)^{|\Scal|},~\|\delta\|_2\le 1}\delta^\top \nabla_\pi V^{{\pi}}(\mu)
	\end{align*} where the last step follows as $\Norm{\bar\pi - {\pi}}_2\leq 2\sqrt{|\Scal|}$.
	And then using Lemma \ref{thm:first} and $\eta \beta = 1$, we have
	\[\min_{t=0,1,\ldots, T} V^\star (\rho) -V^{(t)}(\rho) \le \frac{4\sqrt{|\Scal|}}{1-\gamma} \left\|\frac{d_{\rho}^{\pi^\star}}{\mu}\right\|_\infty  \frac{\sqrt{2\beta(V^\star(\mu)-V^{(0)}(\mu))}}{\sqrt{T}}\]
	We can get our required bound of $\epsilon$, if we set $T$
        such that
\[
\frac{4\sqrt{|\Scal|}}{1-\gamma}
\left\|\frac{d_{\rho}^{\pi^\star}}{\mu}\right\|_\infty
\frac{\sqrt{2\beta(V^\star(\mu)-V^{(0)}(\mu))}}{\sqrt{T}} \leq
\epsilon
\]
or, equivalently,
\[
T \geq \frac{32|\Scal|
\beta(V^\star(\mu)-V^{(0)}(\mu))}{(1-\gamma)^2\epsilon^2}
\left\|\frac{d_{\rho}^{\pi^\star}}{\mu}\right\|^2_\infty .
\]
Using $V^\star(\mu)-V^{(0)}(\mu) \leq \frac{1}{1-\gamma}$ and $\beta=\frac{2\gamma|\Acal|}{(1-\gamma)^3}$ from
Lemma~\ref{lemma:smooth-pi} leads to the desired result.
\end{proof}

\subsection{Proofs for Section \ref{sec:vanishing-gradients}}
\label{app:lb}

Recall the MDP in Figure \ref{fig:chain}. Each trajectory starts from
the initial state $s_0$, and we use the discount factor
$\gamma = H/(H+1)$. Recall that we work with the direct
parameterization, where $\pi_\theta(a | s) = \theta_{s,a}$ for
$a = a_1,a_2,a_3$ and
$\pi_\theta(a_4 | s) = 1 - \theta_{s,a_1} - \theta_{s,a_2} -
\theta_{s,a_3}$. Note that since states $s_0$ and $s_{H+1}$ only have once action, therefore, we only consider the parameters for states $s_1$ to $s_{H}$. For this policy class and MDP, let $P^{\bp}$ be the state transition matrix under $\pi_\theta$, i.e. $[P^{\bp}]_{s, s'}$ is the probability of going from state $s$ to $s'$ under policy $\pi_\theta$:\[
[P^{\bp}]_{s, s'} = \sum_{a\in \Acal}\pi_\theta(a|s) P(s'|s,a).
\]
For the MDP illustrated in Figure~\ref{fig:chain}, the entries of this matrix are given as:
\begin{equation}
[P^{\bp}]_{s,s'} = \left\{\begin{array}{cc} \theta_{s,a_1} & \mbox{if $s' = s_{i+1}$ and $s = s_i$ with $1 \leq i \leq H$}\\
                                            1 - \theta_{s,a_1} & \mbox{if $s' = s_{i-1}$ and $s = s_i$ with $1 \leq i \leq H$}\\
                                            1 & \mbox{if $s' = s_1$ and $s = s_0$}\\
                                            1 & \mbox{if $s' = s = s_{H+1}$}\\
                                            0 & \mbox{otherwise}
                            \end{array}\right..
\label{eqn:lb-P}
\end{equation}

With this definition, we recall that the value function in the initial state $s_0$ is given by
\[
	V^{\pi_\theta}(s_0) = \E_{\tau \sim \pi_{\bp}}[\sum_{t=0}^\infty \gamma^t r_t] =
	e_0^T(I - \gamma P^{\bp})^{-1} r,
\]
where $e_0$ is an indicator vector for the starting state $s_0$. From
the form of the transition probabilities~\eqref{eqn:lb-P}, it is clear
that the value function only depends on the parameters
$\theta_{s,a_1}$ in any state $s$. While care is needed for
derivatives as the parameters across actions are related by the
simplex feasibility constraints, we have assumed each parameter is
strictly positive, so that an infinitesimal change to any parameter
other than $\theta_{s,a_1}$ does not affect the policy value and hence
the policy gradients. With this understanding, we succinctly refer to
$\theta_{s,a_1}$ as $\theta_s$ in any state $s$. We also refer to the
state $s_i$ simply as $i$ to reduce subscripts.

For convenience, we also define $\pmax$ (resp. $\pmin$) to be the
largest (resp. smallest) of the probabilities $\theta_{s}$ across the
states $s \in [1,H]$ in the MDP.

In this section, we prove Proposition~\ref{proposition:small_grad}, that is: for $0<\theta<1$ (componentwise across states and actions),
 $\pmax \leq 1/4$, and for all $k \leq  \frac{H}{40\log(2 H)} - 1$, we have $\norm{\nabla_\theta^k V^{\pi_\theta}(s_0)} \leq
 (1/3)^{H/4}$, where $\nabla_\theta^k V^{\pi_\theta}(s_0)$ is a tensor of the $k_{th}$
order.  Furthermore, we seek to show $V^{\star}(s_0) - V^{\pi_\theta}(s_0) \geq
	(H+1)/8-(H+1)^2/3^H$ (where $\theta^\star$ are the optimal policy's parameters).

It is easily checked that $V^{\pi_\theta}(s_0) = M^{\bp}_{0,H+1}$,
where
\[
M^{\bp} := (I - \gamma P^{\bp})^{-1},
\]
since the only rewards are obtained in the
state $s_{H+1}$. In order to bound the derivatives of the expected
reward, we first establish some properties of the matrix
$M^{\bp}$. 

\begin{lemma}
    Suppose $\pmax \leq 1/4$. Fix any $\mbound \in \left[\tfrac{1- \sqrt{1-4\gamma^2\pmax(1-\pmin)}}{2\gamma(1-\pmin)},\max\left\{\tfrac{1+ \sqrt{1-4\gamma^2\pmax(1-\pmin)}}{2\gamma(1-\pmin)},1\right\}\right]$. Then \begin{enumerate}
    	\item $M^\theta_{a,b} \leq \frac{\mbound^{b-a-1}}{1-\gamma}$ for $0 \leq a \leq b \leq H$.
    	\item $M^{\theta}_{a,H+1} \leq \frac{\gamma \pmax}{1-\gamma} M^\theta_{a,H} \leq \frac{\gamma \pmax}{(1-\gamma)^2} \mbound^{H-a}$ for $0 \leq a \leq H$.
    \end{enumerate}
\label{lemma:Mbound}
\end{lemma}

\begin{proof}
    Let $\rho^k_{a,b}$ be the normalized discounted probability of reaching $b$, when the initial state is $a$, in $k$ steps, that is
    \begin{equation}
    \label{eqn:def-rho}
	    \rho^k_{a,b} := (1-\gamma)\sum_{i=0}^k [(\gamma P^{{\bp}})^i]_{a,b},
    \end{equation}
    where we recall the convention that $U^0$ is the identity matrix for any square matrix $U$.
Observe that
    $0\leq \rho^k_{a,b} \leq 1$, and,  based on the form~\eqref{eqn:lb-P} of $P^{\bp}$, we have the recursive
    relation for all $k > 0$:
    \begin{equation}
        \rho^k_{a,b} = \left\{\begin{array}{cc} \gamma (1-\theta_{b+1})\rho^{k-1}_{a,b+1} + \gamma \theta_{b-1}\,\rho^{k-1}_{a,b-1} & \mbox{if $1 < b < H$}\\
        \gamma \theta_{H-1}\rho^{k-1}_{a,H-1} &\mbox{if $b=H$}\\
        \gamma \theta_{H} \rho^{k-1}_{a,H} + \gamma \rho^{k-1}_{a,H+1}& \mbox{if $b=H+1$ and $a<H+1$}\\
        1-\gamma & \mbox{if $b=H+1$ and $a=H+1$}\\
        \gamma (1-\theta_{2})\rho^{k-1}_{a,2} + \gamma\,\rho^{k-1}_{a,0} & \mbox{if $b=0$}
        \end{array}\right..
        \label{eqn:prob-recurrence}
    \end{equation}
	Note that $\rho^0_{a,b} = 0$ for $a\neq b$ and $\rho^0_{a,b} = 1-\gamma$ for $a = b$. Now let us inductively prove that for all $k \geq 0$
	\begin{equation}
	\label{eqn:induct}
		\rho^k_{a,b} \leq \mbound^{b-a} \quad \text{for} \quad 1 \leq a \leq b \leq H.
	\end{equation}
	Clearly this holds for $k = 0$ since $\rho^0_{a,b} = 0$ for $a\neq b$ and $\rho^0_{a,b} = 1-\gamma$ for $a = b$. Now, assuming the bound for all steps till $k-1$, we now prove it for $k$ case by case.	

        For $a = b$ the result follows since  \[
		\rho^k_{a,b} \leq 1 = \mbound^{b-a}.
	\]
	
        For $1<b < H$ and $a < b$, observe that the recursion~\eqref{eqn:prob-recurrence} and the inductive hypothesis imply that
    \begin{align*}
    \label{eqn:bound-p}
        \rho^k_{a,b} &\leq \gamma (1-\theta_{b+1}) \mbound^{b+1-a} + \gamma \theta_{b-1}\, \mbound^{b-1-a}\\
     &=  \mbound^{b-a-1} \big(\gamma (1-\theta_{b+1}) \mbound^2 + \gamma \theta_{b-1}\big)\\
    	&\leq \mbound^{b-a-1} \big(\gamma(1-\pmin) \mbound^2 + \gamma\pmax\big) \\
    	&= \mbound^{b-a-1} \big(\mbound+\gamma(1-\pmin) \mbound^2-\mbound + \gamma\pmax\big) \leq \mbound^{b-a},
\end{align*}
where the last inequality follows since
\mbox{$\mbound^2\gamma(1-\pmin) - \mbound + \gamma\pmax \leq 0$} due
to that
$\mbound$ is within the roots
of this quadratic equation. Note the
discriminant term in the square root is non-negative provided
$\pmax < 1/4$, since the condition along with the knowledge that
$\pmin \leq \pmax$ ensures that $4\gamma^2\pmax(1-\pmin) \leq 1$.

    For $b = H$ and $a < H$, we observe that
    \begin{align*}
        \rho^k_{a,b} &\leq \gamma \theta_{H-1}\, \mbound^{H-1-a} \\
        &= \mbound^{H-a}  \frac{\gamma \theta_{H-1}}{\mbound}\\
        &\leq \mbound^{H-a}(\frac{\gamma\pmax}{\mbound})\leq \mbound^{H-a}\big(\gamma(1-\pmin) \mbound + \frac{\gamma\pmax}{\mbound}\big) \leq \mbound^{H-a}.
    \end{align*}

    This proves the inductive claim (note that the cases of $b = a = 1$ and $b = a = H$ are already handled in the first part above).
    Next, we prove that for all $k\geq 0$ \[
	    \rho^k_{0,b} \leq \mbound^{b-1}.
    \]
    Clearly this holds for $k = 0$ and $b \neq 0$  since $\rho^0_{0,b}
    = 0$. Furthermore, for all $k\geq 0$ and $b=0$,
\[
  \rho^k_{0,b} \leq 1 \leq \mbound^{b-1},
\]
since $\alpha \leq 1$ by construction and $b = 0$. Now, we
consider the only remaining case when $k>0$ and $b \in
[1,H+1]$. By \eqref{eqn:lb-P}, observe that for $k>0$ and $b \in
[1,H+1]$,
    \begin{equation}
	    \label{eqn:0to1}
	    [(P^{{\bp}})^{i}]_{0,b} = [(P^{{\bp}})^{i-1}]_{1,b},
    \end{equation}
    for all $i \geq 1$. Using the definition of $\rho^k_{a,b}$ \eqref{eqn:def-rho} for $k>0$ and $b \in [1,H+1]$,
    \begin{align*}
    	\rho^k_{0,b} &= (1-\gamma)\sum_{i=0}^k [(\gamma P^{{\bp}})^i]_{0,b} = (1-\gamma)[(\gamma P^{{\bp}})^0]_{0,b} + (1-\gamma)\sum_{i=1}^k [(\gamma P^{{\bp}})^i]_{0,b} \\&= 0 + (1-\gamma)\sum_{i=1}^k \gamma^i [(P^{{\bp}})^i]_{0,b} \tag{since $b \geq 1$}\\
    	&= (1-\gamma) \sum_{i=1}^k \gamma^i \left[(P^{{\bp}})^{i-1}\right]_{1,b}\tag{using Equation \eqref{eqn:0to1}}\\
    	&= (1-\gamma) \gamma \sum_{j=0}^{k-1} \gamma^j [(P^{{\bp}})^{j}]_{1,b} \tag{By substituting $j = i-1$}\\
    	&= \gamma \rho^{k-1}_{1,b} \tag{using Equation \eqref{eqn:def-rho}}\\
    	&\leq \mbound^{b-1} \tag{using Equation \eqref{eqn:induct} and $\gamma, \mbound\leq 1$}
    \end{align*}
    Hence, for all $k\geq 0$ \[
    \rho^k_{0,b} \leq \mbound^{b-1}
    \] In conjunction with Equation \eqref{eqn:induct}, the above display gives for all $k\geq 0$,
    \begin{align*}
	    \rho^k_{a,b} \leq \mbound^{b-a} \quad \text{for} \quad 1 \leq a \leq b \leq H\\
	    \rho^k_{a,b} \leq \mbound^{b-a-1} \quad \text{for} \quad  0 = a \leq b \leq H
    \end{align*} Also observe that \[
    M^{\bp}_{a,b} = \lim_{k \to \infty} \frac{\rho^k_{a,b}}{1-\gamma}.
    \]Since the above bound holds for all $k\geq 0$, it also applies to the limiting value $M^{\bp}_{a,b}$, which shows that
    \begin{align*}
    M^\theta_{a,b} \leq \frac{\mbound^{b-a}}{1-\gamma} \leq \frac{\mbound^{b-a-1}}{1-\gamma}\quad \text{for} \quad 1 \leq a \leq b \leq H\\
    M^\theta_{a,b} \leq \frac{\mbound^{b-a-1}}{1-\gamma} \quad \text{for} \quad  0 = a \leq b \leq H
    \end{align*}
     which completes the proof of the first part of the lemma.

    For the second claim, from recursion~\eqref{eqn:prob-recurrence} and $b=H+1$ and $a < H+1$
    \begin{align*}
        \rho^k_{a,H+1} &= \gamma \theta_{H} \rho^{k-1}_{a,H} + \gamma \rho^{k-1}_{a,H+1} \leq \gamma \pmax \rho^{k-1}_{a,H} + \gamma \rho^{k-1}_{a,H+1},
    \end{align*}
    Taking the limit of $k \to \infty$, we see that
    \[
        M^\theta_{a,H+1} \leq \gamma \pmax M_{a,H} + \gamma M^\theta_{a,H+1}.
    \]
    Rearranging the terms in the above bound yields the second claim in the lemma.
\end{proof}

Using the lemma above, we now bound the
derivatives of $M^{\bp}$.

\begin{lemma}
    The $k_{th}$ order partial derivatives of $M$ satisfy:
    \[
        \Abs{\frac{\partial^k M^\theta_{0,H+1}}{\partial \theta_{\idx_1}\ldots\partial \theta_{\idx_k}}} \leq \frac{\pmax\, 2^k\, \gamma^{k+1} \, k!\,\mbound^{H-2k}}{(1-\gamma)^{k+2}}.
      \]
where $\bidx$ denotes a
$k$ dimensional vector with entries in $\{1,2,\ldots,H\} $.
\label{lemma:M-grad}
\end{lemma}

\begin{proof}
    Since the parameter $\theta$ is fixed throughout, we drop the superscript in $M^\bp$ for brevity. Using $\nabla_\theta M = -M\nabla_\theta (I - \gamma
    P^{\bp}) M$, using the form of $P^{\bp}$ in \eqref{eqn:lb-P}, we get for any $h \in [1, H]$
    \begin{equation}
        -\frac{\partial M_{a,b}}{\partial \theta_{h}} = -\gamma \sum_{i,j=0}^{H+1} M_{a,i} \frac{\partial P_{i,j}}{\partial \theta_h} M_{j,b} =         \gamma M_{a,h}(M_{h-1,b} - M_{h+1,b})
        \label{eqn:M-grad}
    \end{equation}
    where the second equality follows since $P_{h,h+1} = \theta_{h}$ and $P_{h,h-1} = 1-\theta_{h}$ are the only two entries in the transition matrix which depend on $\theta_h$ for $h \in [1,H]$.

    Next, let us consider a $k_{th}$ order partial derivative of $M_{0,H+1}$,
    denoted as $\tfrac{\partial^k M_{0,H+1}}{\partial
      \theta_{\bidx}}$. Note that $\bidx$ can have repeated entries to capture higher order derivative with respect to some parameter. We prove by induction for all $k\geq 1$, $-\frac{\partial^k M_{0,H+1}}{\partial
    	\theta_{\bidx}}$ can be written as $\sum_{n = 1}^N c_n \zeta_n$ where \begin{enumerate}
    	\item $\abs{c_n} = \gamma^k$ and $N \leq 2^k k!$,
    	\item Each monomial $\zeta_n$ is of the form $M_{i_1,j_1}\ldots M_{i_{k+1},j_{k+1}}$, $i_1 = 0$, $j_{k+1} = H+1$, $j_l \leq H$and $i_{l+1} = j_l \pm 1$ for all $l\in [1,k]$.
    \end{enumerate}

    The base case $k=1$ follows from Equation \eqref{eqn:M-grad}, as we can write for any $h \in [H]$
    \[
	    -\frac{\partial M_{0,H+1}}{\partial \theta_{h}} = \gamma M_{0,h}M_{h-1,H+1} - \gamma M_{0,h} M_{h+1,H+1}
    \]
    Clearly, the induction hypothesis is true with $ \abs{c_n} = \gamma$, $N=2$, $i_1 = 0$, $j_{2} = H+1$, $j_1 \leq H$ and $i_{2} = j_1 \pm 1$. Now, suppose the claim holds till
    $k-1$. Then by the chain rule:
    \begin{align*}
        \frac{\partial^k M_{0,H+1}}{\partial \theta_{\idx_1}\ldots\partial \theta_{\idx_k}} &= \frac{\partial \tfrac{\partial^{k-1} M_{0,H+1}}{\partial \theta_{\bidx^{/1}}}}{\partial \theta_{\idx_1}},
    \end{align*}
    where $\bidx^{/i}$ is the vector $\bidx$ with the $i_{th}$ entry removed. By inductive hypothesis,\[
	    -\frac{\partial^{k-1} M_{0,H+1}}{\partial \theta_{\bidx^{/1}}} = \sum_{n=1}^N c_n \zeta_n
    \] where \begin{enumerate}
    	\item $\abs{c_n} = \gamma^{k-1}$ and $N \leq 2^{k-1} (k-1)!$,
    	\item Each monomial $\zeta_n$ is of the form $M_{i_1,j_1}\ldots M_{i_{k},j_{k}}$, $i_1 = 0$, $j_{k} = H+1$, $j_l \leq H$ and $i_{l+1} = j_l \pm 1$ for all $l\in [1,k-1]$.
    \end{enumerate} In order to compute the $(k)_{th}$ derivative of
    $M_{0,H+1}$, we have to compute derivative of each monomial
    $\zeta_n$ with respect to $\theta_{\beta_1}$. Consider one of the
    monomials in the $(k-1)_{th}$ derivative, say,  \mbox{$\zeta =
      M_{i_1,j_1}\ldots M_{i_k,j_k}$}.  We invoke the chain rule as
    before and replace one of the terms in $\zeta$, say $M_{i_m,j_m}$,
    with $\gamma M_{i_m,\beta_1}M_{\beta_1-1,j_m} -\gamma
    M_{i_m,\beta_1} M_{\beta_1+1,j_m}$ using
    Equation~\ref{eqn:M-grad}. That is, the derivative of each entry
    gives rise to two monomials and therefore derivative of $\zeta$
    leads to $2k$ monomials which can be written in the form  $\zeta'=
    M_{i'_1,j'_1}\ldots M_{i'_{k+1},j'_{k+1}}$ where we have the
    following properties (by appropriately reordering terms)
    \begin{enumerate}
    	\item $i'_l, j'_l = i_l, j_l$ for $l < m$
    	\item $i'_l, j'_l = i_{l-1},j_{l-1}$ for $l > m+1$
    	\item $i'_m,j'_m = i_m, \beta_1$ and $i'_{m+1}, j'_{m+1} = j_{m} \pm 1, j_{m}$
    \end{enumerate}
    Using the induction hypothesis, we can write \[
	    -\frac{\partial^k M_{0,H+1}}{\partial \theta_{\idx_1}\ldots\partial \theta_{\idx_k}} = \sum_{n = 0}^{N'} c'_n \zeta'_n
    \] where \begin{enumerate}
    	\item $\abs{c'_n} = \gamma \abs{c_n} = \gamma^k$, since as shown above each coefficient gets multiplied by $\pm \gamma$.
    	\item $N'\leq 2k 2^{k-1} (k-1)! = 2^k k!$, since as shown above each monomial $\zeta$ leads to $2k$ monomials $\zeta'$.
    	\item Each monomial $\zeta'_n$ is of the form $M_{i_1,j_1}\ldots M_{i_{k+1},j_{k+1}}$, $i_1 = 0$, $j_{k+1} = H+1$, $j_l \leq H$ and $i_{l+1} = j_l \pm 1$ for all $l\in [1,k]$.
    \end{enumerate}
    This completes the induction.

    Next we prove a bound on the magnitude of each of the monomials which arise in the derivatives of $M_{0,H+1}$. Specifically, we show that for each monomial $\zeta = M_{i_1,j_1}\ldots M_{i_{k+1},j_{k+1}}$, we have
    \begin{equation}
    \label{eqn:bound-M}
    	\Abs{M_{i_1,j_1}\ldots M_{i_{k+1},j_{k+1}}} \leq \frac{\gamma \pmax \mbound^{H-2k}}{(1-\gamma)^{k+2}}
    \end{equation}
	We observe that it suffices to only consider pairs of indices $i_l, j_l$ where $i_l < j_l$. Since $\abs{M_{i,j}} \leq \frac{1}{1-\gamma}$ for all $i,j$,

    \begin{align}
    	\Abs{\prod_{l=1}^{k+1} M_{i'_l,j'_l}} &\leq \Abs{\prod_{1\leq l \leq k ~:~ i'_l < j'_l} M_{i'_l,j'_l}} \Abs{\prod_{1\leq l \leq k ~:~ i'_l \geq j'_l}\frac{1}{1-\gamma}}  \Abs{ M_{i'_{k+1},j'_{k+1}}}\nonumber\\
    &= \Abs{\prod_{1\leq l \leq k ~:~ i'_l < j'_l} M_{i'_l,j'_l}} \Abs{\prod_{1\leq l \leq k ~:~ i'_l \geq j'_l}\frac{1}{1-\gamma}}  \Abs{ M_{i'_{k+1},H+1}} \tag{by the inductive claim shown above}\nonumber\\
    	&\leq \frac{\mbound^{\sum_{\{1\leq l \leq k ~:~ i'_l < j'_l\}} j'_l - i'_l-1}}{(1-\gamma)^{k}} \frac{\gamma \pmax \mbound^{H-i'_{k+1}}}{(1-\gamma)^2} \nonumber \tag{using Lemma \ref{lemma:Mbound}, parts 1 and 2 on the first and last terms resp.}\\
    	&= \frac{\gamma \pmax \mbound^{\sum_{\{1\leq l \leq k+1 ~:~ i'_l < j'_l\}} j'_l - i'_l-1 }}{(1-\gamma)^{k+2}}
    	\label{eqn:sum-bound}
    \end{align}
    The last step follows from $H+1 = j'_{k+1} \geq i'_{k+1}$. Note that \[
    \sum_{\{1\leq l \leq k+1 ~:~ i'_l < j'_l\}} j'_l - i'_l \geq \sum_{l=1}^{k+1} j'_l - i'_l  = j'_{k+1} - i'_1 + \sum_{l=1}^{k} (j'_{l+1} - i'_l) \geq H+1 - k \geq 0
    \] where the first inequality follows from adding only non-positive terms to the sum, the second equality follows from rearranging terms and the third inequality follows from $i'_1 = 0$, $j'_{k+1} = H+1$ and $i'_{l+1} = j'_l \pm 1$ for all $l\in [1,k]$. Therefore, \[
	    \sum_{\{1\leq l \leq k+1 ~:~ i'_l < j'_l\}} j'_l - i'_l -1 \geq H - 2k
    \] Using Equation \eqref{eqn:sum-bound} and $\mbound \leq 1$ with above display gives
    \[
	    \Abs{\prod_{l=1}^{k+1} M_{i'_l,j'_l}} \leq \frac{\gamma \pmax \mbound^{H - 2k}}{(1-\gamma)^{k+2}}
    \]

    This proves the bound. Now using the claim that \[\frac{\partial^k M_{0,H+1}}{\partial
    	\theta_{\bidx}} = \sum_{n = 1}^N c_n \zeta_n\] where
    $\abs{c_n} = \gamma^k$ and $N \leq 2^k k!$, we have shown that
    \begin{align*}
    \Abs{\frac{\partial^k M_{0,H+1}}{\partial \theta_{\bidx}}} \leq \frac{\pmax\, 2^k\, \gamma^{k+1} \, k!\,\mbound^{H-2k}}{(1-\gamma)^{k+2}},
\end{align*}
which completes the proof.
  \end{proof}

We are now ready to prove Proposition~\ref{proposition:small_grad}.

\begin{proof}[Proof of Proposition~\ref{proposition:small_grad}]
The $k_{th}$ order partial derivative of $V^{\pi_\theta}(s_0)$ is equal to
\begin{align*}
    \frac{\partial^k V^{\pi_\theta}(s_0)}{\partial \theta_{\idx_1}\ldots\partial \theta_{\idx_h}} &= \frac{\partial^k M^{\bp}_{0,H+1}}{\partial \theta_{\idx_1}\ldots\partial \theta_{\idx_k}}.
\end{align*}

Given vectors $u^1,\ldots,u^k$ which are unit vectors in $\R^{H^k}$ (we denote the unit sphere by $\S^{H^k}$), the norm of this gradient tensor is given by:

\begin{align*}
\norm{\nabla_\theta^k V^{\pi_\theta}(s_0)} &=
\max_{u^1,\ldots,u^k\in \S^{H^k}} \Abs{\sum_{\bidx \in [H]^k} \frac{\partial^k V^{\pi_\theta}(s_0)}{\partial \theta_{\idx_1}\ldots\partial \theta_{\idx_k}} u^1_{\idx_1}\ldots u^k_{\idx_k}}\\
&\leq \max_{u^1,\ldots,u^k\in \S^{H^k}} \sqrt{\sum_{\bidx \in [H]^k} \bigg(\frac{\partial^k V^{\pi_\theta}(s_0)}{\partial \theta_{\idx_1}\ldots\partial \theta_{\idx_k}}\bigg)^2}\sqrt{\sum_{\bidx \in [H]^k} \big(u^1_{\idx_1}\ldots u^k_{\idx_k}\big)^2}\\
&= \max_{u^1,\ldots,u^k\in \S^{H^k}} \sqrt{\sum_{\bidx \in [H]^k} \bigg(\frac{\partial^k V^{\pi_\theta}(s_0)}{\partial \theta_{\idx_1}\ldots\partial \theta_{\idx_k}}\bigg)^2}\sqrt{\prod_{i=1}^k \norm{u^i}_2^2}\\
&= \sqrt{\sum_{\bidx \in [H]^k} \bigg(\frac{\partial^k V^{\pi_\theta}(s_0)}{\partial \theta_{\idx_1}\ldots\partial \theta_{\idx_k}}\bigg)^2} = \sqrt{\sum_{\bidx \in [H]^k} \bigg(\frac{\partial^k M^{\bp}_{0,H+1}}{\partial \theta_{\idx_1}\ldots\partial \theta_{\idx_k}}\bigg)^2}\\
&\leq \sqrt{\frac{H^k\pmax^2\, 2^{2k}\, \gamma^{2k+2} \, (k!)^2\,\mbound^{2H-4k}}{(1-\gamma)^{2k+4}}},
\end{align*}
where the last inequality follows from Lemma~\ref{lemma:M-grad}. In
order to proceed further, we need an upper bound on the smallest
admissible value of $\mbound$. To do so, let us consider all possible
parameters $\bp$ such that $\pmax \leq 1/4$ in accordance with the
theorem statement. In order to bound $\mbound$, it suffices to place
an upper bound on the lower end of the range for $\mbound$ in
Lemma~\ref{lemma:Mbound} (note Lemma~\ref{lemma:Mbound} holds for any
choice of $\mbound$ in the range). Doing so, we see that
\begin{align*}
    \frac{1 - \sqrt{1 - 4\gamma^2\pmax(1-\pmin)}}{2\gamma(1-\pmin)} &\leq \frac{1 - 1 + 2\gamma\sqrt{\pmax(1-\pmin)}}{2\gamma(1-\pmin)}\\
    &= \sqrt{\frac{\pmax}{1-\pmin}} \leq \sqrt{\frac{4\pmax}{3}},
\end{align*}
where the first inequality uses $\sqrt{x-y} \geq \sqrt{x} - \sqrt{y}$, by triangle inequality while the last inequality uses $\pmin \leq \pmax \leq 1/4$.

Hence, we have the bound
\begin{align*}
\max_{u^1,\ldots,u^k\in \S^{H^k}} \Abs{\sum_{\bidx \in [H]^h} \frac{\partial^k V^{\pi_\theta}(s_0)}{\partial \theta_{\idx_1}\ldots\partial \theta_{\idx_k}} u^1_{\idx_1}\ldots u^k_{\idx_k}}
&\leq\sqrt{\frac{H^k\pmax^2\, 2^{2k}\, \gamma^{2k+2} \, (k!)^2\,(\tfrac{4\pmax}{3})^{H-2k}}{(1-\gamma)^{2k+4}}}\\
& \stackrel{(a)}{\leq} \sqrt{(H+1)^{2k+4}H^k\pmax^2\, 2^{2k}\, \gamma^{2k+2} \, (k!)^2\,(\tfrac{4\pmax}{3})^{H-2k}}\\
& \stackrel{(b)}{\leq} \sqrt{(2H)^{2k+4}H^k\, 2^{2k} \, (H)^{2k}\,(\tfrac{4\pmax}{3})^{H-2k}}\\
&= \sqrt{(2)^{4k+4} (H)^{5k+4} \,(\tfrac{4\pmax}{3})^{H-2k}}
\end{align*}
where $(a)$ uses $\gamma = H/(H+1)$, $(b)$ follows since $\pmax \leq 1$, $H, k \geq 1$, $\gamma \leq 1$ and $k \leq H$.

Requiring that the gradient norm be no larger than $(\tfrac{4\pmax}{3})^{H/4}$, we would like to satisfy
\[
(2)^{4k+4} (H)^{5k+4} \,(\frac{4\pmax}{3})^{H-2k} \leq (\frac{4\pmax}{3})^{H/2},
\]
for which it suffices to have
\[
k \leq k_0 := \frac{\tfrac{H}{2}\log(3/4\pmax) - \log (2^4 H^4)}{\log(2^4 H^5) + 2 \log (3/4\pmax)}.
\]
Since, \begin{align*}
		&\frac{\tfrac{H}{2}\log(3/4\pmax) - \log (2^4 H^4)}{\log(2^4 H^5) + 2 \log (3/4\pmax)}\\
		 &\stackrel{(a)}{\geq} \frac{\tfrac{H}{2}\log(3/4\pmax) - \log (2^4 H^4)}{2\log(2^4 H^5) 2 \log (3/4\pmax)}\\
		 &\geq \frac{H}{8\log(2^4 H^5)} - \frac{\log (2^4 H^4)}{4\log(2^4 H^5)\log (3/4\pmax)}\\
		 &\stackrel{(b)}{\geq} \frac{H}{8\log(2^4 H^5)} - \frac{\log (2^4 H^4)}{4\log(2^4 H^4)\log (3)} \\
		 &\geq \frac{H}{40\log(2 H)} - 1
	\end{align*} where (a) follows from $a+b \leq 2ab$ when $a,b \geq 1$, (b) follows from $H \geq 1$ and $\pmax \leq 1/4$.
Therefore, in order to obtain the smallest value of $k_0$ for all choices of $0 \leq \pmax < 1/4$, we further lower bound $k_0$ as

\[
k_0 \geq  \frac{H}{40\log(2 H)} - 1,
\]
Thus, the norm of the gradient is bounded by $(\tfrac{4\pmax}{3})^{H/4} \leq (1/3)^{H/4}$ for all $k \leq \frac{H}{40\log(2 H)} - 1$ as long as $\pmax \leq 1/4$, which gives the first part of the lemma.

For the second part, note that the optimal policy always chooses the action $a_1$, and gets a discounted reward of
\[\gamma^{H+2}/(1-\gamma) = (H+1)\left(1-\frac{1}{H+1}\right)^{H+2} \geq \frac{H+1}{8},\]
where the final inequality uses $(1-1/x)^x \geq 1/8$ for $x \geq 1$. On the other hand, the value of $\pi_\theta$ is upper bounded by
\begin{align*}
M_{0,H+1} &\leq \frac{\gamma \pmax\mbound^H}{(1-\gamma)^2}\leq \frac{\gamma \pmax}{(1-\gamma)^2}\left(\frac{4\pmax}{3}\right)^H\\
&\leq \frac{(H+1)^2}{3^H}.
\end{align*} 
This gives the second part of the lemma.
\end{proof}

\section{Proofs for Section \ref{section:softmax}}
\label{app:softmax}

We first give a useful lemma about the structure of policy gradients for the softmax parameterization. We use the notation $\Pr^\pi(\tau | s_0 = s)$ to denote the probability of observing a trajectory $\tau$ when starting in state $s$ and following the policy $\pi$ and $\Pr^\pi_\mu(\tau)$ be $\E_{s\sim \mu}[\Pr^\pi(\tau | s_0 = s)]$ for a distribution $\mu$ over states.

\begin{lemma}
	\label{lemma:softmax-grad}
	For the softmax policy class, we have:
	\[
	\frac{\partial V^{\pi_\theta}(\mu)}{\partial \theta_{s,a}} =
	\frac{1}{1-\gamma} d^{\pi_\theta}_{\mu}(s)\pi_\theta(a|s)A^{\pi_\theta}(s,a)
	\]
\end{lemma}

\begin{proof}
	First note that
	\begin{align}
	\frac{\partial \log \pi_\theta(a|s)}{\partial \theta_{s',a'}} =
	\mathds{1} \Big[ s=s' \Big]\Big(\mathds{1} \Big[ a=a' \Big] - \pi_\theta(a'|s)\Big)
    \label{eqn:softmax-logpi-grad}
	\end{align}
	where $\mathds{1}|\mathcal{E}]$ is the indicator of $\mathcal{E}$ being
	true.

	Using this along with the policy gradient expression~\eqref{eqn:Apg},
	we have:
	\begin{align*}
	\frac{\partial V^{\pi_\theta}(\mu)}{\partial \theta_{s,a}} &= \E_{\tau \sim \Pr^{\pi_\theta}_{\mu}} \left[
	\sum_{t=0}^\infty \gamma ^t \mathds{1}[s_t=s] \Big(
	\mathds{1}[a_t=a] A^{\pi_\theta}(s,a)  - \pi_\theta(a|s) A^{\pi_\theta}(s_t,a_t)\Big)\right]\\
	&=      \E_{\tau \sim \Pr^{\pi_\theta}_{\mu}} \left[
	\sum_{t=0}^\infty \gamma ^t \mathds{1}[(s_t,a_t)=(s,a)] A^{\pi_\theta}(s,a) \right]
	\\&\qquad-\pi_\theta(a|s)
	\sum_{t=0}^\infty \gamma ^t \E_{\tau \sim \Pr^{\pi_\theta}_{\mu}} \left[ \mathds{1}[s_t=s] A^{\pi_\theta}(s_t,a_t)\right]\\
	&=
	\frac{1}{1-\gamma}\E_{s'\sim d^{\pi_\theta}_{\mu}}\E_{a'\sim \pi_\theta(\cdot|s) }\Big[\mathds{1}[(s',a')=(s,a)] A^{\pi_\theta}(s,a) \Big]
	-0 \\
	&=
	\frac{1}{1-\gamma} d^{\pi_\theta}_{\mu}(s)\pi_\theta(a|s)A^{\pi_\theta}(s,a) \, ,
	\end{align*}
	where the second to last step uses that for any policy $\sum_a \pi(a|s) A^\pi(s,a)=0$.
\end{proof}

\subsection{Proofs for Section \ref{sec:asymptotic}}
\label{app:gd-softmax}

We now prove Theorem~\ref{thm:glb-softmax}, i.e. we show that for the
updates given by
\begin{equation}
	\theta^{(t+1)} = \theta^{(t)} + \eta \nabla V^{(t)}(\mu),
	\label{eqn:upd-softmax-ms}
\end{equation}
policy gradient converges to optimal policy for the softmax
parameterization.

We prove this theorem by first proving a series of supporting
lemmas. First, we show in  Lemma \ref{lem:q-improv}, that $V^{(t)}(s)$
is monotonically increasing for all states $s$ using the fact that for
appropriately chosen stepsizes GD makes monotonic improvement for smooth
objectives.

\begin{lemma}[Monotonic Improvement in $V^{(t)}(s)$]
  \label{lem:q-improv}
  For all states $s$ and actions $a$, for updates
  \eqref{eqn:upd-softmax-ms} with learning rate $\eta \leq
  \frac{(1-\gamma)^2}{5}$, we have
  \[
    V^{(t+1)}(s) \geq
    V^{(t)}(s);\quad Q^{(t+1)}(s,a) \geq Q^{(t)}(s,a).
  \]
\end{lemma}

\begin{proof}
The proof will consist of showing that:
\begin{equation}\label{eqn:F_monotone}
  \sum_{a\in \Acal} \pi^{(t+1)}(a|s) A^{(t)}(s,a) \geq \sum_{a\in \Acal} \pi^{(t)}(a|s) A^{(t)}(s,a) = 0.
\end{equation}
holds for all states $s$. To see this,
observe that since the above holds for all states $s'$, the performance difference
lemma (Lemma \ref{lemma:perf_diff}) implies
\[
  V^{(t+1)}(s) - V^{(t)}(s)
  = \frac{1}{1-\gamma}\E_{s'\sim d_{s}^{\pi^{(t+1)}} }\E_{a\sim \pi^{(t+1)}(\cdot|s') }
  \left[A^{{(t)}}(s',a)\right] \geq 0,
\]
which would complete the proof.

Let us use the notation $\theta_s\in \R^{|\Acal|}$ to refer to the vector of
$\theta_{s,\cdot}$ for some fixed state $s$. Define the function
\begin{equation}\label{eq:F}
F_s(\theta_s) := \sum_{a\in \Acal} \pi_{\theta_s}(a|s) c(s,a)
\end{equation}
where $c(s,a)$ is constant, which we later set to be
$A^{(t)}(s,a)$; note we do not treat $c(s,a)$ as a function of $\theta$.
Thus,
	\begin{align*}
		&\frac{\partial F_s(\theta_s)}{\partial \theta_{s,a}}\Big\vert_{\theta_s^{(t)}}  = \sum_{a'\in \Acal} \frac{\partial \pi_{\theta_s}(a'|s)}{\partial \theta_{s,a}}\Big\vert_{\theta_s^{(t)}} c(s,a')\\
		&= \pi^{(t)}(a|s) (1- \pi^{(t)}(a|s)) c(s,a)  - \sum_{a'\neq a} \pi^{(t)}(a|s) \pi^{(t)}(a'|s)  c(s,a')\\
		&= \pi^{(t)}(a|s)\left(c(s,a) - \sum_{a'\in \Acal} \pi^{(t)}(a'|s) c(s,a')\right)
	\end{align*}
Taking $c(s,a)$ to be
$A^{(t)}(s,a)$ implies $\sum_{a'\in \Acal} \pi^{(t)}(a'|s) c(s,a') =
\sum_{a'\in \Acal} \pi^{(t)}(a'|s) A^{(t)}(s,a') =
0$,
\begin{equation}\label{eqn:grad-f}
\frac{\partial F_s(\theta_s)}{\partial \theta_{s,a}}\Big\vert_{\theta_s^{(t)}} = \pi^{(t)}(a|s) A^{(t)}(s,a)
\end{equation}
Observe that for the softmax parameterization,
\[
\theta_s^{(t+1)} = \theta_s^{(t)} + \eta \nabla_s V^{(t)}(\mu)
\]
where $\nabla_s$ is gradient w.r.t. $\theta_s$ and from Lemma
\ref{lemma:softmax-grad} that:
\[
\frac{\partial V^{(t)}(\mu)}{\partial \theta_{s,a}} =
\frac{1}{1-\gamma} d^{\pi^{(t)}}_{\mu}(s)\pi^{(t)}(a|s)A^{(t)}(s,a)
\]
This gives using Equation \eqref{eqn:grad-f}
\[
\theta_s^{(t+1)} = \theta_s^{(t)} + \eta \frac{1}{1-\gamma} d^{\pi^{(t)}}_{\mu}(s) \nabla_s F_s(\theta_s)\Big\vert_{\theta_s^{(t)}}
\]
Recall that for a $\beta$ smooth function, gradient ascent will
decrease the function value provided that $\eta \leq 1/\beta$
(Theorem \ref{thm:beck}).
Because $F_s(\theta_s)$ is $\beta$-smooth for $\beta =
\frac{5}{1-\gamma}$ (Lemma \ref{lemma:lipschitz-theta} and $\Abs{A^{(t)}(s,a)} \leq \frac{1}{1-\gamma}$), then our
assumption that
\[\eta \leq
  \frac{(1-\gamma)^2}{5}= (1-\gamma) \beta^{-1} \]
implies that $\eta \frac{1}{1-\gamma} d^{\pi^{(t)}}_{\mu}(s) \leq 1/\beta$,
and so we have
\[
  F_s(\theta^{(t+1)}_s) \geq F_s(\theta^{(t)}_s)
\]
which implies \eqref{eqn:F_monotone}.
\end{proof}

Next, we show the limit for iterates $V^{(t)}(s)$ and $Q^{(t)}(s,a)$
exists for all states $s$ and actions $a$.

\begin{lemma} \label{lem:delta}
For all states $s$ and actions $a$,  there exists values
 $V^{(\infty)}(s)$ and $Q^{(\infty)}(s,a)$ such that as $t \to \infty$,
$V^{(t)}(s) \rightarrow V^{(\infty)}(s)$ and
$Q^{(t)}(s,a) \rightarrow Q^{(\infty)}(s,a)$.
Define
\[
  \Delta = \min_{\{s,a | A^{(\infty)}(s,a) \neq 0\}} |A^{(\infty)}(s,a)| \]
where $A^{(\infty)}(s,a) =Q^{(\infty)}(s,a) - V^{(\infty)}(s)$.
Furthermore, there
exists a $T_0$ such that for all $t>T_0$, $s\in \Scal$, and
$a\in\Acal$, we have
\begin{equation}
  \label{eqn:delta-close}
  Q^{(t)}(s,a) \geq Q^{(\infty)}(s,a) - \Delta /4
\end{equation}
\end{lemma}

\begin{proof}
  Observe that $Q^{(t+1)}(s,a) \geq Q^{(t)}(s,a)$ (by Lemma
  \ref{lem:q-improv}) and $Q^{(t)}(s,a) \leq \frac{1}{1-\gamma}$,
  therefore by monotone convergence theorem, $Q^{(t)}(s,a) \rightarrow
  Q^{(\infty)}(s,a)$ for some constant $Q^{(\infty)}(s,a)$. Similarly
  it follows that $V^{(t)}(s) \rightarrow V^{(\infty)}(s)$ for some
  constant $V^{(\infty)}(s)$. Due to the limits existing, this implies we can
  choose $T_0$, such that the result \eqref{eqn:delta-close}
  follows.
\end{proof}
\\
\\
Based on the limits $V^{(\infty)}(s)$ and $Q^{(\infty)}(s,a)$, define following sets:
\begin{eqnarray*}
  I^s_{0} &:=&\{a | Q^{(\infty)}(s,a) = V^{(\infty)}(s)\}\\
  I^s_{+} &:=&\{a | Q^{(\infty)}(s,a) > V^{(\infty)}(s)\}\\
  I^s_{-} &:=&\{a | Q^{(\infty)}(s,a) < V^{(\infty)}(s)\} \, .
\end{eqnarray*}
In the following lemmas \ref{lemma:zero-pi}- \ref{lem:ratios}, we
first show that probabilities $\pi^{(t)}(a|s) \to 0$ for actions $a
\in I^s_{+} \cup I^s_{-}$ as $t\to \infty$.
We then show that for actions $a \in
I^s_{-}$, $\lim_{t\to \infty} \theta^{(t)}_{s,a} = -\infty$ and for
all actions $a \in I^s_{+}$, $\theta^{(t)}(a|s)$ is bounded from below
as $t\rightarrow \infty$.

\begin{lemma}
	\label{lemma:constant-sign}
	We have that there
	exists a $T_1$ such that for all $t>T_1$, $s\in \Scal$, and
	$a\in\Acal$, we have
	\begin{equation}
          A^{(t)}(s,a) < - \frac{\Delta}{4} \text{ for } a\in I^s_{-}; \quad
          A^{(t)}(s,a) > \frac{\Delta}{4} \text{ for } a\in I^s_{+}
	\end{equation}
\end{lemma}
\begin{proof}
	 Since, $V^{(t)}(s) \to V^{(\infty)}(s)$, we have that there exists $T_1>T_0$ such that for all $t> T_1$,\begin{equation}
		 \label{eqn:limit-v}
	 	V^{(t)}(s)>V^{(\infty)}(s) - \frac{\Delta}{4}.
	 \end{equation} Using Equation \eqref{eqn:delta-close}, it follows that for $t>T_1>T_0$, for $a \in I^s_{-}$
	 \begin{align*}
	 	A^{(t)}(s,a) &= Q^{(t)}(s,a) -V^{(t)}(s)\\
	 	&\leq Q^{(\infty)}(s,a) -V^{(t)}(s)\\
	 	&\leq Q^{(\infty)}(s,a) -V^{(\infty)}(s) + \Delta/4 \tag{Equation \eqref{eqn:limit-v}}\\
	 	&\leq -\Delta  + \Delta/4\tag{definition of $I^s_{-}$ and Lemma \ref{lem:delta}}\\
	 	&< -\Delta/4
	 \end{align*}
	 Similarly $A^{(t)}(s,a) = Q^{(t)}(s,a) -V^{(t)}(s) > \Delta/4 $ for $a \in I^s_{+}$ as
	 \begin{align*}
	 A^{(t)}(s,a) &= Q^{(t)}(s,a) -V^{(t)}(s)\\
	 &\geq Q^{(\infty)}(s,a) - \Delta/4 -V^{(t)}(s) \tag{Lemma \ref{lem:delta}}\\
	 &\geq Q^{(\infty)}(s,a) -V^{(\infty)}(s) - \Delta/4 \tag{$V^{(\infty)}(s) \geq V^{(t)}(s)$ from Lemma \ref{lem:q-improv}}\\
	 &\geq \Delta  - \Delta/4\\
	 &> \Delta/4
	 \end{align*}
which completes the proof.
\end{proof}

\begin{lemma}
	\label{lemma:zero-pi}
  $\frac{\partial V^{(t)}(\mu)}{\partial \theta_{s,a}} \rightarrow 0$ as $t\to
  \infty$ for all states $s$ and actions $a$. This implies that for $a\in
  I^s_{+}\cup I^s_{-}$, $\pi^{(t)}(a|s)\rightarrow 0$ and that $\sum_{a\in I^s_{0}}\pi^{(t)}(a|s)\rightarrow 1$.
\end{lemma}
\begin{proof}
  Because $V^{\pi_\theta}(\mu)$ is smooth (Lemma \ref{lemma:smooth-loss-lambda}) as a function of $\theta$, it follows from standard
  optimization results (Theorem \ref{thm:beck}) that $\frac{\partial V^{(t)}(\mu)}{\partial \theta_{s,a}}
  \rightarrow 0$ for all states $s$ and actions $a$. We have from Lemma
  \ref{lemma:softmax-grad}
  \[
    \frac{\partial V^{(t)}(\mu)}{\partial \theta_{s,a}}
    =\frac{1}{1-\gamma}d^{\pi^{(t)}}_{\mu}(s)\pi^{(t)}(a|s)
    A^{(t)}(s,a).
  \]
  Since, $|A^{(t)}(s,a)|> \frac{\Delta}{4}$ for all $t>T_1$ (from
  Lemma \ref{lemma:constant-sign}) for all $a\in I^s_{-} \cup I^s_{+}$
  and $d^{\pi^{(t)}}_{\mu}(s)\geq \frac{\mu(s)}{1-\gamma}> 0$ (using
  the strict positivity of $\mu$ in our assumption in Theorem~\ref{thm:glb-softmax}), we have
  $\pi^{(t)}(a|s)\to 0$.
\end{proof}

\begin{lemma}
  \label{lem:monotone}
(Monotonicity in $\theta^{(t)}_{s,a}$). For all $a\in I^s_{+}$, $\theta^{(t)}_{s,a}$ is
strictly increasing for $t\geq T_1$. For all $a\in I^s_{-}$, $\theta^{(t)}_{s,a}$ is
strictly decreasing for $t\geq T_1$.
\end{lemma}

\begin{proof}
We have from Lemma
\ref{lemma:softmax-grad}
\[
\frac{\partial V^{(t)}(\mu)}{\partial \theta_{s,a}}
=\frac{1}{1-\gamma}d^{\pi^{(t)}}_{\mu}(s)\pi^{(t)}(a|s)A^{(t)}(s,a)
\]
From Lemma \ref{lemma:constant-sign}, we have for all $t>T_1$
\[A^{(t)}(s,a) > 0 \text{ for } a\in I^s_{+}; \quad A^{(t)}(s,a) < 0
  \text{ for } a\in I^s_{-}\]
Since $d^{\pi^{(t)}}_{\mu}(s)>0$ and
$\pi^{(t)}(a|s)>0$ for the softmax parameterization, we have for all
$t>T_1$
\[\frac{\partial V^{(t)}(\mu)}{\partial \theta_{s,a}} > 0 \text{ for }
  a\in I^s_{+}; \quad \frac{\partial V^{(t)}(\mu)}{\partial
    \theta_{s,a}} < 0 \text{ for } a\in I^s_{-}\] This implies for all
$a\in I^s_{+}$,
$\theta^{(t+1)}_{s,a} - \theta^{(t)}_{s,a} = \frac{\partial
  V^{(t)}(\mu)}{\partial \theta_{s,a}} > 0$ i.e. $\theta^{(t)}_{s,a}$
is strictly increasing for $t\geq T_1$. The second claim follows
similarly.
\end{proof}

\begin{lemma}
  \label{lem:diverge}
For all $s$ where $I_+^s\neq \emptyset$, we have that:
  \[
\max_{a\in I^s_0}  \theta^{(t)}_{s,a}
\rightarrow \infty, \quad \min_{a\in\Acal}  \theta^{(t)}_{s,a} \rightarrow -\infty
  \]
\end{lemma}
\begin{proof}
Since $I_+^s\neq \emptyset$,  there exists some action $a_+\in I_+^s$. From Lemma \ref{lemma:zero-pi}, \[
\pi^{(t)}(a_+|s)\rightarrow 0, ~\text{as}~ t \to \infty
\] or equivalently by softmax parameterization,\[
	\frac{\exp(\theta^{(t)}_{s,a_+})}{\sum_a \exp(\theta^{(t)}_{s,a})} \to 0, ~\text{as}~ t \to \infty
\]
From Lemma
\ref{lem:monotone}, for any action $a\in I_+^s$ and in particular for $a_+$, $\theta^{(t)}_{s,a_+}$ is
monotonically increasing for $t> T_1$. That is the numerator in
previous display is monotonically increasing. Therefore, the
denominator should go to infinity i.e.
\[
	\sum_a \exp(\theta^{(t)}_{s,a}) \to \infty, ~\text{as}~ t \to \infty.
\]From Lemma \ref{lemma:zero-pi},\[
	\sum_{a\in I^s_{0}}\pi^{(t)}(a|s)\rightarrow 1, ~\text{as}~ t \to \infty
\]or equivalently\[
	\frac{\sum_{a\in I^s_{0}} \exp(\theta^{(t)}_{s,a})}{\sum_a \exp(\theta^{(t)}_{s,a})} \rightarrow 1, ~\text{as}~ t \to \infty
\] Since, denominator goes to $\infty$,\[
	\sum_{a\in I^s_{0}} \exp(\theta^{(t)}_{s,a}) \rightarrow \infty, ~\text{as}~ t \to \infty
\]which implies \[
	\max_{a\in I^s_0}  \theta^{(t)}_{s,a} \rightarrow \infty, ~\text{as}~ t \to \infty
\]
Note this also implies $\max_{a\in\Acal}  \theta^{(t)}_{s,a}
\rightarrow \infty$. The last part of the proof is completed using
that the gradients sum to $0$,
i.e. $\sum_a \frac{\partial V^{(t)}(\mu)}{\partial
  \theta_{s,a}}=0$. From gradient sum to $0$, we get that $\sum_{a\in
  \Acal} \theta^{(t)}_{s,a} = \sum_{a\in \Acal} \theta^{(0)}_{s,a} := c
\text{ for all } t> 0$ where $c$ is defined as the sum (over $\Acal$) of initial parameters.
That is $\min_{a\in \Acal} \theta^{(t)}_{s,a} < -
\frac{1}{|\Acal|}\max_{a\in \Acal} \theta^{(t)}_{s,a} + c$. Since,
$\max_{a\in \Acal} \theta^{(t)}_{s,a} \to \infty$, the result
follows.
\end{proof}

\begin{lemma}
	\label{lemma:stable}
	Suppose $a_+\in I^s_{+}$. For any $a\in I^s_0$, if
	there exists a $t\geq T_0$ such that
	$\pi^{(t)}(a|s)\leq\pi^{(t)}(a_+|s)$, then for all $\tau\geq t$, $\pi^{(\tau)}(a|s)\leq\pi^{(\tau)}(a_+|s)$.
\end{lemma}

\begin{proof}
The proof is inductive.  Suppose $\pi^{(t)}(a|s)\leq\pi^{(t)}(a_+|s)$, this implies from Lemma
\ref{lemma:softmax-grad}
\begin{align*}
  \frac{\partial V^{(t)}(\mu)}{\partial \theta_{s,a}}&= \frac{1}{1-\gamma}d^{(t)}_{\mu}(s) \pi^{(t)}(a|s)
	\bigg(Q^{(t)}(s,a) -V^{(t)}(s)\bigg)\\
	&\leq \frac{1}{1-\gamma}d^{(t)}_{\mu}(s) \pi^{(t)}(a_+|s)
	\bigg(Q^{(t)}(s,a_+) -V^{(t)}(s)\bigg)
	=\frac{\partial V^{(t)}(\mu)}{\partial \theta_{s,a_+}} .
\end{align*}
where the second to last step follows from
$Q^{(t)}(s,a_+) \geq Q^{(\infty)}(s,a_+) - \Delta/4 \geq
Q^{(\infty)}(s,a) + \Delta - \Delta/4 > Q^{(t)}(s,a)$ for
$t>T_0$. This implies that $\pi^{(t+1)}(a|s)\leq\pi^{(t+1)}(a_+|s)$
which completes the proof.
\end{proof}

Consider an arbitrary $a_+ \in I^s_{+}$. Let us partition the set
$ I^s_0$ into $B^s_0(a_+)$ and $\bar B^s_0(a_+)$ as follows: $B^s_0(a_+)$ is
the set of all $a\in I^s_0$ such that for all $t\geq T_0$,
$\pi^{(t)}(a_{+}|s)<\pi^{(t)}(a|s)$, and $\bar B^s_0(a_+)$ contains the
remainder of the actions from $I^s_0$.
We drop the argument $(a_+)$ when
clear from the context.

\begin{lemma}
	\label{lemma:thetab-diverge}
Suppose $I_+^s\neq \emptyset$. For all $a_+ \in I^s_+$, we have that
$B^s_0(a_+)\neq \emptyset$ and that
\begin{equation*}
  \sum_{a\in B^s_0(a_+)} \pi^{(t)}(a|s) \to 1, ~\text{as}~ t \to \infty.
\end{equation*}
This implies that:
\[
  \max_{a \in B^s_0(a_+)} \theta^{(t)}_{s,a} \to \infty.
\]
\end{lemma}
\begin{proof}
Let $a_{+}\in I^s_{+}$.
Consider any $a\in \bar B^s_0$. Then, by definition of $\bar B^s_0$, there exists
  $t'>T_0$ such that $\pi^{(t)}(a_{+}|s)\geq\pi^{(t)}(a|s)$. From
  Lemma \ref{lemma:stable}, for all $\tau > t$
  $\pi^{(\tau)}(a_{+}|s)\geq\pi^{(\tau)}(a|s)$. Also, since
  $\pi^{(t)}(a_{+}|s) \to 0$, this implies
  \[ \pi^{(t)}(a|s) \to 0 ~\text{for all}~ a \in \bar B^s_0
\]
Since, $B^s_0\cup\bar B^s_0 = I^s_0$ and $\sum_{a\in I^s_0} \pi^{(t)}(a|s)
\to 1$ (from Lemma \ref{lemma:zero-pi}), this
implies that $B^s_0\neq \emptyset$ and that
means
\begin{equation*}
  \sum_{a\in B^s_0} \pi^{(t)}(a|s) \to 1, ~\text{as}~ t \to \infty ,
\end{equation*}
which completes the proof of the first claim. The proof of the second
claim is identical to the proof in
Lemma \ref{lem:diverge} where instead of $\sum_{a\in I^s_{0}}\pi^{(t)}(a|s)\rightarrow 1$, we use $\sum_{a\in B^s_0} \pi^{(t)}(a|s) \to 1$.
\end{proof}

\begin{lemma}
	\label{lemma:small}
	Consider any $s$ where $I_+^s\neq \emptyset$. Then, for any
        $a_+ \in I^s_+$, there exists an iteration $T_{a_+}$ such that for all $t>T_{a_+}$,
	\[
		\pi^{(t)}(a_{+}|s) > \pi^{(t)}(a|s)
	\]
	for all $a\in \bar B^s_0(a_+)$.
\end{lemma}
\begin{proof}
	The proof follows from definition of $\bar B^s_0(a_+)$. That is if $a \in \bar B^s_0(a_+)$, then there exists a iteration $t_{a}> T_0$ such that $	\pi^{(t_a)}(a_{+}|s) > \pi^{(t_a)}(a|s)$. Then using Lemma \ref{lemma:stable}, for all $\tau> t_a$, $	\pi^{(\tau)}(a_{+}|s) > \pi^{(\tau)}(a|s)$. Choosing \[
	T_{a_+} = \max_{a\in B^s_0(a_+)}t_a
	\] completes the proof.
\end{proof}

\begin{lemma}
  \label{lem:ratios}
For all actions $a \in I^s_{+}$, we have that
$\theta^{(t)}_{s,a}$ is bounded from below as $t\rightarrow \infty$.
For all actions $a \in I^s_{-}$, we have that
$\theta^{(t)}_{s,a}\rightarrow -\infty$  as $t\rightarrow \infty$.
\end{lemma}

\begin{proof}
For the first claim, from Lemma \ref{lem:monotone}, we know that after $T_1$, $\theta^{(t)}_{s,a}$ is strictly increasing for $a \in I^s_{+}$, i.e. for all $t>T_1$  \[
	\theta^{(t)}_{s,a} \geq \theta^{(T_1)}_{s,a}.
\]
For the second claim, we know that after $T_1$, $\theta^{(t)}_{s,a}$
is strictly decreasing for $a \in I^s_{-}$ (Lemma
\ref{lem:monotone}). Therefore, by monotone convergence theorem,
$\lim_{t \to \infty} \theta^{(t)}_{s,a}$ exists and is either
$-\infty$ or some constant $\theta_0$. We now prove the second claim
by contradiction. Suppose $a \in I^s_{-}$ and that there exists a
$\theta_0$,  such that
$\theta^{(t)}_{s,a}>\theta_0$,  for $t\geq T_1$.
By Lemma~\ref{lem:diverge}, there must exist an action where $a'\in \Acal$ such
that
\begin{equation}
\label{eqn:limit-inf-a'}
	\lim\inf_{t\rightarrow\infty}  \theta^{(t)}_{s,a'} = -\infty.
\end{equation}

Let us consider some $\delta>0$ such that $\theta^{(T_1)}_{s,a'}\geq
\theta_0 -\delta$. Now for $t \geq T_1$ define $\tau(t)$ as follows:
$\tau(t)=k$ if $k$ is the
largest iteration in the interval $[T_1, t]$ such that $\theta^{(k)}_{s,a'}\geq
\theta_0 -\delta$ (i.e. $\tau(t)$ is the latest iteration before
$\theta_{s,a'}$ crosses below $\theta_0 -\delta$). Define
$\mathcal{T}^{(t)}$ as the subsequence of iterations
$ \tau(t) < t'< t$ such that $\theta^{(t')}_{s,a'}$ decreases, i.e.
\[
 \frac{\partial V^{(t')}(\mu)}{\partial \theta_{s,a'}} \leq 0, \textrm{
  for } \tau(t) < t'< t.
\]
Define $Z_t$ as the sum (if $\mathcal{T}^{(t)} = \emptyset$, we define $Z_t = 0$):
\[
Z_t = \sum_{t'\in \mathcal{T}^{(t)}} \frac{\partial V^{(t')}(\mu)}{\partial
    \theta_{s,a'}} \, .
\]
For non-empty $\mathcal{T}^{(t)}$, this gives:
\begin{multline*}
  Z_t=\sum_{t'\in \mathcal{T}^{(t)}} \frac{\partial V^{(t')}(\mu)}{\partial
    \theta_{s,a'}} \leq
  \sum_{t'=\tau(t)-1}^{t-1} \frac{\partial V^{(t')}(\mu)}{\partial
    \theta_{s,a'}} \leq
  \sum_{t'=\tau(t)}^{t-1} \frac{\partial V^{(t')}(\mu)}{\partial
    \theta_{s,a'}} +\frac{1}{1-\gamma^2}\\
  =\frac{1}{\eta}(\theta^{(t)}_{s,a'} - \theta^{(\tau(t))}_{s,a'}) +\frac{1}{1-\gamma^2}
  \leq   \frac{1}{\eta}\bigg(\theta^{(t)}_{s,a'} - (\theta_0 -\delta)
  \bigg) + \frac{1}{1-\gamma^2}\, ,
\end{multline*}
where we have used that $|\frac{\partial V^{(t')}(\mu)}{\partial
    \theta_{s,a'}}|\leq 1/(1-\gamma)$.
By \eqref{eqn:limit-inf-a'}, this implies that:
\[
\lim\inf_{t\rightarrow\infty}  Z_t = -\infty.
\]
For any $\mathcal{T}^{(t)} \neq \emptyset$, this implies that for all $t' \in \mathcal{T}^{(t)}$, from Lemma
 \ref{lemma:softmax-grad}
\begin{align*}
\left|\frac{\partial V^{(t')}(\mu)/\partial \theta_{s,a}}{\partial V^{(t')}(\mu)/\partial \theta_{s,a'}}\right|
=\left|\frac{\pi^{(t')}(a|s) A^{(t')}(s,a)}{\pi^{(t')}(a'|s) A^{(t')}(s,a')}\right|&\geq
\exp\big(\theta_0 - \theta^{(t')}_{s,a'}\big) \frac{(1-\gamma)\Delta}{4}
\\&\geq
\exp\big(\delta \big) \frac{(1-\gamma)\Delta}{4}
\end{align*}
where we have used that $ |A^{(t')}(s,a')| \leq 1/(1-\gamma)$ and
$|A^{(t')}(s,a)| \geq \frac{\Delta}{4}$ for all $t'>T_1$ (from Lemma
\ref{lemma:constant-sign}). Note that since
$\frac{\partial V^{(t')}(\mu)}{\partial \theta_{s,a}}< 0$ and
$\frac{\partial V^{(t')}(\mu)}{\partial \theta_{s,a'}}< 0$ over the
subsequence $\mathcal{T}^{(t)}$, the sign of the inequality reverses. In
particular, for any $\mathcal{T}^{(t)} \neq \emptyset$
\begin{align*}
	\frac{1}{\eta}(\theta^{(T_1)}_{s,a} - \theta^{(t)}_{s,a}) = \sum_{t'=T_1}^{t-1} \frac{\partial V^{(t')}(\mu)}{\partial \theta_{s,a}}
	&\leq \sum_{t'\in\mathcal{T}^{(t)}} \frac{\partial V^{(t')}(\mu)}{\partial \theta_{s,a}}\\
	&\leq \exp\big(\delta \big) \frac{(1-\gamma)\Delta}{4}
	\sum_{t'\in\mathcal{T}^{(t)}} \frac{\partial V^{(t')}(\mu)}{\partial \theta_{s,a'}}\\
	&= \exp\big(\delta \big) \frac{(1-\gamma)\Delta}{4} Z_t
\end{align*}
where the first step follows from that $\theta^{(t)}_{s,a}$ is
monotonically decreasing, i.e.
$\frac{\partial V^{(t)}(\mu)}{\partial \theta_{s,a}}< 0$ for
$t \notin \mathcal{T}$ (Lemma \ref{lem:monotone}). Since,\[
\lim\inf_{t\rightarrow\infty}  Z_t = -\infty,
\] this contradicts
that $\theta^{(t)}_{s,a}$ is lower bounded from below, which completes
the proof.
\end{proof}

\begin{lemma}
	\label{lemma:sum-bs-theta-1}
	Consider any $s$ where $I_+^s\neq \emptyset$. Then, for any $a_+ \in I^s_+$,\[
	\sum_{a \in B^s_0(a_+)} \theta^{(t)}_{s,a} \to \infty,~\text{as}~ t \to \infty
	\]
\end{lemma}
\begin{proof}
	Consider any $a\in B^s_0$. We have by definition of $B^s_0$
        that $\pi^{(t)}(a_{+}|s)<\pi^{(t)}(a|s)$ for all $t>
        T_0$. This implies by the softmax parameterization that
	$\theta^{(t)}_{s,a_{+}} < \theta^{(t)}_{s,a}$. Since,
	$\theta^{(t)}_{s,a_{+}}$ is lower bounded as $t\to \infty$ (using Lemma \ref{lem:ratios}), this implies $\theta^{(t)}_{s,a}$ is lower bounded as $t\to \infty$ for all $a \in B^s_0$. This in conjunction with
	$\max_{a\in B^s_0(a_+)} \theta^{(t)}_{s,a} \to \infty$ implies
	\begin{equation}
	\sum_{a \in B^s_0} \theta^{(t)}_{s,a} \to \infty ,
	\end{equation}
	which proves this claim.
\end{proof}
\\
\\
We are now ready to complete the proof for Theorem \ref{thm:glb-softmax}. We prove it by showing that $I^s_{+}$ is empty for all states $s$ or equivalently $V^{(t)}(s_0) \to V^\star(s_0)$ as $t \to \infty$.

\begin{proof}[Proof for Theorem \ref{thm:glb-softmax}]
Suppose the set $ I^s_{+}$ is non-empty for some $s$, else the proof is
complete. Let $a_{+}\in I^s_{+}$. Then, from Lemma \ref{lemma:sum-bs-theta-1}, \begin{equation}
\label{eqn:div-theta}
\sum_{a \in B^s_0} \theta^{(t)}_{s,a} \to \infty ,
\end{equation}
	
Now we proceed by showing a contradiction. For $a \in I^s_{-}$, we
have that since $\frac{\pi^{(t)}(a|s)}{\pi^{(t)}(a_{+}|s)} = \exp(\theta^{(t)}_{s,a} - \theta^{(t)}_{s,a_{+}}) \to 0$ (as $\theta^{(t)}_{s,a_{+}}$ is lower bounded and $\theta^{(t)}_{s,a} \to -\infty$ by
Lemma~\ref{lem:ratios}), there exists $T_2>T_0$ such that
\[\frac{\pi^{(t)}(a|s)}{\pi^{(t)}(a_{+}|s)} < \frac{(1-\gamma) \Delta}{16 |\Acal|} \]
or, equivalently,
\begin{equation}
  \label{eqn:bound1}
  -\sum_{a\in I^s_{-}} \frac{\pi^{(t)}(a|s)}{1-\gamma} > -\pi^{(t)}(a_+|s) \frac{\Delta}{16}.
\end{equation}

For $a \in \bar B^s_0$, we have $A^{(t)}(s,a) \rightarrow 0$  (by definition of set $I^s_0$ and $\bar B^s_0 \subset I^s_0$) and
$1<\frac{\pi^{(t)}(a_{+}|s)}{\pi^{(t)}(a|s)}$ for all $t>T_{a_+}$ from Lemma \ref{lemma:small}. Thus,
there exists $T_3>T_2, T_{a_+}$ such that
\[|A^{(t)}(s,a)| < \frac{\pi^{(t)}(a_{+}|s)}{\pi^{(t)}(a|s)} \frac{\Delta}{16 |\Acal|}\]
which implies
\[
\sum_{a\in \bar B^s_0} \pi^{(t)}(a|s) |A^{(t)}(s,a)| <  \pi^{(t)}(a_+|s) \frac{\Delta}{16}
\]
\begin{equation}
  \label{eqn:bound2}
  -\pi^{(t)}(a_+|s) \frac{\Delta}{16} < \sum_{a\in \bar B^s_0} \pi^{(t)}(a|s) A^{(t)}(s,a) < \pi^{(t)}(a_+|s) \frac{\Delta}{16}
\end{equation}

We have for $t>T_3$, from $\sum_{a\in \Acal} \pi^{(t)}(a|s) A^{(t)}(s,a) = 0$,
\begin{align*}
  &0 = \sum_{a\in I^s_0} \pi^{(t)}(a|s)
        A^{(t)}(s,a)+\sum_{a\in I^s_{+}} \pi^{(t)}(a|s) A^{(t)}(s,a)
        +\sum_{a\in I^s_{-}} \pi^{(t)}(a|s) A^{(t)}(s,a) \\
    &\stackrel{(a)}{\geq} \sum_{a\in B^s_0} \pi^{(t)}(a|s) A^{(t)}(s,a)
            +\sum_{a\in\bar B^s_0} \pi^{(t)}(a|s) A^{(t)}(s,a)
            +\pi^{(t)}(a_+|s) A^{(t)}(s,a_+) \\&\qquad \qquad+\sum_{a\in I^s_{-}} \pi^{(t)}(a|s) A^{(t)}(s,a)\\
    &\stackrel{(b)}{\geq} \sum_{a\in B^s_0} \pi^{(t)}(a|s) A^{(t)}(s,a)
            +\sum_{a\in\bar B^s_0} \pi^{(t)}(a|s) A^{(t)}(s,a)
            +\pi^{(t)}(a_+|s) \frac{\Delta}{4} -\sum_{a\in I^s_{-}} \frac{\pi^{(t)}(a|s)}{1-\gamma} \\
    &\stackrel{(c)}{>} \sum_{a\in B^s_0} \pi^{(t)}(a|s) A^{(t)}(s,a) -\pi^{(t)}(a_+|s) \frac{\Delta}{16}
         +\pi^{(t)}(a_+|s) \frac{\Delta}{4} -\pi^{(t)}(a_+|s) \frac{\Delta}{16}\\
		&> \sum_{a\in B^s_0} \pi^{(t)}(a|s) A^{(t)}(s,a)
\end{align*}
where in the step (a), we used $A^{(t)}(s,a)>0$ for all actions
$a\in I^s_+$ for $t>T_3>T_1$ from Lemma \ref{lemma:constant-sign}. In
the step (b), we used $A^{(t)}(s,a_+)\geq \frac{\Delta}{4}$ for $t>T_3>T_1$ from Lemma \ref{lemma:constant-sign} and $A^{(t)}(s,a) \geq -\frac{1}{1-\gamma}$ . In the step (c), we used Equation
\eqref{eqn:bound1} and left inequality in \eqref{eqn:bound2}. This
implies that for all $t>T_3$
\[\sum_{a\in B^s_0} \frac{\partial V^{(t)}(\mu)}{\partial \theta_{s,a}} < 0\]
This contradicts Equation \eqref{eqn:div-theta} which
requires
\[
  \lim_{t\to \infty} \sum_{a\in B^s_0} \left(\theta^{(t)}_{s,a} -
    \theta^{(T_3)}_{s,a}\right) = \eta \sum_{t=T_3}^\infty \sum_{a\in B^s_0} \frac{\partial V^{(t)}(\mu)}{\partial
    \theta_{s,a}} \to \infty .
\]
Therefore, the set $ I^s_{+}$ must be empty, which completes the
proof.
\end{proof}

\subsection{Proofs for Section~\ref{sec:entropy}}
\label{app:entropy}

\begin{proof}[\textbf{of Corollary~\ref{corollary:entropy}}]
Using Theorem~\ref{thm:small-gradient}, the desired optimality gap
$\epsilon$ will follow if we set 
\begin{equation}
\lambda = \frac{\epsilon(1-\gamma)}{2\Norm{\frac{d^{\pi^\star}_{\rho}}{\mu} }_\infty}
\label{eqn:entropy-lambda-eps}
\end{equation}
and if $\|\nabla_\theta L_\lambda(\theta)\|_2 \leq \lambda/(2|\Scal|\, |\Acal|)$.
In order to complete the proof, we need to bound the iteration
complexity of making the gradient sufficiently small.

Since the optimization
is deterministic and unconstrained, we can appeal to standard results (Theorem \ref{thm:beck}) which give that
after $T$ iterations of gradient ascent with stepsize of
$1/\beta_\lambda$, we have
\begin{equation}
\min_{t \leq T} \norm{\nabla_\theta L_\lambda(\theta^{(t)})}_2^2 \leq \frac{2\beta_\lambda (L_\lambda(\theta^{\star}) - L_\lambda(\theta^{(0)}))}{T} \leq \frac{2\beta_\lambda}{(1-\gamma)\, T},
\end{equation}
where $\beta_\lambda$ is an upper bound on the smoothness of
$L_\lambda(\theta)$.
We seek to ensure
\begin{align*}
\epsopt \leq \sqrt{\frac{2\beta_\lambda}{(1-\gamma)\, T}} \leq \frac{\lambda}{2 |\Scal|\,|\Acal|}
\end{align*}

Choosing $T \geq \frac{8\beta_\lambda\,
	|\Scal|^2|\Acal|^2}{(1-\gamma)\, \lambda^2}$ satisfies  the
above inequality. By Lemma \ref{lemma:smooth-loss-lambda}, we can
take $\beta_\lambda = \frac{8\gamma }{(1-\gamma)^3} + \frac{2\lambda
}{|\Scal|} $, and so
\begin{eqnarray*}
	\frac{8\beta_\lambda\,
		|\Scal|^2|\Acal|^2}{(1-\gamma)\, \lambda^2}&\leq& \frac{64\,
		|\Scal|^2|\Acal|^2}{(1-\gamma)^4\, \lambda^2}
	+\frac{16\, |\Scal| |\Acal|^{2}}{(1-\gamma)\, \lambda}\\
	&\leq& \frac{80\,
		|\Scal|^2|\Acal|^2}{(1-\gamma)^4\, \lambda^2}\\
	&=& \frac{320\,
		|\Scal|^2|\Acal|^2 }{(1-\gamma)^6\, \eps^2}\Norm{\frac{d^{\pi^\star}_{\rho}}{\mu} }_\infty^2
\end{eqnarray*}
where we have used that $\lambda<1$. This completes the proof.
\end{proof}

\subsection{Proofs for Section \ref{sec:npg}}
\label{app:npg}
\begin{proof}[\textbf{of Lemma~\ref{lemma:npg-softmax}}]
Following the definition of compatible function approximation
in~\citet{sutton1999policy}, which was also invoked
in~\citet{Kakade01}, for a vector $w \in \R^{|\Scal| |\Acal|}$, we define the error function
\[
    L^\theta(w) = \E_{s\sim d^{\pi_\theta}_\rho, a\sim \pi_\theta(\cdot | s)}(w^\top \nabla_\theta \log \pi_\theta(\cdot|s) - A^{\pi_\theta}(s,a))^2.
\]

Let $w^\star_\theta$ be the minimizer of $L^\theta(w)$ with the
smallest $\ell_2$ norm. Then by definition of Moore-Penrose
pseudoinverse, it is easily seen that
\begin{equation}\label{eq:easily}
    w^\star_\theta = F_\rho(\theta)^\dagger \E_{s\sim d^{\pi_\theta}_\rho, a\sim \pi_\theta(a | s)} [\nabla_\theta \log \pi_\theta(a|s) A^{\pi_\theta}(s,a)] = (1-\gamma) F_\rho(\theta)^\dagger \nabla_\theta V^{\pi_\theta}(\rho).
\end{equation}
In other words, $w^\star_\theta$ is precisely proportional to the NPG update direction. Note further that for the Softmax policy parameterization, we have by~\eqref{eqn:softmax-logpi-grad},
\[
    w^\top \nabla_\theta \log \pi_\theta(a|s) = w_{s,a} - \sum_{a'\in \Acal} w_{s,a'} \pi_\theta(a'|s).
\]
Since $\sum_{a\in \Acal} \pi(a|s) A^{\pi}(s,a) = 0$, this immediately
yields that $L^\theta(A^{\pi_\theta}) = 0$. However, this might not be
the unique minimizer of $L^\theta$, which is problematic since
$w^\star(\theta)$ as defined in terms of the Moore-Penrose
pseudoinverse is formally the smallest norm solution to the
least-squares problem, which $A^{\pi_\theta}$ may not be. However,
given any vector $v\in \R^{|\Sset| |\Aset|}$, let us consider solutions of the form $A^{\pi_\theta} + v$. Due to the form of the derivatives of the policy for the softmax parameterization (recall Equation~\ref{eqn:softmax-logpi-grad}), we have for any state $s,a$ such that $s$ is reachable under $\rho$,
\[
v^\top \nabla_\theta \log \pi_\theta(a|s) = \sum_{a'\in\Aset} (v_{s,a'}\ind[a = a'] - v_{s,a'}\pi_\theta(a'|s)) = v_{s,a} - \sum_{a'\in\Aset} v_{s,a'}\pi(a'|s).
\]
Note that here we have used that $\pi_\theta$ is a stochastic policy with $\pi_\theta(a|s) > 0$ for all actions $a$ in each state $s$, so that if a state is reachable under $\rho$, it will also be reachable using $\pi_\theta$, and hence the zero derivative conditions apply at each reachable state.
For $A^{\pi_\theta} + v$ to minimize $L^\theta$, we would like $v^\top \nabla_\theta \log \pi_\theta(a|s) = 0$ for all $s,a$ so that $v_{s,a}$ is independent of the action and can be written as a constant $c_s$ for each $s$ by the above equality. Hence, the minimizer of $L^\theta(w)$ is determined up to a state-dependent offset, and
\[
    F_\rho(\theta)^\dagger \nabla_\theta V^{\pi_\theta}(\rho) = \frac{A^{\pi_\theta}}{1-\gamma} + v,
\]
where $v_{s,a} = c_s$ for some $c_s\in \R$ for each state $s$ and action $a$. Finally, we observe that this yields the updates
\[
\theta^{(t+1)} = \theta^{(t)} + \frac{\eta}{1-\gamma} A^{(t)} + \eta v \quad \mbox{and}\quad \pi^{(t+1)}(a| s) = \pi^{(t)}(a| s) \frac{\exp(\eta A^{(t)}(s,a)/(1-\gamma) + \eta c_s)}{Z_t(s)}.
\]
Owing to the normalization factor $Z_t(s)$, the state dependent offset $c_s$ cancels in the updates for $\pi$, so that resulting policy is invariant to the specific choice of $c_s$. Hence, we pick $c_s \equiv 0$, which yields the statement of the lemma.
\end{proof}

        \section{Smoothness Proofs}

Various convergence guarantees we show leverage results from smooth,
non-convex optimization. In this section, we collect the various
results on smoothness of policies and value functions in the different
parameterizations which are needed in our analysis. 

Define the Hadamard product of two vectors:
\[
	[x\odot y]_i = x_i y_i
\] 
Define $\diag(x)$ for a column vector $x$ as the diagonal matrix
with diagonal as $x$. 

\begin{lemma}[Smoothness of $F$ (see Equation~\ref{eq:F}) ] \label{lemma:lipschitz-theta}
	Fix a state $s$. Let $\theta_s \in \R^{|\Acal|}$ be the column
        vector of parameters for state $s$. Let $\pi_{\theta}(\cdot|s)$ be the
        corresponding vector of action probabilities given by the
        softmax parameterization. For some fixed vector $c \in
        \R^{|\Acal|}$, define:
	\[
	F(\theta) := \pi_{\theta}(\cdot|s) \cdot c = \sum_a \pi_{\theta}(a|s)  c_a.
	\] Then
	\[\|\nabla_{\theta_s} F(\theta_s) - \nabla_{\theta_s} F(\theta_s')\|_2 \leq \beta \|\theta_s - \theta_s'\|_2 \]
	where
	\[\beta = 5\|c\|_\infty.	\]
\end{lemma}

\begin{proof}
For notational convenience,  we do not explicitly state the $s$ dependence.
	For the softmax parameterization, we have that
	\[
		\nabla_\theta \pi_\theta = \diag(\pi_\theta) - \pi_\theta\pi_\theta^\top.
	\]We can then write (as $\nabla_\theta \pi_\theta$ is symmetric),
	\begin{equation}
	\label{eqn:theta-first-derivative}
		\nabla_\theta (\pi_\theta \cdot c) = (\diag(\pi_\theta) - \pi_\theta\pi_\theta^\top) c = \pi_\theta \odot c - (\pi_\theta \cdot c) \pi_\theta
	\end{equation} and therefore\[
		\nabla^2_\theta (\pi_\theta \cdot c) = \nabla_\theta (\pi_\theta \odot c - (\pi_\theta \cdot c) \pi_\theta).
	\] For the first term, we get\[
		\nabla_\theta (\pi_\theta \odot c) = \diag(\pi_\theta\odot c) - \pi_\theta  (\pi_\theta \odot c)^\top,
	\]
and the second term, we can decompose by chain rule\begin{align*}
		\nabla_\theta((\pi_\theta \cdot c) \pi_\theta) &= (\pi_\theta \cdot c) \nabla_\theta \pi_\theta + (\nabla_\theta (\pi_\theta \cdot c)) \pi^\top_\theta
	\end{align*}Substituting these back, we get \begin{equation}
	\label{eqn:theta-state}
		\nabla^2_\theta (\pi_\theta \cdot c) = \diag(\pi_\theta\odot c) - \pi_\theta  (\pi_\theta \odot c)^\top - (\pi_\theta \cdot c) \nabla_\theta \pi_\theta - (\nabla_\theta (\pi_\theta \cdot c)) \pi^\top_\theta.
	\end{equation}Note that
	\begin{align*}
	&\max(\Norm{\diag(\pi_\theta\odot c)}_2,\Norm{\pi_\theta \odot c}_2, \Abs{\pi_\theta \cdot c})\leq \Norm{c}_\infty \\
	&\Norm{\nabla_\theta \pi_\theta}_2 = \Norm{\diag(\pi_\theta) - \pi_\theta\pi_\theta^\top}_2 \leq 1\\
	&\Norm{\nabla_\theta (\pi_\theta \cdot c)}_2 \leq \Norm{\pi_\theta \odot c}_2 + \Norm{(\pi_\theta \cdot c) \pi_\theta}_2 \leq 2\|c\|_\infty,
	\end{align*} which gives \[
		\Norm{\nabla^2_\theta (\pi_\theta \cdot c)}_2 \leq 5 \|c\|_\infty.
	\]
\end{proof}

Before we prove the smoothness results for $\nabla_\pi V^\pi(s_0)$ and
$\nabla_\theta V^{\pi_\theta}(s_0)$, we prove the following helpful
lemma. This lemma is general and not specific to the direct or softmax
policy parameterizations.  

\begin{lemma}
	\label{lemma:general-smoothness}
        Let $\pi_\alpha := \pi_{\theta+ \alpha u}$ and
let $\widetilde V(\alpha)$ be the corresponding value at a
fixed state $s_0$, i.e. 
	\[
	\widetilde V(\alpha) := V^{\pi_\alpha}(s_0).
	\]
	Assume that 
\begin{align*}
	\sum_{a\in \Acal}\Abs{\frac{d \pi_\alpha(a|s_0) }{d \alpha}\bigg\vert_{\alpha = 0}}\leq C_1,\quad
	\sum_{a\in \Acal} \Abs{\frac{d^2 \pi_\alpha(a|s_0) }{(d \alpha)^2}\bigg\vert_{\alpha = 0}}\leq C_2
	\end{align*}
	Then
	\[
	\max_{\Norm{u}_2=1}\Abs{\frac{d^2 \widetilde V(\alpha)}{(d \alpha)^2}\bigg\vert_{\alpha = 0}} \leq  \frac{C_2}{(1-\gamma)^2}+ \frac{2 \gamma C^2_1}{(1-\gamma)^3}.
	\]
\end{lemma}

\begin{proof}
	Consider a unit vector $u$ and let $\widetilde P(\alpha)$ be the state-action transition matrix under $\pi$, i.e.
	\[
	[\widetilde P(\alpha)]_{(s,a)\rightarrow (s',a')} = \pi_\alpha(a'|s') P(s'|s,a).
	\]
	We can differentiate $\widetilde P(\alpha)$ w.r.t $\alpha$ to get \[
	\left[\frac{d \widetilde P(\alpha)}{d\alpha}\bigg\vert_{\alpha=0}\right]_{(s,a)\rightarrow (s',a')} = \frac{d \pi_\alpha(a'|s')}{d \alpha}\bigg\vert_{\alpha=0} P(s'|s,a).
	\] For an arbitrary vector $x$,
	\[
	\left[\frac{d \widetilde P(\alpha)}{d\alpha}\bigg\vert_{\alpha=0} x\right]_{s,a} = \sum_{a',s'} \frac{d \pi_\alpha(a'|s')}{d \alpha}\bigg\vert_{\alpha=0} P(s'|s,a) x_{a',s'}
      \]and therefore \begin{align*}
	\max_{\Norm{u}_2=1}\left|\left[\frac{d \widetilde P(\alpha)}{d\alpha}\bigg\vert_{\alpha=0} x\right]_{s,a}\right|& =
	\max_{\Norm{u}_2=1}\left|\sum_{a',s'} \frac{d \pi_\alpha(a'|s')}{d \alpha}\bigg\vert_{\alpha=0} P(s'|s,a) x_{a',s'}\right|\\
&\leq \sum_{a',s'} \Abs{\frac{d \pi_\alpha(a'|s')}{d \alpha}\bigg\vert_{\alpha=0}} P(s'|s,a) |x_{a',s'}|\\
	&\leq \sum_{s'} P(s'|s,a) \|x\|_\infty \sum_{a'} \Abs{\frac{d \pi_\alpha(a'|s')}{d \alpha}\bigg\vert_{\alpha=0}}\\
	&\leq \sum_{s'} P(s'|s,a) \|x\|_\infty C_1\\
	&\leq C_1 \|x\|_\infty.
	\end{align*} By definition of $\ell_\infty$ norm, \[
	\max_{\Norm{u}_2=1}\Norm{\frac{d \widetilde P(\alpha)}{d\alpha}x}_\infty \leq C_1 \Norm{x}_\infty
      \]
      Similarly, differentiating $\widetilde P(\alpha)$ twice w.r.t. $\alpha$, we get \[
	\left[\frac{d^2\widetilde P(\alpha)}{(d\alpha)^2}\bigg\vert_{\alpha = 0}\right]_{(s,a)\rightarrow (s',a')} = \frac{d^2 \pi_\alpha(a'|s')}{(d \alpha)^2}\bigg\vert_{\alpha=0} P(s'|s,a).
      \]
      An identical argument leads to that, for arbitrary $x$,
      \begin{align*}
	\max_{\Norm{u}_2=1}\Norm{\frac{d^2\widetilde P(\alpha)}{(d\alpha)^2}\bigg\vert_{\alpha = 0} x}_\infty
	&\leq C_2 \Norm{x}_\infty
\end{align*}

	Let $Q^{\alpha}(s_0,a_0)$ be the corresponding $Q$-function for policy $\pi_\alpha$ at state $s_0$ and action $a_0$. Observe that $Q^{\alpha}(s_0,a_0)$ can be written as:
	\[
	Q^{\alpha}(s_0,a_0) = e_{(s_0,a_0)}^\top (\iden-\gamma \widetilde P(\alpha))^{-1} r = e_{(s_0,a_0)}^\top M(\alpha) r
	\]
	where $M(\alpha) := (\iden-\gamma \widetilde P(\alpha))^{-1}$ and differentiating twice w.r.t $\alpha$ gives:
	\begin{align*}
	\frac{d Q^\alpha(s_0,a)}{d \alpha} &= \gamma e_{(s_0,a)}^\top M(\alpha)\frac{d \widetilde P(\alpha)}{d\alpha}M(\alpha)r,\\
	\frac{d^2 Q^\alpha(s_0,a_0)}{(d \alpha)^2}&= 2 \gamma^2 e_{(s_0,a_0)}^\top M(\alpha)\frac{d \widetilde P(\alpha)}{d\alpha}M(\alpha)\frac{d \widetilde P(\alpha)}{d\alpha} M(\alpha) r \\& \qquad\qquad+
	\gamma e_{(s_0,a_0)}^\top M(\alpha)\frac{d^2 \widetilde P(\alpha)}{(d\alpha)^2}M(\alpha) r.
	\end{align*}
	By using power series expansion of matrix inverse, we can write $M(\alpha)$ as: \[
	M(\alpha) = (\iden - \gamma \widetilde P(\alpha))^{-1} = \sum_{n = 0}^{\infty} \gamma^n \widetilde P(\alpha)^n
	\] which implies that $M(\alpha) \geq 0$ (componentwise) and
	$M(\alpha) \ones = \frac{1}{1-\gamma} \ones$, i.e.  each row
	of $M(\alpha)$ is positive and sums to $1/(1-\gamma)$. This implies:
	\[
	\max_{\Norm{u}_2=1}\Norm{M(\alpha)x}_\infty \leq \frac{1}{1-\gamma} \Norm{x}_\infty
	\]

    This gives using expression for $\frac{d^2 Q^\alpha(s_0,a_0)}{(d \alpha)^2}$ and $\frac{d Q^\alpha(s_0,a)}{d \alpha}$, \begin{align*}
	\max_{\Norm{u}_2=1}\Abs{\frac{d^2 Q^\alpha(s_0,a_0)}{(d \alpha)^2}\bigg\vert_{\alpha = 0}}
	&\leq 2 \gamma^2\Norm{ M(\alpha)\frac{d \widetilde P(\alpha)}{d\alpha}M(\alpha)\frac{d \widetilde P(\alpha)}{d\alpha} M(\alpha) r}_\infty\\& \qquad\qquad+
	\gamma \Norm{M(\alpha)\frac{d^2 \widetilde P(\alpha)}{(d\alpha)^2}M(\alpha) r}_\infty\\
	&\leq \frac{2\gamma^2 C^2_1}{(1-\gamma)^3} + \frac{\gamma C_2}{(1-\gamma)^2}\\
	\max_{\Norm{u}_2=1}\Abs{\frac{d Q^\alpha(s_0,a)}{d \alpha}\bigg\vert_{\alpha = 0}}
	&\leq \Norm{\gamma M(\alpha)\frac{d \widetilde P(\alpha)}{d\alpha}M(\alpha)r}_\infty\\
	&\leq \frac{\gamma C_1}{(1-\gamma)^2}
	\end{align*}
Consider the identity:  \[
	\widetilde V(\alpha) = \sum_a \pi_\alpha(a|s_0) Q^\alpha(s_0,a),
	\] By differentiating $\widetilde V(\alpha)$ twice w.r.t $\alpha$, we get
 \begin{align*}
	\frac{d^2 \widetilde V(\alpha)}{(d \alpha)^2} = \sum_a \frac{d^2 \pi_\alpha(a|s_0) }{(d \alpha)^2} Q^\alpha(s_0,a) +  2\sum_a \frac{d \pi_\alpha(a|s_0)}{d \alpha} \frac{d Q^\alpha(s_0,a)}{d \alpha} + \sum_a \pi_\alpha(a|s_0) \frac{d^2 Q^\alpha(s_0,a)}{(d\alpha)^2}.
	\end{align*} Hence, \begin{align*}
		\max_{\Norm{u}_2=1}\Abs{\frac{d^2 \widetilde V(\alpha)}{(d \alpha)^2}} &\leq \frac{C_2}{1-\gamma}+ \frac{2 \gamma C^2_1}{(1-\gamma)^2} + \frac{2\gamma^2 C^2_1}{(1-\gamma)^3} + \frac{\gamma C_2}{(1-\gamma)^2}\\
		&= \frac{C_2}{(1-\gamma)^2}+ \frac{2 \gamma C^2_1}{(1-\gamma)^3},
                            \end{align*}
                            which completes the proof.
\end{proof}

Using this lemma, we now establish smoothness for:
the value functions under the direct policy parameterization and
the log barrier regularized objective~\ref{eqn:loss-reg} for the softmax parameterization.

\begin{lemma}[Smoothness for direct parameterization]
	\label{lemma:smooth-pi}
	For all starting states $s_0$,\[
	\Norm{\nabla_\pi V^{\pi}(s_0) - \nabla_\pi V^{\pi'}(s_0)}_2 \leq \frac{2\gamma|\Acal|}{(1-\gamma)^3} \Norm{\pi - \pi'}_2
	\]
\end{lemma}

\begin{proof}
	By differentiating $\pi_\alpha$ w.r.t $\alpha$ gives
	\[
		\sum_{a\in \Acal}\Abs{\frac{d \pi_\alpha(a|s_0) }{d \alpha}}\leq \sum_{a\in \Acal} \Abs{u_{a,s}} \leq \sqrt{|\Acal|}
	\] and differentiating again w.r.t $\alpha$ gives\[
		\sum_{a\in \Acal} \Abs{\frac{d^2 \pi_\alpha(a|s_0) }{(d \alpha)^2}} = 0
	\] Using this with Lemma \ref{lemma:general-smoothness} with $C_1 = \sqrt{|\Acal|}$ and $C_2 = 0$, we get\begin{align*}
		\max_{\Norm{u}_2=1}\Abs{\frac{d^2 \widetilde V(\alpha)}{(d \alpha)^2}\bigg\vert_{\alpha = 0}} \leq  \frac{C_2}{(1-\gamma)^2}+ \frac{2 \gamma C^2_1}{(1-\gamma)^3}
		 \leq \frac{2\gamma |\Acal|}{(1-\gamma)^3}
	\end{align*}
which completes the proof.
\end{proof}

We now present a smoothness result for the entropy regularized policy optimization problem which we study for the softmax parameterization.

\begin{lemma}[Smoothness for log barrier regularized softmax]
	\label{lemma:smooth-loss-lambda}
	For the softmax parameterization and\begin{eqnarray*}
	L_\lambda(\theta) =
	V^{\pi_\theta}(\mu) + \frac{\lambda}{|\Scal|\,|\Acal|} \sum_{s,a} \log \pi_\theta(a|s)
	\, ,
	\end{eqnarray*} we have that\[
		\Norm{\nabla_\theta L_\lambda(\theta) - \nabla_\theta L_\lambda(\theta') }_2 \leq \beta_\lambda \Norm{\theta - \theta'}_2
	\] where
	\[
	\beta_\lambda = \frac{8}{(1-\gamma)^3} +  \frac{2\lambda }{|\Scal|}
	\]
\end{lemma}

\begin{proof}
Let us first bound the smoothness of $V^{\pi_\theta}(\mu)$.
Consider a unit vector $u$. Let $\theta_s\in \R^{|\Acal|}$ denote the
parameters associated with a given state $s$. We have:
\[
\nabla_{\theta_s} \pi_\theta(a|s) = \pi_\theta(a|s) \bigg(e_a - \pi(\cdot|s)\bigg)
\]
and
\[
\nabla^2_{\theta_s} \pi_\theta(a|s) = \pi_\theta(a|s) \bigg(e_ae_a^\top - e_a\pi(\cdot|s)^\top-\pi(\cdot|s)e_a^\top+ 2\pi(\cdot|s)\pi(\cdot|s)^\top - \diag(\pi(\cdot|s))\bigg),
\]
where $e_a$ is a standard basis vector and $\pi(\cdot|s)$ is a
vector of probabilities. We also have by differentiating $\pi_\alpha(a|s)$ once w.r.t $\alpha$,
\begin{align*}
\sum_{a\in \Acal}\Abs{\frac{d \pi_\alpha(a|s) }{d \alpha}\bigg\vert_{\alpha = 0}}&\leq \sum_{a\in \Acal} \Abs{u^\top \nabla_{\theta + \alpha u} \pi_\alpha(a|s)\bigg\vert_{\alpha = 0}}\\
&\leq \sum_{a\in \Acal} \pi_\theta(a|s) \Abs{u_s^\top e_a - u_s^\top\pi(\cdot|s)}\\
&\leq \max_{a\in \Acal} \bigg(\Abs{u_s^\top e_a} + \Abs{u_s^\top\pi(\cdot|s)}\bigg) \leq 2
\end{align*}
Similarly, differentiating once again w.r.t. $\alpha$, we get \begin{align*}
\sum_{a\in \Acal} \Abs{\frac{d^2 \pi_\alpha(a|s) }{(d \alpha)^2}\bigg\vert_{\alpha = 0}}&\leq \sum_{a\in \Acal} \Abs{u^\top \nabla^2_{\theta + \alpha u} \pi_\alpha(a|s)\bigg\vert_{\alpha = 0} u}\\
&\leq \max_{a\in \Acal} \bigg(\Abs{u_s^\top e_ae_a^\top u_s} + \Abs{u_s^\top e_a\pi(\cdot|s)^\top u_s} + \Abs{u_s^\top \pi(\cdot|s)e_a^\top u_s}\\&\quad+ 2\Abs{u_s^\top \pi(\cdot|s)\pi(\cdot|s)^\top u_s} + \Abs{u_s^\top \diag(\pi(\cdot|s))u_s}\bigg)\\
&\leq 6
\end{align*} Using this with Lemma \ref{lemma:general-smoothness} for $C_1 = 2$ and $C_2 =6$, we get\begin{align*}
\max_{\Norm{u}_2=1}\Abs{\frac{d^2 \widetilde V(\alpha)}{(d \alpha)^2}\bigg\vert_{\alpha = 0}} \leq \frac{C_2}{(1-\gamma)^2}+ \frac{2 \gamma C^2_1}{(1-\gamma)^3}
\leq \frac{6}{(1-\gamma)^2}+ \frac{8\gamma}{(1-\gamma)^3} \leq \frac{8}{(1-\gamma)^3}
\end{align*} or equivalently for all starting states $s$ and hence for all starting state distributions $\mu$, \begin{equation}
	\label{eqn:v-factor}
	\Norm{\nabla_\theta V^{\pi_\theta}(\mu) - \nabla_\theta V^{\pi_{\theta'}}(\mu)}_2 \leq \beta \Norm{\theta - \theta'}_2
	\end{equation} where
$\beta =  \frac{8}{(1-\gamma)^3}$.

Now let us bound the smoothness of the regularizer
$\frac{\lambda}{|\Scal|}R(\theta) $, where
\[
R(\theta) := \frac{1}{|\Acal|} \sum_{s,a} \log \pi_\theta(a|s)
\]
We have
\[
\frac{\partial R(\theta)}{\partial \theta_{s,a}}
= \frac{1}{|\Acal|}-\pi_\theta(a|s).
\]
Equivalently,
\[
\nabla_{\theta_s} R(\theta)
= \frac{1}{|\Acal|}\ones-\pi_\theta(\cdot|s).
\]
Hence,
\[
\nabla^2_{\theta_s} R(\theta)
=-\diag(\pi_\theta(\cdot|s)) + \pi_\theta(\cdot|s) \pi_\theta(\cdot|s)^\top.
\]
For any vector $u_s$,
\[
\Abs{u_s^\top \nabla^2_{\theta_s} R(\theta) u_s} =
\Abs{u_s^\top\diag(\pi_\theta(\cdot|s)) u_s-(u_s\cdot
\pi_\theta(\cdot|s))^2} \leq 2 \|u_s\|^2_\infty.
\]
Since $\nabla_{\theta_s} \nabla_{\theta_{s'}} R(\theta)=0$ for
$s\neq s'$,
\[
\Abs{u^\top \nabla^2_{\theta} R(\theta) u}=\Abs{\sum_s  u_s^\top
\nabla^2_{\theta_s} R(\theta) u_s} \leq 2 \sum_s  \|u_s\|^2_\infty \leq 2\|u\|_2^2.
\]
Thus $R$ is $2$-smooth and $\frac{\lambda}{|\Scal|}R $
is $\frac{2\lambda}{|\Scal|} $-smooth, which completes the proof.
\end{proof}

        \section{Standard Optimization Results}
In this section, we present the standard optimization results from \cite{ghadimi2016accelerated,book:beck} used in our proofs. We consider solving the following problem \begin{equation}
\label{eq:prob}
	\min_{x\in C} \{f(x)\}
\end{equation} with $C$ being a nonempty closed and convex set. We assume the following
\begin{assumption}
	\label{assum:optif}
	 $f: \R^d \to (-\infty, \infty)$ is proper and closed, $\text{dom}(f)$ is convex and $f$ is $\beta$ smooth over $\text{int}(\text{dom}(f))$.
\end{assumption}
Throughout the section, we will denote the optimal $f$ value by $f(x^\ast)$.

\begin{definition}[Gradient Mapping]
	We define the gradient mapping $G^\eta(x)$ as \begin{equation}
	G^\eta(x) := \frac{1}{\eta} \left(x - P_C(x - \eta \nabla f(x))\right)
\end{equation} where $P_C$ is the projection onto $C$.
\end{definition} Note that when $C=\R^d$, the gradient mapping $G^\eta(x) = \nabla f(x)$.

\begin{theorem}[Theorem 10.15 \cite{book:beck}]
	\label{thm:beck}
	Suppose that Assumption \ref{assum:optif} holds and let $\{x_k\}_{k\geq 0}$ be the sequence generated by the gradient descent algorithm for solving the problem \eqref{eq:prob} with the stepsize $\eta = 1/\beta$. Then, \begin{enumerate}
		\item The sequence $\{F(x_t)\}_{t\geq 0}$ is non-increasing.
		\item $G^\eta(x_t) \to 0$ as $t \to \infty$
		\item $
			\min_{t = 0,1,\ldots, T-1} \|G^\eta(x_t)\| \leq \frac{\sqrt{2\beta f(x_0) - f(x^\ast)}}{\sqrt{T}}
	$
	\end{enumerate}
\end{theorem}

\begin{theorem}[Lemma 3 \cite{ghadimi2016accelerated}]
	\label{thm:ghadimi}
	Suppose that Assumption \ref{assum:optif} holds. Let $x^+ =
	x -\eta G^\eta(x)$. Then,
	\begin{align*}
	\nabla f(x^+) \in N_C ( x^+) +\epsilon(\eta \beta+1) B_2,
	\end{align*}
	where $B_2$ is the unit $\ell_2$ ball, and $N_C$ is the normal cone of the set $C$.
\end{theorem}

We
now consider the stochastic projected gradient descent algorithm where
at each time step $t$, we update $x_t$ by sampling a random $v_t$ such
that \begin{equation} x_{t+1} = P_C(x_t - \eta v_t)\, , \quad
  \text{where}~ \E[v_t|x_t] = \nabla f(x_t)
\end{equation}

\begin{theorem}[Theorem 14.8  and Lemma 14.9 \cite{shalev2014understanding}]
	\label{thm:shalev}
        Assume $C = \{x: \norm{x} \leq B\}$, for some $B> 0$. 
        Let $f$ be a convex function and let $x^\ast
        \in \argmin_{x:\norm{x}\leq B} f(w)$. Assume also that for all
        $t$, $\norm{v_t}\leq \rho$, and that stochastic
        projected gradient descent is run for $N$ iterations with
        $\eta = \sqrt{\frac{B^2}{\rho^2 N}}$. Then, 
\begin{equation}
  \E\left[f\left(\frac{1}{N}\sum_{t=1}^N x_t\right)\right] - f(x^\ast) \leq \frac{B\rho}{\sqrt{N}}
\end{equation}
\end{theorem}

\end{document}